\PassOptionsToPackage{dvipsnames,svgnames,x11names,table}{xcolor}
\documentclass{article}
\let\paperaddcontentsline\addcontentsline
\usepackage{chapters/iclr2027_conference,times}

\newif\ifpaperpreprint
\iclrfinalcopy       

\let\addcontentsline\paperaddcontentsline
\usepackage[utf8]{inputenc}
\usepackage[T1]{fontenc}
\usepackage{url}
\usepackage{booktabs}
\usepackage{amsfonts}
\usepackage{nicefrac}
\usepackage{microtype}
\usepackage{mdframed}
\usepackage{wrapfig}
\usepackage{amsthm}
\usepackage{bbm}
\usepackage{thm-restate}
\usepackage{float}
\usepackage{import}
\usepackage{pifont}
\usepackage{xcolor}
\newif\ifrevmarks \revmarksfalse                     
\newcommand{\revcolors}{\color{blue}\hypersetup{linkcolor=blue,citecolor=blue,urlcolor=blue}}
\DeclareRobustCommand{\rev}[1]{\ifrevmarks{\revcolors#1}\else#1\fi}
\newcommand{\revstart}{\ifrevmarks\revcolors\captionsetup{font+={color=blue}}\fi}  
\newcommand{\revredcolors}{\color{red}\hypersetup{linkcolor=red,citecolor=red,urlcolor=red}}  
\DeclareRobustCommand{\revred}[1]{\ifrevmarks{\revredcolors#1}\else#1\fi}
\newcommand{\revredstart}{\ifrevmarks\revredcolors\captionsetup{font+={color=red}}\fi}
\usepackage{svg}

\usepackage{amsmath}
\usepackage{mathtools}
\usepackage{amssymb}
\allowdisplaybreaks
\usepackage{bm}
\usepackage{fix-cm}
\usepackage[separate-uncertainty=true]{siunitx}
\usepackage{xfrac}
\usepackage[percent]{overpic}

\usepackage{graphicx}
\usepackage{flafter} 
\usepackage{placeins}
\usepackage{subcaption}
\usepackage{makecell}
\usepackage{multicol}
\usepackage{multirow}
\usepackage{adjustbox}
\usepackage{needspace}
\usepackage{tikz}
\usetikzlibrary{arrows.meta, positioning, fit, calc, backgrounds, decorations.pathreplacing, shapes.geometric}
\usepackage{etoc}
\usepackage{ifthen}
\usepackage{algorithm}
\usepackage{algpseudocode}
\usepackage[inline]{enumitem}

\usepackage{hyperref}
\usepackage[nameinlink]{cleveref}
\usepackage{csquotes}
\DeclareRobustCommand{\localejepa}{\textsc{LocaLeJEPA}}
\newcommand{\cmark}{\textcolor{Green}{\ding{51}}}
\newcommand{\xmark}{\textcolor{Red}{\ding{55}}}

\usepackage[most]{tcolorbox}
\tcbuselibrary{skins, breakable}

\definecolor{proofbar}{RGB}{225,225,225}
\definecolor{theorembar}{RGB}{128, 0, 128}
\definecolor{propositionbar}{RGB}{0, 100, 0}
\definecolor{definitionbar}{RGB}{0, 0, 200}
\definecolor{lemmabar}{RGB}{255, 140, 0}
\definecolor{corollarybar}{RGB}{204, 153, 255}
\definecolor{remarkbar}{RGB}{0, 139, 139}
\definecolor{assumptionbar}{RGB}{178, 34, 34}

\tcbset{
    mybarstyle/.style={
        enhanced,
        frame hidden,
        borderline west={2pt}{0pt}{#1},
        sharp corners,
        breakable,
        boxrule=0pt,        
        leftrule=0pt,       
        colframe=#1,        
        opacityback=0,      
        colback=white,
        left=6pt,           
        right=0pt,
        top=0pt,
        bottom=0pt,
        before skip=10pt,
        after skip=10pt,
    }
}

\makeatletter
\theoremstyle{plain}
\@ifundefined{theorem}{%
  \@ifundefined{c@theorem}{%
  }{%
    \@ifundefined{thetheorem}{}{}%
    \@namedef{theorem}{\@thm{\th@plain}{theorem}{Theorem}}%
    \expandafter\let\csname endtheorem\endcsname\@endtheorem
  }%
}{}

\newcommand{\DeclareSafeTheoremEnvironment}[4]{%
  \@ifundefined{#1}{%
    \expandafter\gdef\csname #1\endcsname{\@thm{#3}{#4}{#2}}%
    \expandafter\let\csname end#1\endcsname\@endtheorem
  }{}%
}
\DeclareSafeTheoremEnvironment{proposition}{Proposition}{\th@plain}{theorem}
\DeclareSafeTheoremEnvironment{lemma}{Lemma}{\th@plain}{theorem}
\DeclareSafeTheoremEnvironment{corollary}{Corollary}{\th@plain}{theorem}
\DeclareSafeTheoremEnvironment{assumption}{Assumption}{\th@plain}{theorem}
\DeclareSafeTheoremEnvironment{definition}{Definition}{\th@definition}{theorem}
\DeclareSafeTheoremEnvironment{remark}{Remark}{\th@remark}{}
\makeatother

\newif\ifpaperbars
\AtBeginDocument{%
\ificlrfinal\paperbarstrue\fi
\ifpaperpreprint\paperbarstrue\fi
\ifpaperbars
\tcolorboxenvironment{theorem}{mybarstyle=theorembar}
\tcolorboxenvironment{proposition}{mybarstyle=propositionbar}
\tcolorboxenvironment{lemma}{mybarstyle=lemmabar}
\tcolorboxenvironment{corollary}{mybarstyle=corollarybar}
\tcolorboxenvironment{definition}{mybarstyle=definitionbar}
\tcolorboxenvironment{remark}{mybarstyle=remarkbar}
\tcolorboxenvironment{assumption}{mybarstyle=assumptionbar}
\tcolorboxenvironment{proof}{mybarstyle=proofbar}
\fi
}

\newboolean{showcomments}
\setboolean{showcomments}{false}
\newcommand{\fr}[1]{\ifthenelse{\boolean{showcomments}}{\textcolor{red}{\textbf{[FR: #1]}}}{}}

\definecolor{Top1}{HTML}{0072B2} 
\definecolor{Top2}{HTML}{E69F00} 
\definecolor{Top3}{HTML}{CC79A7}

\definecolor{ufill}{HTML}{E1D5E7}
\definecolor{uborder}{HTML}{9673A6}
\definecolor{vfill}{HTML}{D5E8D4}
\definecolor{vborder}{HTML}{82B366}

\definecolor{ink}{HTML}{172D46}
\definecolor{teal}{HTML}{087F8C}
\definecolor{violet}{HTML}{7357A6}
\definecolor{amber}{HTML}{CC8431}
\definecolor{mist}{HTML}{F2F5F8}
\hypersetup{colorlinks=true,linkcolor=ink,citecolor=teal,urlcolor=teal,
  pdftitle={LocaLeJEPA: Maximum Entropy Representations under Geometric Priors},
  pdfauthor={Anonymous}}

\newcommand{\papertitle}{\localejepa{}: Maximum Entropy Representations under Geometric Priors}
\title{\localejepa{}: Maximum Entropy Representations under Geometric Priors}

\makeatletter
\newcommand{\silentthanks}[1]{%
  \begingroup\let\@makefnmark\@empty\thanks{#1}\endgroup}
\makeatother
\author{%
\parbox[t]{0.96\textwidth}{%
\raggedright
{\normalfont\small\bfseries\linespread{1.18}\selectfont
Kartik Tandon$^{*,1,\dagger}$,\hspace{0.45em}
Alessio Borgi$^{*,2,7,\dagger}$,\hspace{0.45em}
Pavel Snopov$^{3}$,\hspace{0.45em}
Inés García-Redondo$^{4}$,\hspace{0.45em}
Davide D'Ascenzo$^{5}$,\hspace{0.45em}
Sebastiano Cultrera di Montesano$^{6}$,\hspace{0.45em}
Fabrizio Silvestri$^{2}$,\hspace{0.45em}
Pietro Liò$^{7}$,\hspace{0.45em}
Francesca Dominici$^{8}$,\hspace{0.45em}
Randall Balestriero$^{9}$,\hspace{0.45em}
Claudio Battiloro$^{9,8,\dagger}$\endgraf}
\vspace{0.9em}
{\normalfont\footnotesize\linespread{1.12}\selectfont
$^1$University of Pennsylvania $\cdot$
$^2$Sapienza University of Rome $\cdot$
$^3$University of Texas Rio Grande Valley $\cdot$
$^4$University of Fribourg $\cdot$
$^5$University of Milan $\cdot$
$^6$Broad Institute $\cdot$
$^7$University of Cambridge $\cdot$
$^8$Harvard University $\cdot$
$^9$Brown University\endgraf}
}%
\silentthanks{: Equal contribution.\space
$\dagger$: Corresponding authors: \href{mailto:ktandon@sas.upenn.edu}{ktandon@sas.upenn.edu},
\href{mailto:alessio.borgi@uniroma1.it}{alessio.borgi@uniroma1.it},
\href{mailto:ab3352@cam.ac.uk}{ab3352@cam.ac.uk},
\href{mailto:claudio_battiloro@brown.edu}{claudio\_battiloro@brown.edu}.\space
}%
}

\begin{document}
\maketitle
\begin{abstract}
Joint-Embedding Predictive Architectures (JEPA) offer \rev{an appealing} approach to self-supervised learning, but are typically prone to representation collapse. LeJEPA resolves this by introducing a \rev{theoretically grounded} distribution-matching objective that drives the embedding towards an isotropic Gaussian. However, this approach is fundamentally tied to Euclidean assumptions\rev{, and its global target does not reflect} domain-specific geometric priors, such as the use of non-Euclidean embedding geometries, e.g., equivariant CNNs or hyperbolic ViTs. We thus introduce LocaLeJEPA, a framework motivated by a manifold-theoretic reformulation of LeJEPA’s entropy-maximization property. The accompanying LocaLeReg, a novel locally Euclidean distribution-matching regularizer, enables a \rev{stable and scalable} implementation of the LocaLeJEPA framework. We then perform extensive experiments in the vision setting on models trained on \revred{ImageNet-1K}. LocaLeJEPA \rev{significantly} outperforms LeJEPA when paired with hyperbolic ViT-S/8 and equivariant ResNet-50 backbones\rev{, and also with ordinary Euclidean ViT-S/8 and ViT-L/14 backbones. This holds} both in terms of linear and \revred{$k$-NN} classification performance on ImageNet and on a series of downstream tasks tested. \rev{We} view the LocaLeJEPA formalism as a principled approach to the JEPA paradigm that better respects intrinsic data geometry while also accommodating generalizations to exotic embedding geometries. Code available at: \url{https://github.com/alessioborgi/LocaLeJEPA}.
\end{abstract}

\vspace{-0.3cm}
\section{Introduction}
\vspace{-0.3cm}
\label{sec:introduction}

The Joint-Embedding Predictive Architecture (JEPA) paradigm \citep{bromley1993signature, lecun2022autonomous, Assran_2023_CVPR} provides a tantalizing alternative to both feature-space masked modeling and contrastive sample-based objectives for Self-Supervised Learning (SSL). Allowing the model to instead learn relevant features by understanding the relationship between various semantically overlapping `views' of the same object in latent space, it ideally represents a \rev{simpler} approach to SSL. However, in practice, this learning objective falls prey to representation collapse, requiring the incorporation of techniques such as stop-gradient \rev{\citep{chen2021siamsiam}}, EMA schedulers \citep{Caron_2021_ICCV}, and various whitening-based regularizers \citep{ermolov2021whitening, tian2021understanding}. The recent Latent-Euclidean JEPA (LeJEPA) \citep{balestriero2025lejepa} bypasses these ad-hoc methodologies and identifies the isotropic Gaussian as the optimal embedding distribution for downstream tasks. Optimality is minimax, \rev{i.e.,} no other embedding distribution yields a lower worst-case risk for linear regression or \revred{$k$-NN} classification on an unknown downstream task. LeJEPA \rev{therefore} adds a Gaussian distribution matching objective to the JEPA loss, preventing collapse while being stable and scalable.

Consequent to the introduction of LeJEPA, there has been a flurry of activity in adapting this methodology for use in video models \citep{kuhn2026levjepaefficientscalable}, world models \citep{maes_lelidec2026lewm}, and vision-language models \citep{kuhn2026levljepa}, as well as diverse domains such as time series \citep{chemeris2026lenepa}\rev{ and} molecular representations \citep{rottach2026moljepa}. 

\textbf{Motivation.} Concurrently, however, another line of recent work \rev{questions} the claimed optimality of LeJEPA’s isotropic Gaussian, proposing instead objectives that match to rectified Gaussians \citep{kuang2026rectifiedlpjepajointembeddingpredictive}, subspace Gaussians \citep{zhao2026subjepa}, or the uniform distribution on the sphere \citep{nicollier2026spherejepa, nicollier2026expandingspherejepa}. A throughline of these works is that they operate in regimes that violate the fundamental Euclidean-space assumptions of LeJEPA, thus voiding the theoretical optimality of the Gaussian distribution. For instance, \cite{nicollier2026spherejepa} considers embeddings constrained to a hypersphere $\mathbb{S}^d$ while \cite{kuang2026rectifiedlpjepajointembeddingpredictive} considers $l_p$-normed spaces. As the JEPA framework is increasingly adapted to diverse, and often explicitly non-Euclidean domains, this naturally raises the question whether there is a principled generalization of LeJEPA to such domains that still retains its theoretical grounding and \rev{its stability in practice}.

\vspace{-0.2cm}
\begin{figure}[hbpt]
    \centering
    \includegraphics[
        width=.85\linewidth
    ]{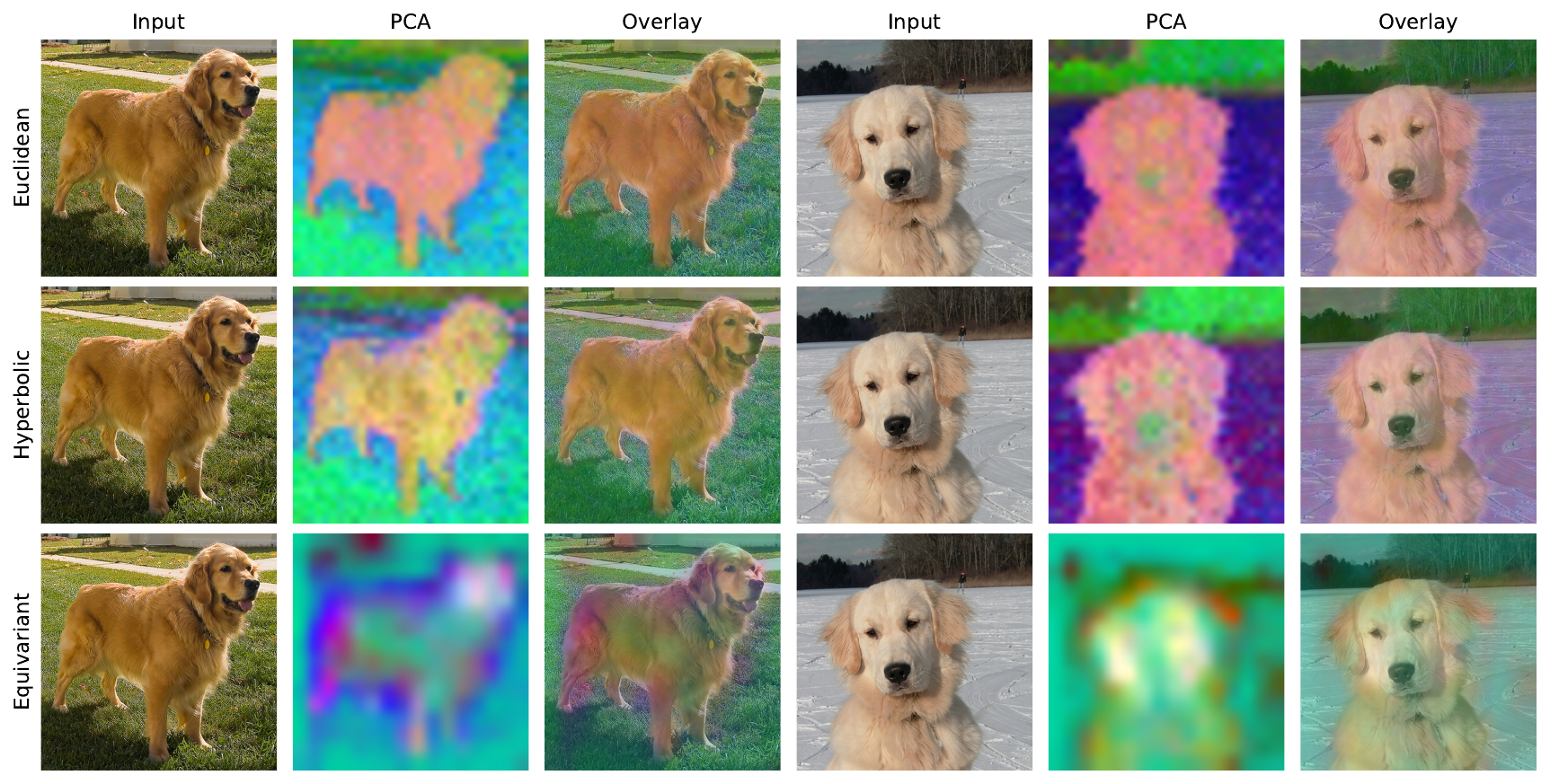}
    \caption{\emph{LocaLeJEPA learns semantically \rev{organized} representations, across geometries.}
    PCA visualizations of frozen LocaLeJEPA patch features for two ImageNet-1K validation images. Rows correspond to an ordinary (i.e. Euclidean) ViT-S/8, hyperbolic ViT-S/8, and $D_4$-equivariant ResNet-50, while columns show the input, the first three principal components mapped to RGB, and their overlay. Across all three geometries, the self-supervised features learned by LocaLeJEPA exhibit coherent object--background structure despite their \rev{very} different representation spaces. Full visualization details and comparisons with LeJEPA are provided in \Cref{app:pca}\rev{ (\Cref{fig:pca-vs-sigreg})}.}
    \label{fig:pca-teaser}
\end{figure}
\vspace{-0.2cm}

\textbf{Contribution.} In this work, we introduce \textbf{LocaLeJEPA}, a manifold-theoretic JEPA objective designed to \rev{couple} the geometry of the data manifold $\mathcal{M}$ \rev{with that} of an exotic embedding space $\mathcal{N}$. The accompanying regularizer, \textit{\textbf{Loca}lly \textbf{L}atent \textbf{e}uclidean \textbf{Reg}ularization} (LocaLeReg), drives the embeddings to be locally uniformly distributed on a Euclidean ball. Since every manifold is locally Euclidean and the uniform distribution is entropy maximization on a compact space, we view this as a local adaptation of the global entropy maximization property of LeJEPA’s Gaussian distribution. We implement this by first partitioning each batch of embeddings into $K$ \revred{neighborhoods} by $K$-means in latent space, and then applying an Epps-Pulley-based distribution-matching objective to each \revred{neighborhood}. This retains LeJEPA's sketched projections approach and thus inherits its linear cost and stability, while allowing for a far more expressive embedding distribution due to locality. \rev{While} locality allows us to consider exotic $\mathcal{N}$, even the global $K=1$ case \rev{already improves on LeJEPA trained with the same optimization settings, on ViT-S/8, hyperbolic ViT-S/8 and the $D_4$-equivariant ResNet-50 (\Cref{tab:in1k-main}). Our target distribution is also compact,} theoretically allowing the encoder to learn a continuous function with respect to $\mathcal{M}$.  

We experimentally instantiate LocaLeJEPA in the vision setting, though we note the framework is suitably general to support diverse modalities such as graph data. Across Euclidean, hyperbolic, and $D_4$-equivariant backbones pretrained on Imagenette and ImageNet-1K, we show LocaLeJEPA produces qualitatively better embeddings that inherently preserve the intrinsic topology of the data manifold $\mathcal{M}$, such as preserving known circular or toroidal features, while also accommodating non-Euclidean embedding spaces $\mathcal{N}$. We then quantitatively establish the downstream performance of LocaLeJEPA's embeddings, showing it outperforms LeJEPA under \rev{linear probing} and \revred{$k$-NN} classification using Euclidean, hyperbolic, and equivariant backbones, for all datasets tested. LocaLeJEPA expands the scope of the JEPA paradigm as a whole while retaining the stability and scalability afforded by the entropy-maximizing objective of LeJEPA.


\begin{figure}[t!]
\centering
\includegraphics[width=\linewidth]{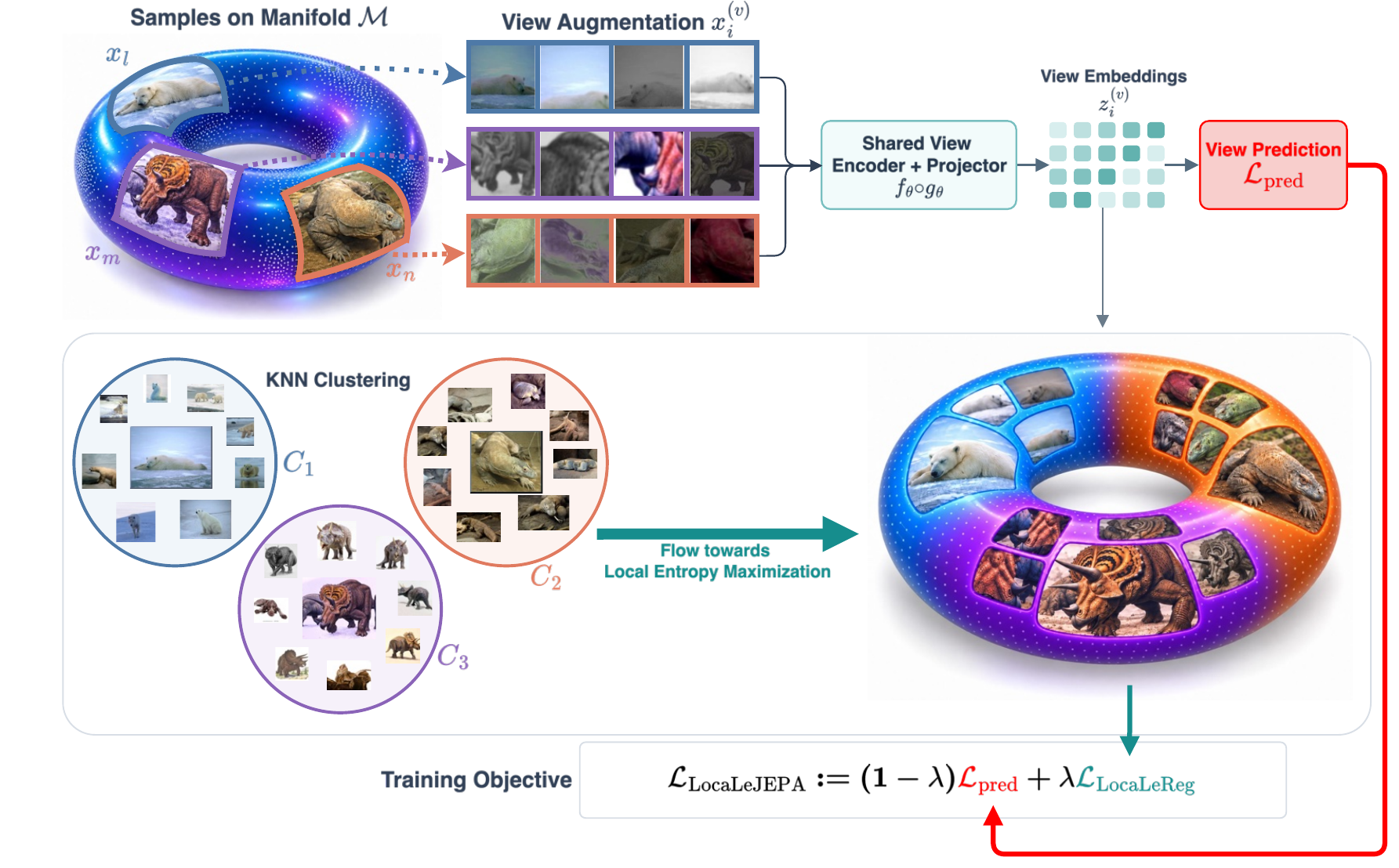}
\caption{\emph{LocaLeJEPA Pipeline.}
A shared encoder and projector feed two losses: $\mathcal{L}_{\text{pred}}$ for view prediction and $\mathcal{L}_{\text{LocaLeReg}}$ for distribution matching. $\mathcal{L}_{\text{pred}}$ drives the model to predict an image in latent space from semantically synecdochic views. Concurrently, $\mathcal{L}_{\text{LocaLeReg}}$ counteracts representation collapse, without driving to a fixed global distribution. This is achieved by clustering batches of embedded views by K-Means, and then driving each cluster towards the uniform distribution on the Euclidean ball $\mathbb{B}^d$. Locality allows for the expressivity necessary to learn the intrinsic geometry of the data manifold, while uniformity within clusters encourages entropy maximization.}
\label{fig:architecture}
\end{figure}
\vspace{-0.3cm}
\section{Related Works}
\vspace{-0.3cm}
\label{sec:related_works}

\textbf{Self-Supervised Learning} The goal of self-supervised learning (SSL) is to learn data representations that capture intrinsic features that are useful for a wide-range of downstream tasks. Initial approaches to SSL relied largely on positive-negative sample pairs combined with a contrastive loss, as in CPC \citep{oord2018representation}, MoCo \citep{he2020moco}, and
SimCLR \citep{chen2020simclr}. However, requiring such pairs can be costly and inefficient\rev{. This led} to the development of non-contrastive approaches such as BYOL \citep{grill2020bootstrap}, \rev{SimSiam} \citep{chen2021siamsiam}, Barlow Twins \citep{zbontar2021barlow}, VICReg \citep{bardes2022vicreg}, which instead involve objectives based on learning invariant features with respect to certain transformations. The JEPA paradigm extends this non-contrastive approach and may be understood as instead predicting a latent image representation from semantically overlapping views \citep{lecun2022autonomous, Assran_2023_CVPR}.

\textbf{SSL with Geometric Priors} The adaptation of principles of geometric deep learning and its incorporation of geometric priors is still a \rev{growing} subfield of \rev{SSL}. Initial works in this field have explored applications of equivariance to SSL, such as in the contrastive learning setting \citep{dangovski2022equivariant} or as alternatives to invariance-based learning objectives \citep{wang2024understanding}. In another direction, self-supervised and unsupervised approaches have also seen applications to 3D and point-cloud data, as in \cite{xie2020pointcontrast} or \cite{zhang2021self}. 
\rev{Other works have also used} geometric architectures towards representation learning for molecular data, such as \cite{stark2022infomax} and \cite{liu2022pretraining}, as well as for graph-valued data more generally \citep{velickovic2019deep, you2020graphcl, hassani2020contrastive, thakoor2021bootstrapped, xu2026hpjepa}.

\vspace{-0.3cm}
\section{Theoretical Motivation}
\vspace{-0.3cm}
\label{sec:theory}

By the manifold hypothesis \citep{FeffermanMitterNarayanan2016}, we may suppose that our data lies on a compact Riemannian manifold $\mathcal{M}$, though embedded in a high-dimensional ambient space, i.e., $\mathcal{M} \subset \mathbb{R}^{\revred{D}}$. The general objective of self-supervised learning may be understood as learning a $f: \mathbb{R}^{\revred{D}} \to \mathcal{N}$, for $\mathcal{N}$ an appropriate embedding space for the downstream tasks under consideration. For simplicity, we assume $f$ restricts to a smooth embedding on $\mathcal{M}$, so that the image $\mathcal{W} = f(\mathcal{M})$ is itself a closed Riemannian manifold, carrying the metric induced from $\mathcal{N}$. Here, taking $\mathcal{N}$ to be the Euclidean $\mathbb{R}^d$ recovers the setting of LeJEPA \citep{balestriero2025lejepa} while letting $\mathcal{N}$ be the hypersphere $\mathbb{S}^{d}$ recovers the setting of Sphere-JEPA \citep{nicollier2026spherejepa}. However, we may consider more general geometric priors. For instance, we may consider $\mathcal{N}$ to be hyperbolic space or a space equivariant under actions of a given Lie group. 

LeJEPA \citep{balestriero2025lejepa} then ensures non-collapse of the embedding by a distribution matching objective that drives the final embeddings towards an isotropic Gaussian, which they identify as the optimal distribution for downstream linear regression and \revred{$k$-NN} classification. A natural first attempt to generalize LeJEPA to geometric settings would be to consider a Gaussian distribution on $\mathcal{N}$. While there is no canonical analog of a Gaussian distribution such that it is `intrinsic' to an arbitrary geometry, perhaps the most natural extension is given by considering heat flow on a manifold. Recall that the heat equation is given by  $\frac{\partial \mathbf{U}(x,t)}{\partial t} = -\Delta_\mathcal{W}\mathbf{U}(x,t)$, where $\Delta_\mathcal{W}$ is the Laplace-Beltrami operator on the Riemannian manifold $\mathcal{W}$ with metric $g$. We then have the following.

\begin{proposition}[Gaussian as the Solution of Heat Flow] Let $(\mathcal{W},g)$ be a closed, connected Riemannian manifold. The solution to the heat equation on $\mathcal{W}$ with initial condition $\mathbf{U}$ is given by: 
\begin{equation}
    \mathbf{U}(x,t)  = \int_{\mathcal{W}} K_\mathcal{W}(x,y,t) \mathbf{U}(y,\revred{0}) d\operatorname{Vol}_g(\revred{y})
\end{equation}
where $K_\mathcal{W}(x,y,t)$ is the heat kernel and $d\operatorname{Vol}_g$ is the canonical volume form. Explicitly, the heat kernel on $\mathbb{R}^n$ is given by $K_{\mathbb{R}^n}(x,y,t) = \frac{1}{\revred{(4\pi t)}^{n/2}} \exp\left(-\frac{|x-y|^2}{4t}\right)$ and on arbitrary closed Riemannian $\mathcal{W}$, as $t \to 0$, a first order approximation is given by $K_{\mathcal{W}}(x,y,t) \approx \frac{1}{\revred{(4\pi t)}^{n/2}} \exp\left(-\frac{d(x,y)^2}{4t}\right)$, where $d(x,y)$ is the geodesic distance on $\mathcal{W}$. (See \cite{BerlineGetzlerVergne1992} for details).
\label{prop:Gaussian-heat-flow}
\end{proposition}

\Cref{prop:Gaussian-heat-flow} therefore lets us define a `Gaussian' intrinsically on a manifold, as the heat kernel $K_\mathcal{W}(x,\cdot,t)$ centered at $x$ with $t$ in the role of a variance. That identification holds only as $t \to 0$, however, and in exactly that regime the kernel is concentrated near $x$. An intrinsic Gaussian is therefore a \textit{strictly local} object on $\mathcal{W}$. As such, in \rev{practice}, we may consider a local version of LeJEPA's SIGReg as an appropriate analog for the manifold setting. Unfortunately, this naive approach does not preserve LeJEPA's optimality guarantees. For instance, in Sphere-JEPA \citep{nicollier2026spherejepa}, it is shown that for the case $\mathcal{N}=\mathbb{S}^{d}$, the uniform distribution is optimal for downstream linear regression and \revred{$k$-NN} classification tasks, and they propose SUSReg, a distribution matching objective that instead drives towards the uniform distribution. Taking on a heat flow perspective allows us to interpolate between these objectives on an arbitrary $\mathcal{N}$\rev{, of which the sphere is a special case}.  In particular, we then have an entire family of distributions $K_{\mathcal{W}}(\cdot,\cdot,t)$, parameterized by time, that may reasonably be understood as targets for a distribution-matching objective that `lifts' LeJEPA to manifolds. However, this would not consider the entropy maximization property of the Gaussian distribution that has been tied to the success of LeJEPA, in works such as Rectified LpJEPA \citep{kuang2026rectifiedlpjepajointembeddingpredictive}. More generally, an entropy-maximization objective has been shown to be desirable in SSL \citep{liu2022maximum, zhang2024matrix, chakraborty2025entropy}. We recall that the usual Gaussian distribution, under a given variance, is the unique entropy-maximizing distribution in Euclidean space, while the uniform distribution is the unique entropy-maximizing distribution on an arbitrary compact Riemannian manifold. The optimal transport perspective allows us to understand heat flow as gradient flow towards entropy maximization on an arbitrary manifold. 

\begin{proposition}[Heat Flow as Wasserstein Gradient Flow] Let $(\mathcal{W},g)$ be a closed, connected Riemannian manifold. Recall that the differential entropy of a distribution $\rho$ on $\mathcal{W}$ is given by :
\begin{equation}
    h(\rho)=-\int_{\mathcal{W}} \rho(x)\log \rho(x)\,d\operatorname{Vol}_g(x)
\end{equation}

Then in the 2-Wasserstein space of measures, $\mathcal{P}_2(\mathcal{W})$, the sequence of distributions arising as solutions of $\frac{\partial \rho_t}{\partial t} = -\Delta_\mathcal{W}\rho$ constitute gradient flow of the entropy functional in $\mathcal{P}_2{(\mathcal{W})}$. (See \cite{AmbrosioGigliSavare2013HeatFlow} or \cite{JordanKinderlehrerOtto1998} for details).

\label{prop:transport-heat-flow}
\end{proposition}

This sequence is bracketed by its two asymptotic ends: the short-time regime of \Cref{prop:Gaussian-heat-flow}, and the following long-time regime. 

\begin{proposition}[Uniform Distribution as the Solution of Heat Flow] Let $(\mathcal{W},g)$ be a closed, connected Riemannian manifold. Then in the infinite-time regime, heat flow satisfies:
\begin{equation}
    \mathbf{U}(\cdot, t) \xrightarrow[t \to \infty]{} \frac{1}{\operatorname{Vol}_{\revred{g}}(\mathcal{W})} \int_\mathcal{W} \mathbf{U}(\revred{x}, 0) d \operatorname{Vol}_g(x)
\end{equation}
in $L^2(\mathcal{W})$. Equivalently, if $\textbf{U}(\cdot,0)$ is an initial probability distribution, $
\textbf{U}(x,t) \xrightarrow[t \to \infty]{} \frac{1}{\operatorname{Vol}_{\revred{g}}(\mathcal{W})}
$

uniformly in $x$, i.e. converges to the uniform Riemannian volume measure. (See \cite{BerlineGetzlerVergne1992}).
\label{prop:Uniform-heat-flow}
\end{proposition}

As such, under this additional constraint of entropy maximization, the optimal embedding distribution sits at the $t \to \infty$ end of the heat-flow family. That is, we want $f(\mathcal{M}) \subset \mathcal{N}$ to be `maximally spread out' while preserving the relevant geometry of $\mathcal{M}$. \rev{The} simple `geometry preservation' assumption already precludes the final embedding from being Gaussian even in the Euclidean setting $\mathcal{N} = \mathbb{R}^d$. Indeed, a continuous $f$ carries the compact $\mathcal{M}$ to a compact image, and thus driving towards the infinitely-supported Gaussian distributions necessarily causes us to learn an $f$ that `tears' $\mathcal{M}$. Thus, we wish to modify the LeJEPA paradigm such that \textbf{i.} we can preserve the structure of $\mathcal{M}$ and \textbf{ii.} without conditioning on a predetermined $\mathcal{N}$, while \textbf{iii.} still allowing for \rev{an} implementation that is stable and scalable\rev{ (\Cref{sec:cost-locality})}. 
To drive towards a true entropy-maximizing distribution would typically require either a static but globally formulated objective, conditioned on a fixed choice of $\mathcal{N}$ and violating \textbf{ii}, or a local-to-global heat flow style objective exploiting \Cref{prop:transport-heat-flow}, which would then require dynamic re-computations of the distribution matching objective during training, likely violating \textbf{iii}. As most reasonable notions of `geometry preserving’ for \textbf{i} would imply that $f(\mathcal{M})$ should be locally Euclidean, we propose a compromise \textit{locally entropy maximizing} objective, by driving towards the uniform distribution on a Euclidean ball $\mathbb{B}^d$ on neighborhood clusters.

\vspace{-0.3cm}
\section{LocaLeReg}
\vspace{-0.3cm}
\label{sec:LocaLeReg}
We propose \textit{locally latent Euclidean Regularization} (LocaLeReg), a novel regularizer that realizes the local target of \Cref{sec:theory} in practice. To keep the implementation stable and scalable, we obtain it through minor modifications to the sketched isotropic Gaussian regularization (SIGReg) of LeJEPA.

A batch consists of $N$ samples, each presented in $V$ augmented views $x_i^{(v)}$. As is standard in LeJEPA \citep{balestriero2025lejepa}, an encoder is followed by a projector, and it is on the embeddings $z_i^{(v)}\in\mathbb{R}^{d}$ it produces that the pretraining loss is computed, while downstream probes read the encoder outputs.  In the notation of \Cref{sec:theory}, the map $f$ is this composition and $\mathcal{N}$ the space in which the embeddings live. We write $\bar z_i=\tfrac1V\sum_{v=1}^V z_i^{(v)}$ for the mean of a sample's embeddings across views. Recall that LeJEPA's objective is given as $ (1- \lambda) \mathcal{L}_{\rm pred} + \lambda \mathcal{L}_ {\rm SIGReg} $, where $\lambda\in(0,1)$ trades the two terms off and $ \mathcal{L}_{\rm pred}$ is the usual JEPA prediction term,
\begin{equation}
    \mathcal L_{\rm pred}=\frac{1}{NVd}\sum_{i=1}^N\sum_{v=1}^V
 \bigl\|z_i^{(v)}-\bar z_i\bigr\|_2^2.
  \label{eq:lejepa}
\end{equation}
LeJEPA's SIGReg \rev{uses} a hyperspherical version of the Cramér–Wold theorem \citep{CramerWold1936}, which implies that, for $X, Y$, $\mathbb{R}^d$-valued random vectors, 
\begin{equation}
\langle{u}, X\rangle \stackrel{d}{=}\langle{u}, Y\rangle, \forall {u} \in \mathbb{S}^{d-1} \Longleftrightarrow X \stackrel{d}{=} Y 
      \label{eq:cramer-wold}.
\end{equation}
where $ \stackrel{d}{=}$ denotes equality in distribution. This allows us to bring to bear the univariate Epps-Pulley test \citep{EppsPulley1983}, which measures $L^2$ discrepancy between characteristic functions under a weight $w$, which we take throughout to be the Gaussian window $w(t)=e^{-t^2/2}$: 
\begin{equation}
    \operatorname{EP}(X, Y)=N \int_{\mathbb{R}}\left|\hat{\varphi}_X(t)-\varphi_Y(t)\right|^2 w(t) d t
    \label{eq:epps-pulley}.
\end{equation}
Here, $X$ is now a scalar random variable with empirical characteristic function $\hat{\varphi}_X(t)=\frac{1}{N} \sum_{j=1}^N e^{i t x_j}$ while $Y$ is a fixed reference distribution with characteristic function $\varphi_Y(t)=\mathbb{E}\left[e^{i t Y}\right]$. SIGReg thus consists of drawing directions ${u} \sim \operatorname{Unif}\left(\mathbb{S}^{d-1}\right)$, and for each view $v$, matching the empirical distribution of the  
$\bigl\{{u}^{\top} z_{i}^{(v)}\bigr\}_{i=1}^N$ against the standard normal $Y = \mathcal{N}(0,1)$ via Epps-Pulley. 

In the LocaLeJEPA pipeline\rev{ (\Cref{fig:architecture}, with the full algorithm in \Cref{sec:locale-algorithm})}, we begin by constructing local neighborhoods of our embeddings through cosine-similarity $K$-means on the view means $\bar z_i$, held fixed so that no gradient passes through the assignment, giving a clustering $\{C_q\}_{q=1}^{K}$ that is shared across views. Neighborhoods falling below a fixed minimum size are discarded as too sparse to test, leaving the valid set $\mathcal{Q}$ of retained neighborhoods. If it is empty, the objective reverts to the global test on the whole batch. Within a cluster, we consider the residuals of our embeddings
\vspace{-0.1cm}
\begin{equation}
 r_i^{(v)}=z_i^{(v)}-\mu_q^{(v)},
 \qquad \mu_q^{(v)}=\frac{1}{|C_q|}\sum_{j\in C_q}z_j^{(v)}.
 \label{eq:residuals}
\end{equation}

\rev{Centering is motivated by entropy maximization. Differential entropy is invariant under translation, so maximizing it within a neighborhood constrains the shape of that neighborhood's distribution but not its location. Matching the residuals therefore leaves each neighborhood free to sit anywhere in the embedding space, and the neighborhood means are not pulled to a common origin.} We then drive each cluster towards the uniform distribution on a Euclidean ball $\mathbb B^d$. Note that, like the isotropic Gaussian distribution or the uniform distribution on the sphere, this target distribution is radially symmetric and thus amenable to being precomputed for use in the Epps-Pulley test. We may explicitly derive (see \Cref{sec:derivations-characteristic-functions} for details) that the characteristic function of the uniform distribution $Y$ on the ball $\mathbb B^d_R$ of radius $R$, which we set to $R=\sqrt{d+2}$ so that the target has identity covariance and is \revred{normalized} exactly as SIGReg's isotropic Gaussian\rev{ (\Cref{sec:ball-radius})}, is given by \begin{equation}
 \varphi_{Y}(t)=\mathbb E[e^{\mathrm it\langle Y,u\rangle}]
 =2^{d/2}\Gamma\!\left(\frac d2+1\right)
 \frac{J_{d/2}(R|t|)}{(R|t|)^{d/2}},
 \qquad \varphi_{Y}(0)=1,
 \label{eq:ball-cf}
\end{equation} where $J_\nu$ is the Bessel function of the first kind. In this case, our empirical characteristic function is 

\begin{equation}
\widehat\varphi_{q,v,u}(t)=|C_q|^{-1}\sum_{i\in C_q}e^{\mathrm it\langle r_i^{(v)},u\rangle}
\label{eq:emp-cf}.
\end{equation}
\vspace{-0.1cm}


\begin{tcolorbox}[
  enhanced,
  breakable,
  colback=mist,
  colframe=ink,
  boxrule=2.2pt,
  arc=2mm,
  left=10pt,
  right=10pt,
  top=10pt,
  bottom=8pt,
  before skip=10pt,
  after skip=10pt,
  title={\textbf{LocaLeJEPA Objective}},
  fonttitle=\small,
  coltitle=white,
  attach boxed title to top left={
    xshift=0pt,
    yshift=-2.5mm
  },
  boxed title style={
    colback=ink,
    colframe=ink,
    arc=1.5mm,
    boxrule=0pt,
    left=6pt,
    right=6pt,
    top=2pt,
    bottom=2pt
  }
]

LocaLeReg aggregates the local discrepancies over the valid neighborhoods
$\mathcal Q$ and the $V$ augmented views:
\begin{equation}
\mathcal{L}_{\mathrm{LocaLeReg}}
=
\frac{1}{|\mathcal Q|V}
\sum_{q\in\mathcal Q}
\sum_{v=1}^{V}
D_{q,v},
\label{eq:objective}
\end{equation}
where:
\begin{equation}
D_{q,v}
=
|C_q|\,
\mathbb{E}_{u\sim\operatorname{Unif}(\mathbb S^{d-1})}
\int_{\mathbb R}
\left|
\widehat{\varphi}_{q,v,u}(t)
-
\varphi_Y(t)
\right|^2
e^{-t^2/2}\,dt.
\label{eq:local-discrepancy}
\end{equation}


Combining local distribution matching with the prediction objective gives:
\begin{equation}
\hspace{-1.2cm}
\boxed{
\mathcal{L}_{\mathrm{LocaLeJEPA}}
=
(1-\lambda)\mathcal{L}_{\mathrm{pred}}
+
\lambda\mathcal{L}_{\mathrm{LocaLeReg}}
}
\label{eq:LocaLEJEPA-loss}
\end{equation}

\end{tcolorbox}
The number of neighborhoods $K$ controls the discretization of the local regime of \Cref{sec:theory}, as it sets the scale at which the embedded manifold is treated as
Euclidean, and at $K=1$ the objective becomes global, differing from SIGReg in its target distribution and in that centering the residuals leaves the location free. \rev{Each neighborhood's discrepancy is weighted by its size $|C_q|$ (\Cref{eq:local-discrepancy}) and the discrepancies are averaged over neighborhoods (\Cref{eq:objective}). The regularizer therefore scales with the number of samples per neighborhood, about $N/K$, and increasing $K$ at fixed $\lambda$ also lowers the effective regularization strength $\lambda N/K$ (\Cref{sec:effective-strength}).} The construction does require each neighborhood to retain enough residuals for its statistic to be informative, since centering leaves only $|C_q|-1$ free directions and the neighborhoods are formed within the batch.
\begin{figure}[t]
  \centering
  \includegraphics[width=\textwidth]{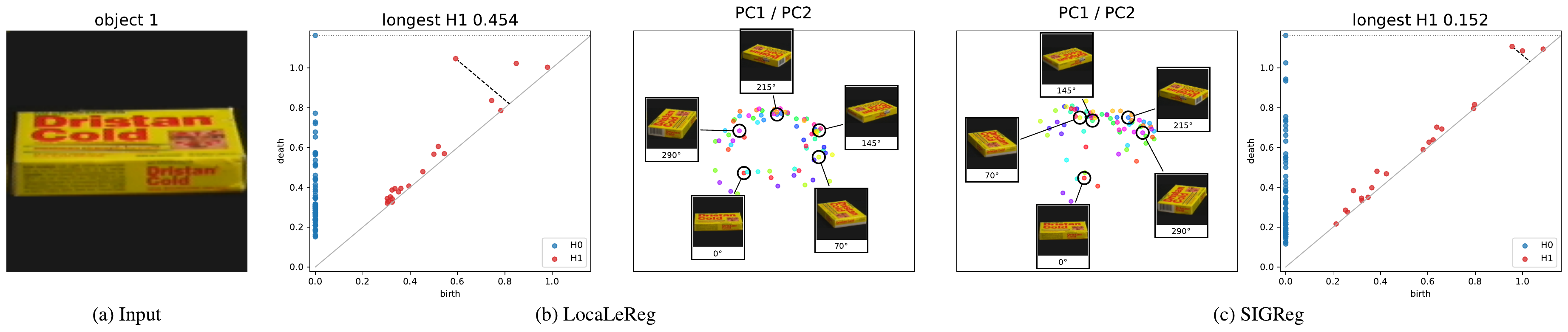}
  \caption{\emph{LocaLeJEPA embeddings better preserve intrinsic topology.} The COIL-100 consists of images of objects being rotated. As such, these images should trace out the circle, $\mathbb{S}^1$, in latent space. As measured by the persistent homology class $H_1$ and PCA visualization, the embeddings more naturally lie on $\mathbb{S}^1$ when using LocaLeReg as opposed to LeJEPA's SIGReg.
  Additional experimental details may be found in \Cref{app:coil_experiment}\rev{ and \Cref{fig:coil-grid}}.
  }
  \label{fig:coil-grid-main}
\end{figure}
\vspace{-0.3cm}

\begin{figure}[h]
\centering
\captionsetup{font=small}
\captionsetup[subfigure]{font=footnotesize}
\setlength{\tabcolsep}{0pt}

\begin{subfigure}[t]{0.235\textwidth}
  \centering
  \includegraphics[width=\linewidth]{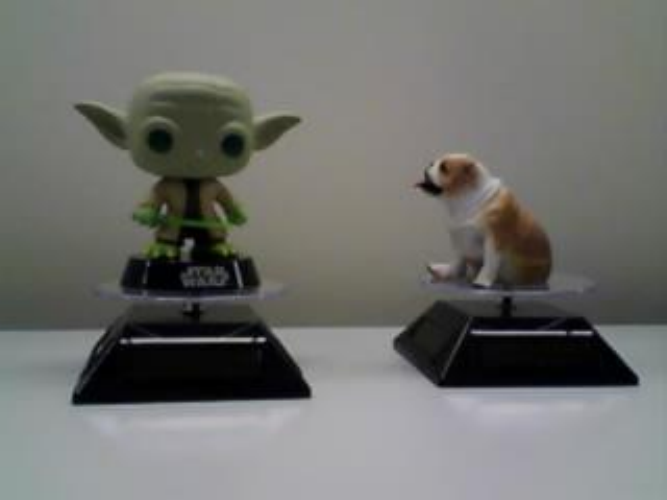}
  \caption{$\theta_0{=}129^\circ$}
  \label{fig:yoda-torus-main-f0a}
\end{subfigure}\hfill
\begin{subfigure}[t]{0.235\textwidth}
  \centering
  \includegraphics[width=\linewidth]{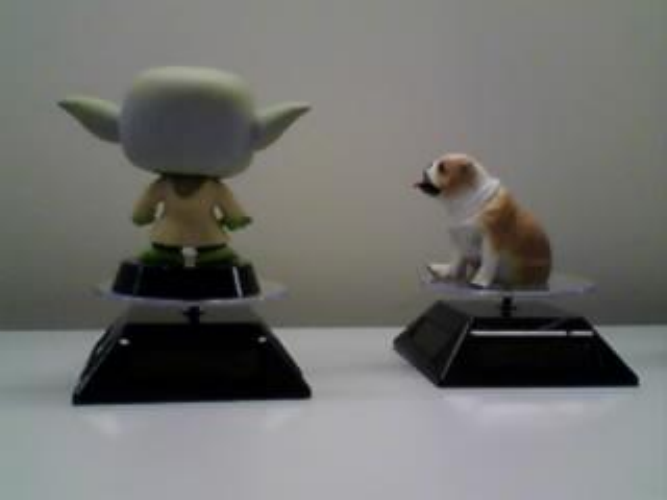}
  \caption{$\theta_0{=}327^\circ$}
  \label{fig:yoda-torus-main-f0b}
\end{subfigure}\hfill
\begin{subfigure}[t]{0.235\textwidth}
  \centering
  \includegraphics[width=\linewidth]{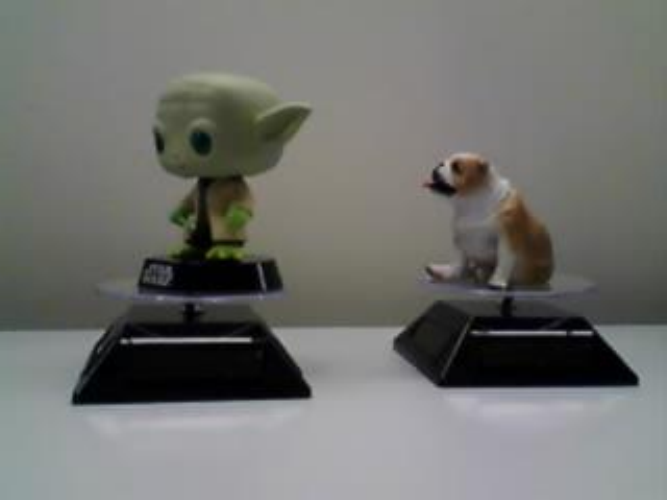}
  \caption{$\theta_1{=}169^\circ$}
  \label{fig:yoda-torus-main-f1a}
\end{subfigure}\hfill
\begin{subfigure}[t]{0.235\textwidth}
  \centering
  \includegraphics[width=\linewidth]{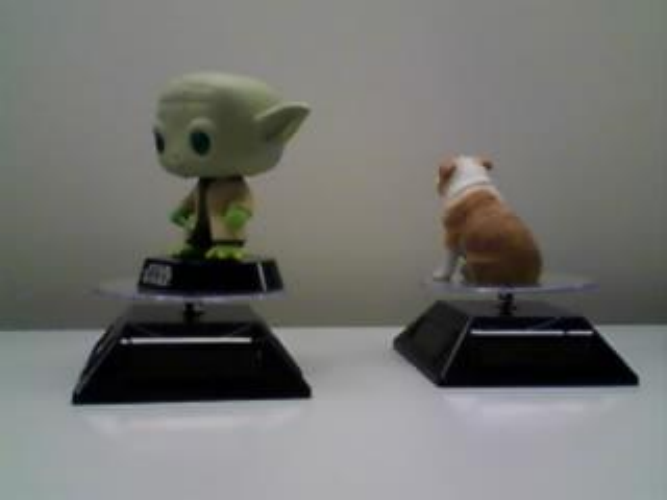}
  \caption{$\theta_1{=}333^\circ$}
  \label{fig:yoda-torus-main-f1b}
\end{subfigure}

\vspace{0.6em}

\begin{subfigure}[t]{0.196\textwidth}
  \centering
  \includegraphics[width=\linewidth]{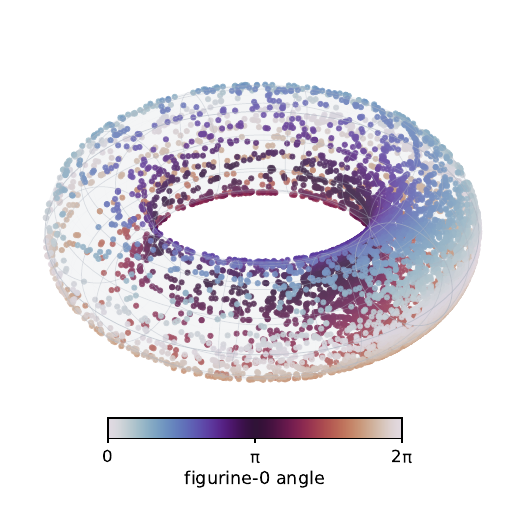}
  \caption{Ground-truth}
  \label{fig:yoda-torus-main-a}
\end{subfigure}\hfill
\begin{subfigure}[t]{0.196\textwidth}
  \centering
  \includegraphics[width=\linewidth]{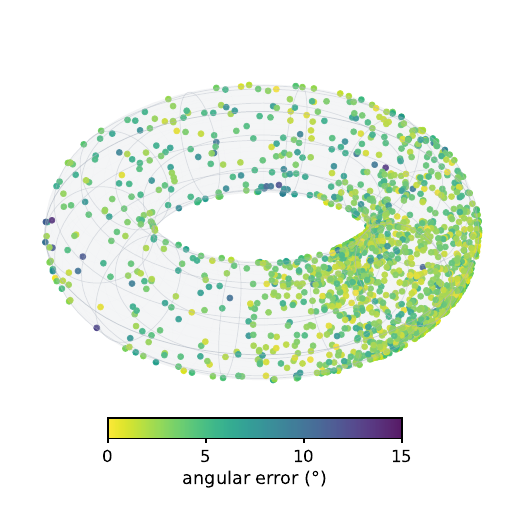}
  \caption{LeJEPA}
  \label{fig:yoda-torus-main-b}
\end{subfigure}\hfill
\begin{subfigure}[t]{0.196\textwidth}
  \centering
  \includegraphics[width=\linewidth]{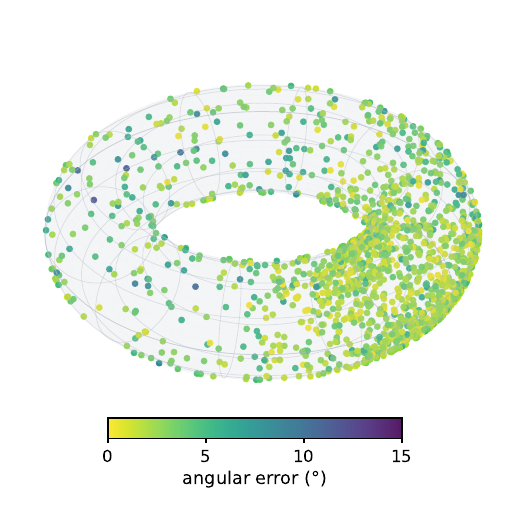}
  \caption{LocaLeJEPA}
  \label{fig:yoda-torus-main-c}
\end{subfigure}\hfill
\begin{subfigure}[t]{0.196\textwidth}
  \centering
  \includegraphics[width=\linewidth]{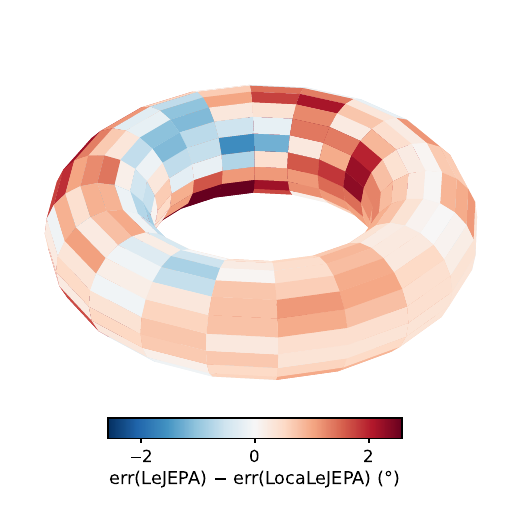}
  \caption{Full pool}
  \label{fig:yoda-torus-main-d}
\end{subfigure}\hfill
\begin{subfigure}[t]{0.196\textwidth}
  \centering
  \includegraphics[width=\linewidth]{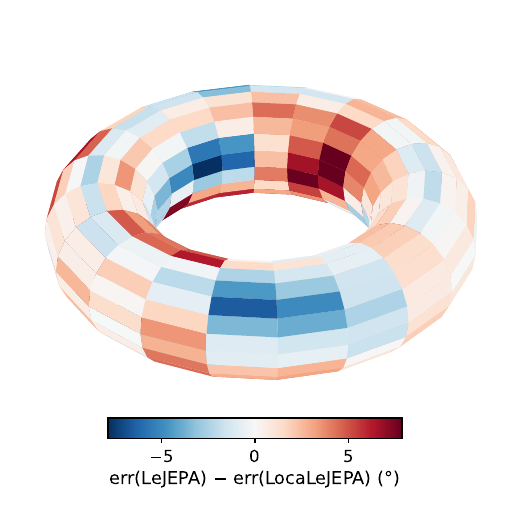}
  \caption{100 frames}
  \label{fig:yoda-torus-main-e}
\end{subfigure}

\caption{\emph{Test-time Topology Preservation.} We perform a controlled experiment to test whether LocaLeJEPA’s topology-preservation extends to datasets unseen during training.  After pretraining on Imagenette \citep{Howard_Imagenette_2019}, we expose it to a dataset consisting of \revred{two independently} rotating figurines \emph{(Top)}, which should thus lie on a latent torus $\mathbb{T}^2$. We consider performance by testing angle prediction of each of the figurines. \emph{(Bottom)}\rev{(e)} Ground-truth torus defined by the two rotation angles. \rev{(f,g)} Reconstructions obtained from frozen LeJEPA (MAE $3.28^\circ$) and LocaLeJEPA \rev{with $\lambda{=}0.1$ and $K{=}2$} ($2.68^\circ$) features, colored by angular error\rev{, on one held-out probe split}. \rev{(h,i)} Regional difference $\mathrm{err}(\mathrm{SIGReg})-\mathrm{err}(\mathrm{LocaLeJEPA})$ using the full probe pool ($+0.60^\circ$, $61\%$ wins) and 100 training frames ($+1.31^\circ$, $57\%$ wins), with red \revred{favoring} LocaLeJEPA. Full grids are in \Cref{app:yoda-torus}\rev{ (\Cref{fig:yoda-torus,fig:yoda-torus-compare,fig:yoda-torus-adv})}.}
\label{fig:yoda-torus-main}
\end{figure}

\vspace{-0.1cm}
\section{Experiments}
\vspace{-0.3cm}
\label{sec:experiments}

LocaLeJEPA is designed to better preserve the intrinsic manifold structure of data, while also being adaptable to exotic embedding geometries. While LocaLeJEPA’s embeddings can be qualitatively shown to better preserve topology (see \Cref{fig:coil-grid-main} and \Cref{fig:yoda-torus-main}\revred{, with the angle-recovery errors in \Cref{tab:yoda-budget-main}}) of data while also adapting to the geometry of the embedding space (see \Cref{fig:pca-teaser} and \Cref{fig:cluster-visualization}), our main experiments focus on the performance of the  LocaLeJEPA framework at scale. In particular, we perform pretraining on ImageNet-1K using 
Euclidean, hyperbolic, and equivariant backbones under the LocaLeJEPA objective. We then evaluate performance through frozen-feature classification on ImageNet-1K in \Cref{subsec:imagenet-classification}, using both linear and \revred{$k$-NN} classification, where it outperforms LeJEPA in all regimes tested. We then consider transferability of the resulting representations across seven downstream datasets in \Cref{subsec:imagenet-transfer-main}, where\revred{, with $K{=}4$,} it again outperforms LeJEPA in all regimes tested. Complete experimental details, including training protocols, checkpoint-selection procedures, evaluation settings, and additional experiments\rev{,} are provided in \Cref{sec:additional-experiments}, with the ImageNet-specific analysis reported in \Cref{app:imagenet-full}\rev{. Ablations over $K$, $\lambda$, and batch size are in \Cref{sec:ablations}}.
\vspace{-0.2cm}

\begin{figure}[t]
\centering

\begin{minipage}[c]{0.5\textwidth}
    \centering
    \includegraphics[width=\linewidth]{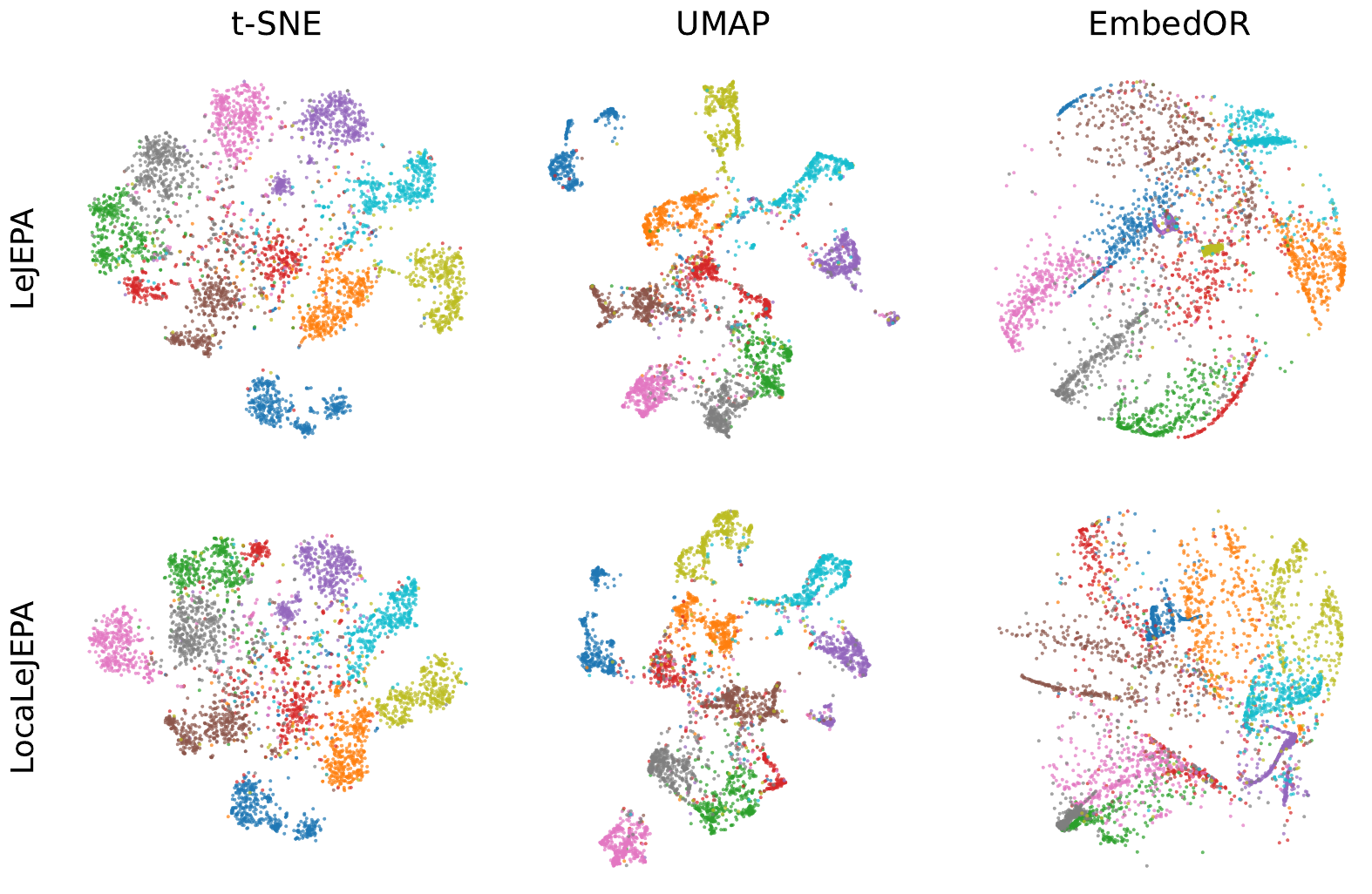}
\end{minipage}\hfill
\begin{minipage}[c]{0.5\textwidth}
    \captionof{figure}{\emph{LocaLeJEPA clusterings are curvature-aware.}
    Two-dimensional projections of frozen hyperbolic ViT-S/8 representations
    trained with LeJEPA (top) and LocaLeJEPA
    (\rev{bottom,} $\lambda{=}0.1$, $K{=}4$). Columns show t-SNE~\citep{vandermaaten08a}, UMAP~\citep{mcinnes2020umapuniformmanifoldapproximation}, and
    EmbedOR~\citep{saidi2026provable}\rev{. Colors} indicate ground-truth classes.
    While Euclidean visualizers show similar layouts, curvature-aware EmbedOR
    \rev{shows a very} different cluster structure.}
    \label{fig:cluster-visualization}
\end{minipage}

\end{figure}
\vspace{-0.1cm}

\subsection{ImageNet-1K Frozen-Feature Classification}
\label{subsec:imagenet-classification}
\begin{wrapfigure}{R}{0.60\columnwidth}
\vspace{-1.0em}

\refstepcounter{table}
\label{tab:yoda-budget-main}

{\footnotesize
{Table~\thetable: \emph{LocaLeJEPA better preserves toroidal structure across probe-label budgets.} Continuing with the setting of the toroidal angle prediction task of \Cref{fig:yoda-torus-main}, we consider performance against a range of global distribution matching objectives: \cite{balestriero2025lejepa}'s SIGReg,  \cite{nicollier2026spherejepa}'s SUSReg, \cite{zhao2026subjepa}'s SubReg, and \cite{kuang2026rectifiedlpjepajointembeddingpredictive}'s RDMReg\rev{ (\Cref{app:objectives})}. \rev{Entries are the wrapped angular MAE ($^\circ$, lower is better), averaged over the two rotation angles, for each number of labeled probe-training frames. The $\pm$ values are standard deviations over ten probe splits. All encoders are pretrained on Imagenette for 800 epochs. LocaLeJEPA uses $\lambda{=}0.1$, and Dyn-$K$ grows $K$ from 1 to 8.} LocaLeJEPA features prove superior for this task across label budgets. Additional details are provided in \Cref{app:yoda}\rev{, and the complete results in \Cref{app:torus-complete-table,app:yoda-lc}}.}
\par}
\vspace{0.3em}
\scriptsize
\setlength{\tabcolsep}{2.2pt}
\renewcommand{\arraystretch}{1.05}
\resizebox{\linewidth}{!}{%
\begin{tabular}{lccccc}
\toprule
Method & 10 & 50 & 100 & 500 & Full \\
\midrule
\rev{Spherical ViT} SUSReg
& $44.6{\pm}4.3$
& $20.2{\pm}2.0$
& $11.0{\pm}0.7$
& $4.91{\pm}.09$
& $3.64{\pm}.04$ \\

ViT SIGReg
& $39.2{\pm}3.6$
& $17.8{\pm}1.1$
& $11.0{\pm}0.4$
& $4.70{\pm}.09$
& $3.30{\pm}.05$ \\

ViT SubReg
& $37.9{\pm}4.4$
& $18.0{\pm}1.7$
& $11.2{\pm}0.5$
& $4.70{\pm}.10$
& $3.40{\pm}.04$ \\

ViT RDMReg
& $39.8{\pm}7.0$
& $19.8{\pm}2.0$
& $10.5{\pm}0.6$
& $5.47{\pm}.08$
& $3.79{\pm}.05$ \\
\midrule
LocaLeJEPA ($K{=}2$)
& $\mathbf{36.0{\pm}3.8}$
& $\mathbf{15.5{\pm}1.2}$
& $\mathbf{9.0{\pm}0.3}$
& $\mathbf{3.80{\pm}.05}$
& $\mathbf{2.67{\pm}.05}$ \\

LocaLeJEPA (Dyn-$K$)
& $37.5{\pm}4.3$
& $17.1{\pm}1.1$
& $10.9{\pm}0.5$
& $4.10{\pm}.10$
& $2.89{\pm}.04$ \\
\bottomrule
\end{tabular}%
}

\vspace{-0.6em}
\end{wrapfigure} 
\textbf{Task.}
We employ ImageNet-1K \revred{\citep{deng2009imagenet}} to test whether the advantages of LocaLeJEPA persist at scale and remain consistent across distinct representation geometries. Following \revred{\citet{balestriero2025lejepa}}, we pretrain three encoder families for 100 epochs: a standard ViT-S/8 with Euclidean outputs, a hyperbolic ViT-S/8 whose representations lie in a Poincar\'e ball of curvature $-0.1$ \revred{\citep{ermolov2022hyperbolic}}, and a $D_4$-equivariant ResNet-50 retaining the regular representation at its output \revred{\citep{klee2023comparison}}. All \rev{three} use four $128\times128$ views, a $2048$-$2048$-$16$ projector, and \rev{256 images per step in four shards of 64, within which neighborhoods are formed. We also pretrain a ViT-L/14 at $238\times238$ pixels, with LeJEPA's published ViT-L settings for both methods (\Cref{app:imagenet-pretrain}). We train LeJEPA both with its published settings and with the optimization settings of LocaLeJEPA, denoted \emph{LeJEPA (our settings)}, so that the methods can also be compared at matched settings. Besides $K{=}4$, we report $K{=}1$, a single neighborhood, and $K{=}1$ at a quarter of the weight ($\lambda{=}0.025$\revred{, \Cref{tab:in1k-indomain-full}}). The latter matches the effective regularization strength of $K{=}4$ at $\lambda{=}0.1$ (\Cref{sec:LocaLeReg}).} After pretraining, the encoder is frozen and evaluated \rev{at its final-epoch checkpoint (\Cref{app:shared-protocol})} with cosine-weighted $k$-NN ($k{=}20$) and linear probing on its 512-dimensional representation. \rev{ViT-S/8 results, except RDMReg, are means over three pretraining seeds.} In the comparison, we also include \emph{LocaLeJEPA-SIGReg}, which keeps the same local partition but replaces centered matching to $\operatorname{Unif}\bigl( \mathbb B^{\rev{d}}_{\sqrt{d+2}}\bigr)$ with Gaussian matching on uncentered clusters. This tests whether locality alone explains the gains, without disentangling centering from reference-distribution choice. Complete optimization details, $k$-NN sweeps, top-5 results, and ablations over $K$, $\lambda$, and the local reference distribution are reported in \Cref{app:imagenet-full}.

\begin{table}[t]
\revstart
\centering
\caption{\emph{ImageNet-1K Frozen-feature Evaluation across Representation Geometries.} ImageNet classification and seven-dataset transfer for the Euclidean, hyperbolic and $D_4$-equivariant encoders at 128 px, and ViT-L/14 at 238 px. Final-epoch checkpoints and one linear-probe seed, with the learning rate selected on a held-out slice of the training set. ViT-S/8 rows other than RDMReg are mean $\pm$ SD over three pretraining seeds, and all other rows are single runs. \emph{LeJEPA} uses LeJEPA's published settings, and \emph{LeJEPA (our settings)} uses the optimization settings of LocaLeJEPA. Full sweeps in \Cref{tab:in1k-indomain-full}. \textbf{Bold} denotes the best result within each encoder family and metric.}
\label{tab:in1k-main}
\tiny
\setlength{\tabcolsep}{4pt}
\resizebox{\linewidth}{!}{\begin{tabular}{l cc c c}
\toprule
 & \multicolumn{2}{c}{ImageNet-1K linear (\%)} & $k$-NN (\%) & Transfer (\%) \\
\cmidrule(lr){2-3}\cmidrule(lr){4-4}\cmidrule(lr){5-5}
Method & top-1($\uparrow$) & top-5($\uparrow$) & $k{=}20$($\uparrow$) & 7-dataset avg($\uparrow$) \\
\midrule
\multicolumn{5}{l}{\emph{Plain ViT-S/8 (Euclidean, 21.7M)}} \\ 
Rectified LpJEPA (RDMReg) & $25.7$ & $48.5$ & $12.6$ & $45.0$ \\ 
LeJEPA & $44.6 \pm 0.23$ & $70.8 \pm 0.17$ & $32.8 \pm 0.33$ & $60.5 \pm 0.52$ \\ 
LeJEPA (our settings) & $51.6 \pm 0.28$ & $77.3 \pm 0.36$ & $38.9 \pm 0.09$ & $64.3 \pm 0.13$ \\ 
LocaLeJEPA ($\lambda{=}0.1$, $K{=}1$) & $59.7 \pm 0.22$ & $83.5 \pm 0.09$ & $45.4 \pm 0.22$ & $70.3 \pm 0.04$ \\ 
LocaLeJEPA ($\lambda{=}0.1$, $K{=}4$) & $\mathbf{62.1 \pm 0.27}$ & $\mathbf{85.1 \pm 0.16}$ & $\mathbf{50.1 \pm 0.12}$ & $\mathbf{72.3 \pm 0.29}$ \\ 
\midrule
\multicolumn{5}{l}{\emph{Hyperbolic ViT-S/8 (Poincar\'e ball, 21.7M)}} \\ 
Rectified LpJEPA (RDMReg) & $26.0$ & $48.7$ & $13.1$ & $44.6$ \\ 
LeJEPA & $44.0$ & $69.8$ & $32.6$ & $60.3$ \\ 
LeJEPA (our settings) & $51.6$ & $77.5$ & $39.4$ & $64.1$ \\ 
LocaLeJEPA ($\lambda{=}0.1$, $K{=}1$) & $59.5$ & $83.4$ & $45.7$ & $70.1$ \\ 
LocaLeJEPA ($\lambda{=}0.1$, $K{=}4$) & $61.9$ & $84.9$ & $50.6$ & $72.1$ \\ 
LocaLeJEPA ($\lambda{=}0.1$, $K{=}1{\to}10$) & $\mathbf{62.2}$ & $\mathbf{85.1}$ & $\mathbf{50.9}$ & $\mathbf{72.6}$ \\ 
LocaLeJEPA-SIGReg ($\lambda{=}0.1$, $K{=}4$) & $48.2$ & $73.8$ & $35.4$ & $62.8$ \\ 
\midrule
\multicolumn{5}{l}{\emph{$D_4$-equivariant ResNet-50 (24.4M)}} \\ 
Rectified LpJEPA (RDMReg) & $31.3$ & $56.9$ & $23.4$ & $48.1$ \\ 
LeJEPA & $43.4$ & $70.7$ & $34.5$ & $54.7$ \\ 
LeJEPA (our settings) & $48.8$ & $75.7$ & $39.7$ & $58.8$ \\ 
LocaLeJEPA ($\lambda{=}0.1$, $K{=}1$) & $52.4$ & $79.0$ & $43.9$ & $60.3$ \\ 
LocaLeJEPA ($\lambda{=}0.1$, $K{=}4$) & $\mathbf{55.2}$ & $80.6$ & $\mathbf{47.5}$ & $\mathbf{63.0}$ \\ 
LocaLeJEPA ($\lambda{=}0.1$, $K{=}1{\to}10$) & $55.1$ & $\mathbf{80.8}$ & $47.4$ & $62.5$ \\ 
LocaLeJEPA-SIGReg ($\lambda{=}0.1$, $K{=}4$) & $43.2$ & $69.5$ & $34.8$ & $52.4$ \\ 
\midrule
\multicolumn{5}{l}{\emph{Plain ViT-L/14 (Euclidean, 303.2M, LeJEPA's ViT-L settings)}} \\ 
LeJEPA & $54.7$ & $79.9$ & $46.0$ & $61.3$ \\ 
LocaLeJEPA ($\lambda{=}0.05$, $K{=}1$) & $54.8$ & $80.1$ & $46.5$ & $60.1$ \\ 
LocaLeJEPA ($\lambda{=}0.05$, $K{=}4$) & $\mathbf{61.0}$ & $\mathbf{84.2}$ & $\mathbf{47.4}$ & $\mathbf{65.5}$ \\ 
\bottomrule
\end{tabular}}
\end{table}

\noindent\textbf{Results.} We consider LocaLeJEPA with both fixed numbers of local clusters, as well as a dynamic approach, $K{=}1{\to}10$, where we increase the number of clusters every ten training epochs, \rev{following} our heat-flow heuristic. LocaLeJEPA \rev{with $K{=}4$} improves frozen-feature recognition across all three representation geometries (\Cref{tab:in1k-main}), with gains of \rev{$11.9$--$17.9$} points in linear top-1 accuracy and \rev{$13.0$--$17.9$} points in $k$-NN over LeJEPA\rev{. Over LeJEPA (our settings), the gains are $6.4$--$10.5$ and $7.8$--$11.2$ points, respectively. The dynamic schedule ($K{=}1{\to}10$), run on the hyperbolic and equivariant encoders, adds $10.7$ and $6.2$ points in linear top-1 accuracy and $11.5$ and $7.7$ in $k$-NN over LeJEPA (our settings), within $0.4$ points of $K{=}4$. On ViT-L/14, where both methods use LeJEPA's published ViT-L settings, LocaLeJEPA adds $6.3$ points in linear top-1 accuracy and $1.3$ in $k$-NN.} The \rev{$k$-NN gain over LeJEPA (our settings), at least $7.8$ points on every $128$-pixel encoder,} is particularly informative because this readout operates directly on the frozen feature metric, and is consistent with the stronger local semantic organization quantified in \Cref{app:knn-neighborhood}\rev{ and \Cref{fig:knn-retrieval}}. 

\begin{table}[t]
\revstart
\centering
\caption{\emph{Linear Transfer across Seven Downstream Datasets.} Frozen-feature top-1 accuracy (\%), where \emph{Avg} is the unweighted mean. One linear-probe seed. ViT-S/8 rows are mean $\pm$ SD over three pretraining seeds, and all other rows are single runs. \textbf{Bold} denotes the best mean within each encoder family and column.}
\label{tab:in1k-transfer-main}
\small
\setlength{\tabcolsep}{4pt}
\resizebox{\linewidth}{!}{\begin{tabular}{l ccccccc c}
\toprule
Method & DTD & Aircraft & CIFAR-10 & CIFAR-100 & Flowers & Food & Pets & Avg \\
\midrule
\multicolumn{9}{l}{\emph{Plain ViT-S/8 (Euclidean, 21.7M)}} \\ 
LeJEPA & $60.0{\pm}0.6$ & $28.4{\pm}0.5$ & $86.1{\pm}0.5$ & $62.6{\pm}1.3$ & $75.4{\pm}0.4$ & $51.0{\pm}0.8$ & $59.8{\pm}1.1$ & $60.5{\pm}0.5$ \\ 
LeJEPA (our settings) & $64.1{\pm}0.6$ & $30.3{\pm}0.9$ & $88.7{\pm}0.2$ & $66.0{\pm}0.2$ & $78.9{\pm}0.2$ & $56.4{\pm}0.3$ & $65.9{\pm}0.5$ & $64.3{\pm}0.1$ \\ 
LocaLeJEPA ($\lambda{=}0.1$, $K{=}1$) & $69.3{\pm}0.3$ & $38.4{\pm}0.4$ & $91.7{\pm}0.3$ & $71.5{\pm}0.9$ & $83.7{\pm}0.2$ & $63.5{\pm}0.1$ & $73.6{\pm}0.8$ & $70.3{\pm}0.0$ \\ 
LocaLeJEPA ($\lambda{=}0.1$, $K{=}4$) & $\mathbf{69.4{\pm}1.5}$ & $\mathbf{42.1{\pm}0.6}$ & $\mathbf{92.2{\pm}0.3}$ & $\mathbf{73.6{\pm}0.4}$ & $\mathbf{85.3{\pm}0.3}$ & $\mathbf{66.0{\pm}0.1}$ & $\mathbf{77.3{\pm}1.8}$ & $\mathbf{72.3{\pm}0.3}$ \\ 
\midrule
\multicolumn{9}{l}{\emph{Hyperbolic ViT-S/8 (Poincar\'e ball, 21.7M)}} \\ 
LeJEPA & $61.5$ & $28.2$ & $85.9$ & $62.6$ & $75.2$ & $50.7$ & $58.0$ & $60.3$ \\ 
LeJEPA (our settings) & $60.8$ & $31.4$ & $88.6$ & $66.1$ & $78.7$ & $56.7$ & $66.7$ & $64.1$ \\ 
LocaLeJEPA ($\lambda{=}0.1$, $K{=}1$) & $68.8$ & $37.5$ & $92.0$ & $71.4$ & $83.7$ & $63.8$ & $73.3$ & $70.1$ \\ 
LocaLeJEPA ($\lambda{=}0.1$, $K{=}4$) & $69.1$ & $41.5$ & $92.3$ & $72.7$ & $85.3$ & $65.3$ & $\mathbf{78.6}$ & $72.1$ \\ 
LocaLeJEPA ($\lambda{=}0.1$, $K{=}1{\to}10$) & $\mathbf{70.1}$ & $\mathbf{42.5}$ & $\mathbf{92.6}$ & $\mathbf{73.3}$ & $\mathbf{85.6}$ & $\mathbf{65.9}$ & $78.5$ & $\mathbf{72.6}$ \\ 
LocaLeJEPA-SIGReg ($\lambda{=}0.1$, $K{=}4$) & $57.3$ & $30.8$ & $87.9$ & $65.1$ & $77.8$ & $55.9$ & $65.0$ & $62.8$ \\ 
\midrule
\multicolumn{9}{l}{\emph{$D_4$-equivariant ResNet-50 (24.4M)}} \\ 
LeJEPA & $55.4$ & $23.7$ & $86.2$ & $60.6$ & $63.9$ & $40.6$ & $52.2$ & $54.7$ \\ 
LeJEPA (our settings) & $59.5$ & $25.7$ & $88.9$ & $63.9$ & $67.4$ & $45.5$ & $61.1$ & $58.8$ \\ 
LocaLeJEPA ($\lambda{=}0.1$, $K{=}1$) & $60.2$ & $25.7$ & $90.2$ & $65.9$ & $67.8$ & $47.5$ & $65.0$ & $60.3$ \\ 
LocaLeJEPA ($\lambda{=}0.1$, $K{=}4$) & $61.3$ & $\mathbf{29.3}$ & $\mathbf{91.2}$ & $\mathbf{69.1}$ & $\mathbf{71.5}$ & $\mathbf{49.5}$ & $69.1$ & $\mathbf{63.0}$ \\ 
LocaLeJEPA ($\lambda{=}0.1$, $K{=}1{\to}10$) & $\mathbf{61.5}$ & $27.4$ & $90.9$ & $67.7$ & $70.4$ & $49.4$ & $\mathbf{69.9}$ & $62.5$ \\ 
LocaLeJEPA-SIGReg ($\lambda{=}0.1$, $K{=}4$) & $48.4$ & $25.5$ & $88.8$ & $60.6$ & $52.5$ & $25.9$ & $65.3$ & $52.4$ \\ 
\midrule
\multicolumn{9}{l}{\emph{Plain ViT-L/14 (Euclidean, 303.2M, LeJEPA's ViT-L settings)}} \\ 
LeJEPA & $62.3$ & $20.1$ & $90.5$ & $66.2$ & $68.2$ & $57.2$ & $64.8$ & $61.3$ \\ 
LocaLeJEPA ($\lambda{=}0.05$, $K{=}1$) & $61.3$ & $17.3$ & $90.4$ & $65.9$ & $66.8$ & $56.3$ & $62.6$ & $60.1$ \\ 
LocaLeJEPA ($\lambda{=}0.05$, $K{=}4$) & $\mathbf{66.0}$ & $\mathbf{26.1}$ & $\mathbf{92.3}$ & $\mathbf{70.6}$ & $\mathbf{74.0}$ & $\mathbf{62.8}$ & $\mathbf{66.4}$ & $\mathbf{65.5}$ \\ 
\bottomrule
\end{tabular}}
\end{table}

{\revredstart On ImageNet-1K, training LeJEPA with our optimization settings already improves it by $6.1$ / $7.1$ points ($k$-NN / linear, ViT-S/8). Against this stronger baseline, LocaLeJEPA still adds $11.2$ / $10.5$, $11.2$ / $10.3$ and $7.8$ / $6.4$ points on ViT-S/8, hyperbolic ViT-S/8 and the $D_4$-equivariant ResNet-50 (\Cref{tab:in1k-main,tab:in1k-indomain-full}). With a single neighborhood ($K{=}1$), our centered ball target alone adds $6.5$ / $8.1$, $6.3$ / $7.9$ and $4.2$ / $3.5$ points over this baseline. Dividing the batch into $K{=}4$ neighborhoods adds a further $4.6$ / $2.5$, $4.9$ / $2.4$ and $3.6$ / $2.9$. Since a larger $K$ also lowers the effective regularization strength (\Cref{sec:LocaLeReg}), we compare $K{=}4$ with $K{=}1$ at a quarter of the weight, which has the same strength. At matched strength, locality still adds $2.3$ / $2.3$ on ViT-S/8 (positive for each of the three seeds) and $2.6$ / $2.0$ and $3.0$ / $2.8$ on the other two encoders (\Cref{tab:in1k-indomain-full}). Within the same neighborhoods, replacing the centered ball target with Gaussian matching on uncentered clusters (LocaLeJEPA-SIGReg) costs $15.2$ / $13.7$ points on the hyperbolic and $12.7$ / $12.0$ on the equivariant encoder (\Cref{tab:in1k-main}). The resulting model falls below LeJEPA (our settings). The gains therefore come from the centered local target, our local entropy-maximizing reference, and locality alone does not explain them. On ViT-L/14, a single neighborhood accounts for only $0.5$ / $0.1$ of the $1.3$ / $6.3$ points that LocaLeJEPA adds.\par}

\vspace{-0.2cm}
\subsection{Transfer to Downstream Datasets}
\vspace{-0.2cm}
\label{subsec:imagenet-transfer-main}
\textbf{Task.}
We next test whether the ImageNet-1K gains persist beyond the pretraining distribution. Keeping each encoder frozen, we fit a single linear classifier on standardized features for seven downstream datasets spanning textures, fine-grained recognition, and generic object classification: DTD \revred{\citep{cimpoi14describing}}, FGVC-Aircraft \revred{\citep{maji13fine-grained}}, CIFAR-10 and CIFAR-100 \revred{\citep{krizhevsky2009learning}}, Flowers-102 \revred{\citep{nilsback2008automated}}, Food-101 \revred{\citep{bossard14}}, and Oxford-IIIT Pets \revred{\citep{parkhi12a}}. The probe learning rate is selected on a $10\%$ holdout from the training split, and top-1 accuracy is reported on the official test set. Full per-configuration results and protocol details are provided in \Cref{app:imagenet-full} and \Cref{app:imagenet-protocol}.

\noindent\textbf{Results.}
The ImageNet-1K gains \rev{also hold} beyond the pretraining distribution
(\Cref{tab:in1k-transfer-main}). Across the Euclidean, hyperbolic, and
equivariant encoder families, LocaLeJEPA \rev{with $K{=}4$} improves the seven-dataset average by
\rev{$8.3$--$11.8$} percentage points over LeJEPA\rev{, and by $4.2$--$8.0$ over LeJEPA (our settings). All three encoders improve
on all seven datasets against both}, with the largest gains \rev{over LeJEPA} on fine-grained recognition tasks
such as Pets and \rev{Food-101. The dynamic schedule's transfer average is $0.5$ points above fixed $K{=}4$ on the hyperbolic and $0.5$ points below it on the equivariant encoder. On ViT-L/14, LocaLeJEPA improves on all seven datasets and raises the seven-dataset average by $4.1$ points. The benefit of local distribution matching extends beyond ImageNet-1K classification, across the representation geometries and downstream domains tested.}

\FloatBarrier

\Needspace{8\baselineskip}
\vspace{-0.3cm}
\section{Conclusion}
\vspace{-0.3cm}
\revred{LocaLeJEPA} brings a manifold-theoretic perspective to the JEPA paradigm, introducing the heuristic that self-supervised training should behave like heat flow. \rev{To implement this heuristic}, we develop the novel LocaLeReg regularizer, based upon a locally Euclidean entropy-maximizing objective. We find that both qualitatively and quantitatively, LocaLeJEPA produces embeddings that better preserve intrinsic semantic structure and can adapt to the geometry of the representation space, while being stable and scalable. \rev{Comparisons at matched optimization settings, on four encoders up to ViT-L/14 and over three pretraining seeds on ViT-S/8, show that the gains do not come from the training settings. Within the same neighborhoods, an uncentered Gaussian target loses all of the gain over LeJEPA (our settings), so locality alone does not explain it.} While we instantiate our framework on image data in the current work, we believe LocaLeJEPA represents a general-use framework that allows for a principled extension of the JEPA paradigm to settings that require handling non-Euclidean data or \rev{using} non-Euclidean representations.

\FloatBarrier
\bibliography{iclr_ref} 
\bibliographystyle{chapters/iclr2027_conference}
\clearpage
\appendix
%

\providecolor{Green}{RGB}{0,128,0}
\providecolor{Red}{RGB}{200,0,0}
\providecommand{\cmark}{\textcolor{Green}{\ding{51}}}
\providecommand{\xmark}{\textcolor{Red}{\ding{55}}}

\providecolor{Top1}{HTML}{0072B2}
\providecolor{Top2}{HTML}{E69F00}
\providecolor{Top3}{HTML}{CC79A7}
\providecolor{ufill}{HTML}{E1D5E7}
\providecolor{uborder}{HTML}{9673A6}
\providecolor{vfill}{HTML}{D5E8D4}
\providecolor{vborder}{HTML}{82B366}

\sisetup{
    mode=text,
    table-alignment-mode=none,
    reset-text-family=false,
    reset-text-series=false,
    reset-text-shape=false,
    table-number-alignment=center,
    table-align-comparator=false,
    table-align-text-pre=false,
    table-align-text-post=false,
    round-mode=uncertainty,
    round-precision=3
}

\providecolor{proofbar}{RGB}{225,225,225}
\providecolor{theorembar}{RGB}{128,0,128}
\providecolor{propositionbar}{RGB}{0,100,0}
\providecolor{definitionbar}{RGB}{0,0,200}
\providecolor{lemmabar}{RGB}{255,140,0}
\providecolor{corollarybar}{RGB}{204,153,255}
\providecolor{remarkbar}{RGB}{0,139,139}
\providecolor{assumptionbar}{RGB}{178,34,34}

%
%

\newenvironment{restatementframe}[1]{%
  \ifpaperbars
  \begin{mdframed}[
    hidealllines=true,
    leftline=true,
    linecolor=#1,
    linewidth=2pt,
    innerleftmargin=6pt,
    innerrightmargin=0pt,
    innertopmargin=2pt,
    innerbottommargin=2pt
  ]%
  \else
    \par\addvspace{\topsep}%
  \fi
}{%
  \ifpaperbars
    \end{mdframed}%
  \else
    \par\addvspace{\topsep}%
  \fi
}

\newenvironment{restatedtheorem}[2][]{%
  \begin{restatementframe}{theorembar}%
  \noindent\textbf{Theorem~\ref{#2}}%
  \if\relax\detokenize{#1}\relax
    \textbf{.}%
  \else
    \textbf{ (#1).}%
  \fi
  \enspace\itshape\ignorespaces
}{%
  \end{restatementframe}
}

\newenvironment{restatedproposition}[2][]{%
  \begin{restatementframe}{propositionbar}%
  \noindent\textbf{Proposition~\ref{#2}}%
  \if\relax\detokenize{#1}\relax
    \textbf{.}%
  \else
    \textbf{ (#1).}%
  \fi
  \enspace\itshape\ignorespaces
}{%
  \end{restatementframe}
}

\newenvironment{restatedlemma}[2][]{%
  \begin{restatementframe}{lemmabar}%
  \noindent\textbf{Lemma~\ref{#2}}%
  \if\relax\detokenize{#1}\relax
    \textbf{.}%
  \else
    \textbf{ (#1).}%
  \fi
  \enspace\itshape\ignorespaces
}{%
  \end{restatementframe}
}

\newenvironment{restatedcorollary}[2][]{%
  \begin{restatementframe}{corollarybar}%
  \noindent\textbf{Corollary~\ref{#2}}%
  \if\relax\detokenize{#1}\relax
    \textbf{.}%
  \else
    \textbf{ (#1).}%
  \fi
  \enspace\itshape\ignorespaces
}{%
  \end{restatementframe}
}

\providecommand*{\theoremautorefname}{Theorem}
\providecommand*{\lemmaautorefname}{Lemma}
\providecommand*{\propositionautorefname}{Proposition}
\providecommand*{\definitionautorefname}{Definition}
\providecommand*{\corollaryautorefname}{Corollary}
\providecommand*{\remarkautorefname}{Remark}
\providecommand*{\exampleautorefname}{Example}
\providecommand*{\assumptionautorefname}{Assumption}

\makeatletter
\@ifundefined{ifshowcomments}{\newboolean{showcomments}}{}
\makeatother
\providecommand{\fr}[1]{\textcolor{red}{\textbf{[FR: #1]}}}

\begin{center}
  {\Large\bfseries \papertitle\par}
  \vspace{0.5em}
  {\large Supplementary Material\par}
\end{center}
\phantomsection
\addcontentsline{toc}{part}{Supplementary Material}
\etocsetnexttocdepth{subsubsection}
\etocsettocstyle{\subsection*{Contents}}{}
\localtableofcontents
\vspace{1em}
\newpage

\section{Additional Implementation Details}
\label{app:implementation-details}

This appendix derives the uniform ball distribution, the reference used by LocaLeReg. \Cref{sec:derivations-characteristic-functions} computes its characteristic function and one-dimensional marginals, including equivalent hypergeometric forms, and \Cref{sec:locale-algorithm} describes how they enter the local distribution-matching objective and summarizes the training procedure.

\subsection{Characteristic Function of the Uniform Ball}
\label{sec:derivations-characteristic-functions}

Throughout, $d$ denotes the output dimension of the projector, i.e. the space in which the regularizer acts.
For $c\in\mathbb{R}^d$ and $R>0$, we denote the ball of radius $R$ centered at $c$ as:   
\begin{equation}
     \mathbb B_R^d(c)
    =
    \{x\in\mathbb{R}^d:\|x-c\|_2\leq R\}
\end{equation}
and write $\mathbb B_R^d=\mathbb B_R^d(0)$. The volume of the unit ball is:
\begin{equation}
     V_d
    =
    \frac{\pi^{d/2}}{\Gamma(d/2+1)}
\end{equation}
where $\Gamma$ is the Gamma function. As a consequence, the uniform distribution over the ball has the following density function.

\begin{lemma}[Uniform Ball]
A random variable uniformly distributed over the ball, denoted $X\sim\operatorname{Unif}(\mathbb B_R^d)$, has density:
\begin{equation}
    f_X(x) = \frac{1}{V_dR^d}\mathbf{1}_{\mathbb B_R^d}(x)
\end{equation}
\end{lemma}

We recall the following definition, which we then apply to the uniform ball distribution.
\begin{definition}[Characteristic Function]
Given a random variable $X$, its characteristic function is defined as
\begin{equation}
     \varphi_X(t)=\mathbb{E}[e^{it^\top X}],
    \qquad t\in\mathbb{R}^d
\end{equation}
\end{definition}

\begin{proposition}[Bessel form of the Characteristic Function]
\label{prop:ball-cf-bessel}
Let $X\sim\operatorname{Unif}(\mathbb B_R^d)$. Then:
\begin{equation}
    \varphi_X(t)
    =
    2^{d/2}\Gamma\!\left(\frac d2+1\right)
    \frac{J_{d/2}(R\|t\|_2)}
         {(R\|t\|_2)^{d/2}}
    \label{eq:ball-cf-bessel}
\end{equation}
and at $t=0$ the expression is defined by continuity as $\varphi_X(0)=1$.
\end{proposition}

The proof of \Cref{prop:ball-cf-bessel} follows from the next two lemmas. \Cref{lem:poisson-integral} is a known fact about the Bessel functions of the first kind, hence we state it here without a proof.

\begin{lemma}[Integral Representation of the Characteristic Function]
\label{lem:ball-cf-integral}
Let $X\sim\operatorname{Unif}(\mathbb B_R^d)$ and $z=R\|t\|_2$. Then:
\begin{equation}
    \varphi_X(t)
    =
    \frac{\Gamma(d/2+1)}
         {\sqrt{\pi}\,\Gamma((d+1)/2)}
    \int_{-1}^{1}
        e^{izu}(1-u^2)^{(d-1)/2}\,du
    \label{eq:ball-cf-integral}
\end{equation}
\end{lemma}

\begin{proof}
Rotational invariance implies that:
\begin{equation}
    \varphi_X(t) = \int_{\mathbb B_R^d} f_{X}(x) e^{it^\top x}\,dx
\end{equation}
depends only on $\|t\|_2$. We may therefore \rev{assume that} $t$ \rev{points along} the first coordinate axis and write $t=\tau e_1$, where $\tau=\|t\|_2$. Hence:
\begin{equation}
      \varphi_{X}(t) = \int_{\mathbb B_R^d} f_{X}(x) e^{i\tau x_1}\,dx
\end{equation}
At a fixed first coordinate $x_1=s\in[-R,R]$, the cross-section of $\mathbb B_R^d$ is a $(d-1)$-dimensional ball of radius $\sqrt{R^2-s^2}$, since $s^2 + \|y\|^2 \leq R^2$ implies $\|y\| \leq \sqrt{R^2 - s^2}$. Integrating over these cross-sections gives:
\begin{align*}
    \varphi_X(t)
    &= \frac{1}{V_dR^d}\int_{-R}^R e^{i\tau s}  \left[\int_{\|y\|\leq \sqrt{R^2-s^2}}dy \right] ds \\
    &=
    \frac{V_{d-1}}{V_dR^d}
    \int_{-R}^{R}
        e^{i\tau s}(R^2-s^2)^{(d-1)/2}\,ds \\
    &=
    \frac{\Gamma(d/2+1)}
         {\sqrt{\pi}\,\Gamma((d+1)/2)}
    \int_{-1}^{1}
        e^{izu}(1-u^2)^{(d-1)/2}\,du,
\end{align*}
where the third equality uses $s=Ru$ and $z=R\|t\|_2$.
\end{proof}

\begin{lemma}[Poisson Integral]
\label{lem:poisson-integral}
For $\nu>-1/2$ and $z>0$, the Bessel function of the first kind satisfies:
\begin{equation}
    J_\nu(z)
    =
    \frac{(z/2)^\nu}
         {\sqrt{\pi}\,\Gamma(\nu+1/2)}
    \int_{-1}^{1}
        e^{izu}(1-u^2)^{\nu-1/2}\,du
\end{equation}
\end{lemma}

\begin{proof}[Proof of \Cref{prop:ball-cf-bessel}]
Substitute $\nu=d/2$ from \Cref{lem:poisson-integral} into \eqref{eq:ball-cf-integral}.
\end{proof}

This characteristic function admits two hypergeometric representations as alternative expressions.

\begin{proposition}[Hypergeometric forms]
\label{prop:ball-cf-hypergeometric}
Let $X\sim\operatorname{Unif}(\mathbb B_R^d)$. Then:
\begin{equation}
    \varphi_X(t)
    =
    {}_0F_1\!\left(
        ;\frac d2+1;-\frac{R^2\|t\|_2^2}{4}
    \right)
    \label{eq:ball-cf-0f1}
\end{equation}
and:
\begin{equation}
    \varphi_X(t)
    =
    e^{-iR\|t\|_2}
    {}_1F_1\!\left(
        \frac{d+1}{2};d+1;2iR\|t\|_2
    \right)
    \label{eq:ball-cf-1f1}
\end{equation}
\end{proposition}

\begin{proof}
\emph{Expression in terms of ${}_0F_1$.}
Using the rising factorial $(a)_k=\Gamma(a+k)/\Gamma(a)$, with $(a)_0=1$, the Bessel series can be written as:
\begin{align*}
    J_\nu(z)
    &=
    \sum_{k=0}^{\infty}
    \frac{(-1)^k}{k!\,\Gamma(k+\nu+1)}
    \left(\frac z2\right)^{2k+\nu}\\
    &=
    \frac{(z/2)^\nu}{\Gamma(\nu+1)}
    \sum_{k=0}^{\infty}
    \frac{(-z^2/4)^k}{(\nu+1)_k\,k!}\\
    &=
    \frac{(z/2)^\nu}{\Gamma(\nu+1)}
    {}_0F_1\!\left(;\nu+1;-\frac{z^2}{4}\right)
\end{align*}
Substitution into \eqref{eq:ball-cf-bessel} gives \eqref{eq:ball-cf-0f1}.

\emph{Expression in terms of ${}_1F_1$.}
Starting from \eqref{eq:ball-cf-integral}, let $u=2w-1$ and set $a=(d+1)/2$. Since $1-u^2=4w(1-w)$, we obtain:
\begin{equation}
     \varphi_X(t)
    =
    \frac{2^d\Gamma(d/2+1)}
         {\sqrt{\pi}\,\Gamma(a)}
    e^{-iz}
    \int_0^1 e^{2izw}w^{a-1}(1-w)^{a-1}\,dw
\end{equation}
Kummer's integral representation gives:
\begin{equation}
    \int_0^1 e^{2izw}w^{a-1}(1-w)^{a-1}\,dw
    =
    \frac{\Gamma(a)^2}{\Gamma(2a)}
    {}_1F_1(a;2a;2iz)
\end{equation}
The gamma duplication formula implies:
\begin{equation}
      \Gamma(d+1)
    =
    \frac{2^d}{\sqrt{\pi}}
    \Gamma\!\left(\frac{d+1}{2}\right)
    \Gamma\!\left(\frac d2+1\right)
\end{equation}
so the prefactors cancel, yielding \eqref{eq:ball-cf-1f1}.
\end{proof}

\begin{remark}
Unlike the Bessel quotient \eqref{eq:ball-cf-bessel}, the expression \eqref{eq:ball-cf-0f1} is directly defined at $t=0$. Although \eqref{eq:ball-cf-1f1} involves complex quantities, the characteristic function is real-valued because the reference distribution is symmetric about the origin.
\end{remark}

\subsubsection{One-Dimensional Projections}
The sliced distribution-matching objective compares scalar projections of the embeddings with the corresponding projections of the reference distribution.

\begin{proposition}[Projected characteristic function]
\label{prop:projected-ball-cf}
Let $X\sim\operatorname{Unif}(\mathbb B_R^d(c))$, fix $u\in\mathbb{S}^{d-1}$, and define $Y_u=\langle u,X\rangle$. Then:
\begin{equation}
    \varphi_{Y_u}(s)
    =
    e^{is\langle u,c\rangle}
    2^{d/2}\Gamma\!\left(\frac d2+1\right)
    \frac{J_{d/2}(R|s|)}
         {(R|s|)^{d/2}},
    \qquad s\in\mathbb{R}
    \label{eq:projected-ball-cf}
\end{equation}
with $\varphi_{Y_u}(0)=1$.
\end{proposition}

\begin{proof}
Writing $X=c+X_0$, with $X_0\sim\operatorname{Unif}(\mathbb B_R^d)$, gives:
\begin{align*}
    \varphi_{Y_u}(s)
    & = \mathbb{E} [e^{is\langle X, u\rangle}] \\
    & = e^{is \langle c, u\rangle}  \, \mathbb{E} [e^{i\langle X_0, su\rangle}] \\
    & = e^{is\langle u,c\rangle}\varphi_{X_0}(su)
\end{align*}
and \eqref{eq:ball-cf-bessel} with $\|su\|_2=|s|$ yields \eqref{eq:projected-ball-cf}.
\end{proof}



\begin{corollary}[Common scalar reference]
\label{cor:scalar-reference}
After centering, \eqref{eq:projected-ball-cf} is independent of $u$, so all sketch directions use the same scalar reference characteristic function:
\[
    \varphi_{R, d}(s) = 2^{d/2}\Gamma\!\left(\frac d2+1\right)
    \frac{J_{d/2}(R|s|)}
         {(R|s|)^{d/2}},
    \qquad s\in\mathbb{R}.
\]
\end{corollary}

\begin{proposition}[Projected density]
\label{prop:projected-ball-density}
In the setting of \Cref{prop:projected-ball-cf}, write $m_u=\langle u,c\rangle$. Then
\begin{equation}
    f_{Y_u}(y)
    =
    \frac{\Gamma(d/2+1)}
         {R\sqrt{\pi}\,\Gamma((d+1)/2)}
    \left(
        1-\frac{(y-m_u)^2}{R^2}
    \right)^{(d-1)/2},
    \qquad |y-m_u|<R
    \label{eq:projected-ball-density}
\end{equation}
and $f_{Y_u}(y)=0$ outside this interval.
\end{proposition}

\begin{proof}
This follows from the cross-section argument in the proof of \Cref{lem:ball-cf-integral}. Integrating this density against $e^{isy}$ recovers \eqref{eq:projected-ball-cf}.
\end{proof}

\begin{remark}
For the centered unit ball, the hypergeometric expressions of \Cref{prop:ball-cf-hypergeometric} reduce to:
\begin{equation}
      \varphi_{Y_u}(s)
    =
    {}_0F_1\!\left(;\frac d2+1;-\frac{s^2}{4}\right)
    =
    e^{-i|s|}
    {}_1F_1\!\left(
        \frac{d+1}{2};d+1;2i|s|
    \right)
\end{equation}
\end{remark}


\subsubsection{Reference Scale Used in Training}

LocaLeReg centers each cluster and matches the resulting residuals to a common uniform ball distribution reference. The radius is chosen to give identity covariance.

\begin{lemma}[Covariance of the uniform ball]
\label{lem:ball-covariance}
If $X_0\sim\operatorname{Unif}(\mathbb B_R^d)$, then:
\begin{equation}
    \operatorname{Cov}(X_0)
    =
    \frac{R^2}{d+2}I_d
\end{equation}
\end{lemma}

\begin{proof}
This follows from rotational symmetry and $\mathbb{E}\|X_0\|_2^2=dR^2/(d+2)$.
\end{proof}

\begin{corollary}[LocaLeReg reference]
\label{cor:localereg-reference}
The training target is:
\begin{equation}
     P_{\mathrm{ball}}
    =
    \operatorname{Unif}\!\left(B_{\sqrt{d+2}}^d\right)
\end{equation}
which has identity covariance and whose one-dimensional projections have characteristic function:
\begin{equation}
    \varphi_d(s)
    =
     2^{d/2}\Gamma\!\left(\frac d2+1\right)
    \frac{J_{d/2}(\sqrt{d+2}\,|s|)}
         {(\sqrt{d+2}\,|s|)^{d/2}},
    \qquad s\in\mathbb{R}
    \label{eq:localereg-reference-cf}
\end{equation}
\end{corollary}

\begin{proof}
Set $R=\sqrt{d+2}$ in \Cref{lem:ball-covariance} and \Cref{cor:scalar-reference}.
\end{proof}

For the $d=16$ projector used in our experiments, the target radius is $\sqrt{18}$ (see \Cref{sec:ball-radius} for more discussion on the scale equivalence of target radius).


\subsection{LocaLeJEPA Algorithm}
\label{sec:locale-algorithm}

LocaLeJEPA combines the multi-view prediction objective with distribution matching to a uniform ball within learned neighborhoods. The computation consists of constructing a shared neighborhood partition, centering each neighborhood independently, and averaging a sliced Epps--Pulley discrepancy over sufficiently populated neighborhoods.

\paragraph{Neighborhood Construction and Centering.}
For each image, we compute a reference embedding $\bar z_i$ by averaging its projected reference views. Cosine-similarity $K$-means is then applied to these reference embeddings, using at most ten iterations. This produces a partition $\{C_q\}_{q=1}^{K}$ of the images, shared across all views: every view of a given image belongs to the same neighborhood.

For a neighborhood $C_q$ containing $n_q=|C_q|$ images, we compute a separate mean within each view and form the residuals:
\begin{equation}
    \mu_q^{(v)}
    =
    \frac{1}{n_q}\sum_{i\in C_q}z_i^{(v)},
    \qquad
    r_i^{(v)}
    =
    z_i^{(v)}-\mu_q^{(v)}
\end{equation}
Gradients propagate through both the projected embeddings and their cluster mean, while the neighborhood assignments remain detached. The residuals are matched to
\begin{equation}
    P_{\mathrm{ball}}
    =
    \operatorname{Unif}\!\left(B_{\sqrt{d+2}}^d\right)
\end{equation}
where $d$ is the projector dimension. This reference has identity covariance, and centering allows the regularizer to constrain the distribution within each neighborhood without driving all neighborhood means to a common origin.

\paragraph{Sliced Distribution Matching.}
For a direction $u\in\mathbb{S}^{d-1}$, the empirical characteristic function of the projected residuals is:
\begin{equation}
    \widehat{\varphi}_{q,v,u}(s)
    =
    \frac{1}{n_q}
    \sum_{i\in C_q}
    e^{is\langle u,r_i^{(v)}\rangle}
\end{equation}
Let $\varphi_d(s)$ denote the characteristic function of a one-dimensional projection of $P_{\mathrm{ball}}$, as in \eqref{eq:localereg-reference-cf}. Using $S$ directions $\{u_\ell\}_{\ell=1}^{S}$ sampled uniformly from $\mathbb{S}^{d-1}$, we compute:
\begin{equation}
    D_{q, v}
    =
    \frac{n_q}{S}
    \sum_{\ell=1}^{S}
    \int_{\mathbb{R}}
    \left|
        \widehat{\varphi}_{q,v,u_\ell}(s)-\varphi_d(s)
    \right|^2 w(s)\,ds
    \label{eq:localereg-sliced-discrepancy}
\end{equation}
We use $S=256$ directions and the Gaussian frequency-weighting function $w(s)=e^{-s^2/2}$. The integral is approximated over $[-3,3]$ using the trapezoidal rule, exploiting symmetry to evaluate 17 equally spaced frequencies on $[0,3]$. \revred{ViT-L/14 uses 1024 directions and the range $[-5,5]$ (\Cref{app:imagenet-pretrain}).} For a fixed projector dimension, the same reference characteristic function is used for every cluster and view.

\paragraph{Minimum Neighborhood Size.}
Although the Epps--Pulley discrepancy is defined for any nonempty sample, very small neighborhoods provide limited information about the shape of a distribution. Centering further constrains their empirical geometry. For instance, a singleton becomes identically zero, while a two-point neighborhood reduces to a pair of opposite residuals. \rev{As a result}, a discrepancy computed from very few points need not reliably distinguish the uniform ball target from alternative distributions.
We therefore exclude neighborhoods containing fewer than four images. This \rev{minimum size is a heuristic choice. It is neither} a theoretical condition for the statistic to be defined \rev{nor} a guarantee of reliable distribution discrimination. 

\paragraph{Sensitivity to Sample Size and Dimension.}
To characterize this finite-sample limitation, we compare the statistic on independent samples drawn either from the uniform ball of radius $\sqrt{d+2}$ or from $\mathcal{N}(0,I_d)$. The two distributions have the same covariance, so the comparison tests sensitivity to differences beyond their first two population moments. Samples are evaluated directly, without clustering. For each setting, we perform 200 trials and summarize the separation between the resulting statistic values as:
\begin{equation}
    \mathrm{sep}
    =
    \frac{
        \bar D_{\mathrm{gauss}}-\bar D_{\mathrm{ball}}
    }{
        \sqrt{
            (s_{\mathrm{gauss}}^2+s_{\mathrm{ball}}^2)/2
        }
    }
    \label{eq:sep}
\end{equation}
where $\bar D$ and $s$ denote the empirical mean and standard
deviation of the sliced EP discrepancy across trials. 
As such, a positive value indicates that the ball receives a lower discrepancy on average.

At $d=16$, \Cref{tab:sensitivity-n} shows that separation with centering increases from $0.81$ at the minimum accepted size $n=4$ to $1.10$ at $n=8$ and $3.54$ at $n=256$. The corresponding uncentered values are lower, with separation remaining negative in the reported trials for $n\leq32$ and reaching $0.66$ at $n=256$. Thus, centering improves separation in this experiment, but the minimum accepted neighborhood size still provides only modest discrimination between these covariance-matched targets.

\Cref{tab:sensitivity-d} \rev{also} shows that discrimination weakens as the projector dimension increases. At $n=256$, separation decreases from $6.29$ at $d=4$ to $3.54$ at $d=16$, $0.91$ at $d=64$, and $0.16$ at $d=256$. This is expected: the one-dimensional marginal of the covariance-normalized uniform ball on $\mathbb{R}^d$ has density proportional to $\left(1 - t^2/(d+2)\right)^{(d-1)/2}$, which is also the one-dimensional marginal of the uniform distribution on the $(d+1)$-dimensional sphere of radius $\sqrt{d+2}$ in $\mathbb{R}^{d+2}$, and therefore converges to the standard Gaussian as $d$ grows~\rev{\citep{diaconis1987dozen}}.

\rev{These} observations expose a tradeoff between locality and statistical resolution. Increasing $K$ produces finer neighborhoods but reduces their typical size, approximately $N/K$ for a balanced partition of a batch of $N$ images. Larger neighborhoods provide more information for distribution matching, whereas the four-image threshold only excludes the smallest neighborhoods. Our sensitivity study characterizes this tradeoff under independent sampling, but it does not establish an optimal cluster size for learned neighborhoods.

\begin{table}[htbp]
    \caption{\emph{Sample size, Centering and Projector Dimension Sensitivity.} In this table, we show the sensitivity of the sliced Epps--Pulley discrepancy to sample size, centering, and projector dimension, with 200 trials per setting. ``Ball'' and ``Gaussian'' report the mean discrepancy $\bar D_{\mathrm{ball}}$ and $\bar D_{\mathrm{gauss}}$ for samples drawn from $\operatorname{Unif}(B_{\sqrt{d+2}}^d)$ and $\mathcal{N}(0,I_d)$, respectively. Both are evaluated against the same uniform ball distribution target. ``Centered'' denotes subtraction of the empirical sample mean before evaluating the discrepancy. (a) Sample size $n$ at $d=16$, with and without centering. (b) Projector dimension $d$ at $n=256$, with centering. ``sep.'' is the standardized difference defined in Equation~\eqref{eq:sep}\rev{, and} pooled $s$ is its denominator.}
    \label{tab:sensitivity}
    \centering\small
    \setlength{\tabcolsep}{4pt}
    \begin{subtable}[t]{0.58\linewidth}
        \centering
        \caption{Cluster size and centering, $d=16$.}
        \label{tab:sensitivity-n}
        \begin{tabular}{r rrr rrr}
            \toprule
                & \multicolumn{3}{c}{centered} & \multicolumn{3}{c}{uncentered}              \\
            \cmidrule(lr){2-4}\cmidrule(l){5-7}
            $n$ & ball  & Gaussian & sep.  & ball  & Gaussian & sep.    \\
            \midrule
            2   & 0.505 & 0.515    & 0.37  & 1.067 & 1.045    & $-0.29$ \\
            4   & 0.533 & 0.551    & 0.81  & 1.071 & 1.042    & $-0.33$ \\
            8   & 0.543 & 0.570    & 1.10  & 1.071 & 1.060    & $-0.11$ \\
            16  & 0.551 & 0.580    & 1.23  & 1.076 & 1.061    & $-0.15$ \\
            32  & 0.551 & 0.589    & 1.52  & 1.084 & 1.065    & $-0.20$ \\
            64  & 0.553 & 0.605    & 2.04  & 1.072 & 1.080    & 0.08    \\
            128 & 0.554 & 0.629    & 2.64  & 1.078 & 1.098    & 0.21    \\
            256 & 0.555 & 0.678    & 3.54  & 1.081 & 1.146    & 0.66    \\
            \bottomrule
        \end{tabular}
    \end{subtable}
    \hfill
    \begin{subtable}[t]{0.38\linewidth}
        \centering
        \caption{Embedding dimension, $n=256$.}
        \label{tab:sensitivity-d}
        \begin{tabular}{rrrrr}
            \toprule
            $d$ & ball  & Gaussian & pooled $s$ & sep. \\
            \midrule
            4   & 0.532 & 1.596    & 0.169      & 6.29 \\
            16  & 0.555 & 0.678    & 0.035      & 3.54 \\
            64  & 0.564 & 0.580    & 0.018      & 0.91 \\
            256 & 0.569 & 0.572    & 0.017      & 0.16 \\
            \bottomrule
        \end{tabular}
    \end{subtable}
\end{table}

\paragraph{Aggregation and Training Objective.}
Let $\mathcal{Q}
    =
    \{q:|C_q|\geq4\}$
index the retained neighborhoods. When $\mathcal{Q}$ is nonempty, the regularizer is averaged uniformly over retained neighborhoods and views:
\begin{equation}
     \mathcal{L}_{\mathrm{LocaLeReg}}
    =
    \frac{1}{|\mathcal{Q}|V}
    \sum_{q\in\mathcal{Q}}
    \sum_{v=1}^{V}
    D_{q,v}
\end{equation}
If no neighborhood meets the minimum size, we instead apply the same centered uniform ball distribution regularizer to the complete batch, separately within each view. The training objective is:
\begin{equation}
     \mathcal{L}
    =
    (1-\lambda)\mathcal{L}_{\mathrm{pred}}
    +
    \lambda\mathcal{L}_{\mathrm{LocaLeReg}}
\end{equation}
\Cref{alg:local-ballreg} summarizes the computation. The neighborhood count $K$ is either fixed or follows the epoch-dependent schedule described in \Cref{app:shared-protocol}.

\begin{algorithm}[htbp]
\caption{LocaLeJEPA objective with Uniform-ball local regularization}
\label{alg:local-ballreg}
\begin{algorithmic}[1]
\Require Projected embeddings
    $\{z_i^{(v)}\}_{i=1,\ldots,N;\,v=1,\ldots,V}
    \subset\mathbb{R}^d$
\Require Neighborhood count $K$, weight $\lambda$,
    sketch count $S$, minimal neighborhood size $n_{\mathrm{min}}$

\State Compute each image's reference embedding $\bar z_i$
\State $\displaystyle
    \mathcal{L}_{\mathrm{pred}}
    \gets
    \frac{1}{NVd}
    \sum_{i=1}^{N}\sum_{v=1}^{V}
    \|z_i^{(v)}-\bar z_i\|_2^2$

\State Compute a shared partition using $K$-means:
    \[
        \{C_1,\ldots,C_K\}
        \gets
        \operatorname{KMeans}
        \left(
            \{\operatorname{stopgrad}(\bar z_i)\}_{i=1}^{N},
            K
        \right)
    \]
\State $\mathcal{C}\gets\{C_q:|C_q|\geq n_{\mathrm{min}}\}$ 
\If{$\mathcal{C}=\varnothing$}
    \State $\mathcal{C}\gets\{\{1,\ldots,N\}\}$
        \Comment{Fall back to the complete batch}
\EndIf
\State $P_{\mathrm{ball}}\gets
    \operatorname{Unif}(B_{\sqrt{d+2}}^d)$
\State $\mathcal{L}_{\mathrm{LocaLeReg}}\gets0$

\For{each neighborhood $C\in\mathcal{C}$}
    \For{each view $v=1,\ldots,V$}
        \State $\displaystyle
            \mu_C^{(v)}
            \gets
            \frac{1}{|C|}
            \sum_{i\in C}z_i^{(v)}$
        \State $\displaystyle
            R_C^{(v)}
            \gets
            \{z_i^{(v)}-\mu_C^{(v)}:i\in C\}$
        \State $\displaystyle
            \mathcal{L}_{\mathrm{LocaLeReg}}
            \gets
            \mathcal{L}_{\mathrm{LocaLeReg}}
            +
            D_{q,v}$
    \EndFor
\EndFor

\State $\displaystyle
    \mathcal{L}_{\mathrm{LocaLeReg}}
    \gets
    \frac{\mathcal{L}_{\mathrm{LocaLeReg}}}
         {V|\mathcal{C}|}$
\State \Return $(1-\lambda)\mathcal{L}_{\mathrm{pred}}
    +\lambda\mathcal{L}_{\mathrm{LocaLeReg}}$
\end{algorithmic}
\end{algorithm}

\subsection{Scale Equivalence Across Different Ball Radii}
\label{sec:ball-radius}
As we have already mentioned, in our experiments we set $R=\sqrt{d+2}$, so that the target is whitened, i.e.\ has identity covariance. In fact, as we will see next, this choice \rev{only fixes} a unit of length\rev{ and is not a separate} hyperparameter.

\begin{lemma}\label{lem:radius-scale}
For a random vector $X\in\mathbb R^d$ and a unit vector $u$, let $\varphi_X(t)=\mathbb E\,[e^{\mathrm i t u^\top X}]$, $t\in \mathbb{R}$. Then for every $\sigma>0$ and $t\in\mathbb R$:
\begin{equation}
     \varphi_X(t/\sigma)=\varphi_{X/\sigma}(t)
\end{equation}
In particular, for $\sigma=R/\sqrt{d+2}$, $\varphi_{d,R}(t/\sigma)=\varphi_{d,\sqrt{d+2}}(t)$, and if the characteristic functions in the Epps--Pulley statistic are evaluated at $t/\sigma$ instead of $t$, with the Gaussian weights unchanged, the statistic at radius $R$ on residuals $r$ equals the statistic at $\sqrt{d+2}$ on $r/\sigma$.
\end{lemma}

\begin{proof}
The identity follows from the definition: $\mathbb E\,\left[e^{\mathrm i (t/\sigma)\,u^\top X}\right]=\mathbb E\,\left[e^{\mathrm i t\,u^\top (X/\sigma)}\right]$.

For $\sigma=R/\sqrt{d+2}$ and $X\sim\operatorname{Unif}\left(\mathbb B^d_R\right)$ we have $X/\sigma\sim\operatorname{Unif}\left(\mathbb B^d_{\sqrt{d+2}}\right)$, so the identity gives $\varphi_{d,R}(t/\sigma)=\varphi_{d,\sqrt{d+2}}(t)$. Applied to the empirical distribution of the residuals, it gives $\widehat\varphi_r(t/\sigma)=\widehat\varphi_{r/\sigma}(t)$, so the two statistics have the same integrand at every $t$ and $u$, and the Gaussian weights depend on neither $R$ nor $r$.
\end{proof}

The invariance loss also has to be rescaled accordingly by dividing it by $\sigma^2$:

\begin{corollary}\label{cor:radius-inv}
Let $\sigma=R/\sqrt{d+2}$, and let $\mathcal L^{R}_{\revred{\rm LocaLeReg}}$ evaluate the characteristic functions at $t/\sigma$ as in \Cref{lem:radius-scale}. For every batch of embeddings $z$:
\begin{equation}
    \mathcal L^{R}_{\revred{\rm LocaLeReg}}(z)=\mathcal L^{\sqrt{d+2}}_{\revred{\rm LocaLeReg}}(z/\sigma),
 \qquad
 \mathcal L_{\rm pred}(z)=\sigma^{2}\,\mathcal L_{\rm pred}(z/\sigma)
\end{equation}
\rev{Hence}, $\mathcal L_R=(1-\lambda)\,\sigma^{-2}\,\mathcal L_{\rm pred}+\lambda\,\mathcal L^{R}_{\revred{\rm LocaLeReg}}$ satisfies $\mathcal L_R(z)=\mathcal L_{\sqrt{d+2}}(z/\sigma)$.
\end{corollary}

\begin{proof}
Since $k$-means assignments do not change when all points are rescaled by the same factor, $z$ and $z/\sigma$ give the same clusters, with residuals divided by $\sigma$\rev{, and} \Cref{lem:radius-scale} applied to every cluster gives the first identity.

For the second, $\bar z_i/\sigma$ is the mean of the reference views of $z/\sigma$, so:
\begin{equation}
    \mathcal L_{\rm pred}(z/\sigma)
 =\frac1{NVd}\sum_{i=1}^N\sum_{v=1}^V\Bigl\lVert \frac{z_i^{(v)}}{\sigma}-\frac{\bar z_i}{\sigma}\Bigr\rVert^2
 =\frac{\sigma^{-2}}{NVd}\sum_{i=1}^N\sum_{v=1}^V\lVert z_i^{(v)}-\bar z_i\rVert^2
 =\sigma^{-2}\,\mathcal L_{\rm pred}(z)
\end{equation}
Multiplying the second identity by $(1-\lambda)\sigma^{-2}$, the first by $\lambda$, and adding gives the last statement.
\end{proof}

The effect of the two scale corrections is examined in a radius ablation (\Cref{tab:radius}). With both corrections, the run at $R=1$ solves the same problem as the run at $R=\sqrt{18}$ (\Cref{cor:radius-inv}), and the two agree closely. The gap of $4$ points in probe accuracy is not explained by the objective, which is identical, and we attribute it to optimization. With only the frequencies rescaled, the objective is the one at $\sqrt{18}$ with $\lambda$ replaced by $\lambda/(\lambda+(1-\lambda)\sigma^2)=0.818$. The regularizer is matched better than at $\sqrt{18}$, but the normalized invariance loss is $2.9$ times larger and the probe accuracy drops to $20.5$. Without either correction, the two errors partly cancel: the invariance loss is weighted $\sigma^{-2}=18$ times less, but the regularizer evaluated on the compressed frequency range is about $90$ times weaker against collapse. The representation does not collapse and its invariance is close to that at $\sqrt{18}$, but the probe accuracy is $16$ points lower. All runs use a single seed.

\begin{table}[thbp]
     \caption{\emph{Radius Comparison.} In this table, we show the Radius comparison over Imagenette, $K=4$, $N=256$, $\lambda=0.2$, 100 epochs. Freq.: characteristic functions evaluated at $t/\sigma$ (\Cref{lem:radius-scale})\rev{.} $\sigma^{-2}$: invariance loss divided by $\sigma^2$ (\Cref{cor:radius-inv}). At $R=\sqrt{18}$, $\sigma=1$ and both corrections are the identity. $^*$Without frequency rescaling, the discrepancy for standardized embeddings is evaluated at frequencies $\sigma t_j$ instead of $t_j$, while the Gaussian quadrature weights remain defined on $t_j$. Thus the regularizer values are not directly comparable.}
     \label{tab:radius}
\centering\small
     \begin{tabular}{lccrrrr}
         \toprule
         $R$ & freq. & $\sigma^{-2}$ & probe & $\mathcal L_{\rm pred}/\sigma^2$ & $\mathcal L_{\revred{\rm LocaLeReg}}$ & $\overline{\Sigma_{ii}}/\sigma^2$ \\
         \midrule
         $\sqrt{18}$ & \multicolumn{2}{c}{$\sigma=1$} & 63.7 & 0.249 & 0.654 & 1.08 \\
         \midrule
         $1$         & --         & --         & 47.5 & 0.294 & $0.021^{*}$ & 0.95 \\
         $1$         & \checkmark & --         & 20.5 & 0.722 & 0.566 & 1.03 \\
         $1$         & \checkmark & \checkmark & 59.7 & 0.283 & 0.668 & 1.07 \\
         \bottomrule
     \end{tabular}
\end{table}

\subsection{Complexity Analysis and Computational Cost}
\label{sec:cost-locality}
LocaLeJEPA evaluates its distribution-matching regularizer independently within
each neighborhood, introducing additional computation through clustering and
cluster-specific statistics. 







\begin{table}[t]
\centering
\caption{\emph{End-to-end Training Cost.}
Wall-clock time per step on one A100-40GB for ViT-S/8
($128{\times}128$, $V{=}4$, $d{=}16$, $N{=}128$, $K{=}8$,
$S{=}256$, $I{=}10$), averaged over 300 repeated steps.}
\label{tab:cost-step}

\vspace{0.3em}

\small
\setlength{\tabcolsep}{4pt}
\renewcommand{\arraystretch}{1.05}

\begin{tabular}{lrr}
\toprule
             & s/step & vs.\ SIGReg \\
\midrule
SIGReg       & 0.263  & ---        \\
\rev{LocaLeJEPA-SIGReg} & 0.277 & $+5.3\%$ \\
LocaLeJEPA   & 0.278  & $+5.7\%$ \\
\bottomrule
\end{tabular}

\end{table}

\textbf{Settings.} Let $N$ denote the batch size, $d$ the projector
dimension, $K$ the number of neighborhoods, $I$ the number of $K$-means
iterations, and $S$ the number of sketch directions. Because the neighborhoods
form a disjoint partition of the batch, the distribution statistic itself remains
linear in the total number of samples. 

\textbf{Theoretical Time and Memory Complexity.} Relative to global SIGReg, the additional
cost is dominated by $K$-means and by the cluster-specific sketch matrices,
yielding:
\begin{equation}
    O(INKd + KdS)
    \quad\text{time},
    \qquad
    O(NK + KdS)
    \quad\text{memory}
    \label{eq:overhead}
\end{equation}
where the $NK$ term corresponds to the $K$-means distance matrix and $KdS$ to
the sketch matrices retained for backpropagation. The additional cost is therefore
linear in $K$ and, for fixed $K$, linear in $N$.

\begin{table}[h]
    \centering
    \caption{\emph{Regularizer Cost as a Function of Batch and Neighborhood Size.}
    Forward--backward runtime and peak allocated memory on one A100-40GB in
    FP32. Values are medians over three seeds with $V{=}4$, $d{=}16$,
    $S{=}256$, $I{=}10$.}
    \label{tab:cost-reg}
    \small
    \setlength{\tabcolsep}{4pt}
    \begin{tabular}{@{}lrrrrr@{}}
        \toprule
                       & \multicolumn{3}{c}{Time (ms)}
                       & \multicolumn{2}{c}{Memory (MiB)} \\
        \cmidrule(lr){2-4}\cmidrule(l){5-6}
                       & SIGReg & \rev{LocaLeJEPA-SIGReg} & LocaLeJEPA
                       & SIGReg & Local \\
        \midrule
        \multicolumn{6}{@{}l}{\emph{$K=8$}} \\
        $N=32$ & 1.24 & 8.89  & 9.97  & 12.8  & 5.4  \\
        $N=64$         & 1.33 & 13.14 & 14.86 & 25.5  & 10.4 \\
        $N=128$        & 1.27 & 14.57 & 16.39 & 51.0  & 18.6 \\
        $N=256$        & 1.24 & 16.25 & 18.02 & 105.0 & 33.2 \\
        $N=512$        & 1.54 & 16.24 & 18.12 & 203.9 & 62.0 \\
        \midrule
        \multicolumn{6}{@{}l}{\emph{$N=256$}} \\
        $K=1$          & 1.28 & 3.18  & 3.33  & 108.0 & 108.0 \\
        $K=2$          & 1.25 & 8.63  & 9.03  & 108.0 & 67.2  \\
        $K=4$          & 1.27 & 11.22 & 12.22 & 108.0 & 45.6  \\
        $K=8$          & \rev{1.24} & \rev{16.25} & \rev{18.02} & \rev{105.0} & 33.2  \\
        $K=16$         & 1.24 & 25.77 & 29.67 & 108.0 & 28.8  \\
        $K=32$         & 1.29 & 44.41 & 51.59 & 108.0 & 26.9  \\
        \bottomrule
    \end{tabular}
\end{table}

\textbf{Empirical Time and Memory Complexity.} At the representative training configuration in
\Cref{tab:cost-step}, \rev{LocaLeJEPA-SIGReg} and LocaLeJEPA increase end-to-end
step time by only $5.3\%$ and $5.7\%$, respectively, relative to global
SIGReg. Peak reserved GPU memory is $22.3$~GiB for all three methods, \rev{so} the additional regularizer state is small relative to the memory required by
the backbone activations. The cost is more visible when isolating the regularizer itself. At $N{=}256$
and $K{=}8$, LocaLeJEPA requires $18.02$~ms for a forward--backward pass,
compared with $1.24$~ms for global SIGReg
(\Cref{tab:cost-reg}). Runtime increases with $K$, consistent with the
$K$-dependent terms in \eqref{eq:overhead}, while its dependence on $N$ is
comparatively weak over the tested range. \rev{Increasing} locality does not increase the peak memory of the
regularizer. At $N{=}256$, allocated memory decreases from $108.0$~MiB at
$K{=}1$ to $26.9$~MiB at $K{=}32$, and is lower than global SIGReg for every
$K\geq2$. This occurs because the largest intermediate tensors scale with the
size of an individual cluster\rev{, which is a fraction of} the full batch. In end-to-end
training, however, these differences are dominated by the backbone activations
and do not affect peak reserved GPU memory.
\section{Ablation Analysis}
\label{sec:ablations}

\definecolor{ablationgold}{HTML}{F4D675}
\definecolor{ablationsilver}{HTML}{DCE1E7}
\definecolor{ablationbronze}{HTML}{DEB18B}

In this section, we ablate three hyperparameters of LocaLeJEPA:
\begin{itemize}
    \item $\lambda$, the regularization weight
    \item $K$, the requested number of neighborhoods formed by $K$-means clustering
    \item $B$, the per-step batch size on which clustering and regularization are performed.
\end{itemize}
We also report $\widehat{K}$, the effective number of neighborhoods remaining after discarding clusters with fewer than four samples. Unlike the quantities above, $\widehat{K}$ is \rev{measured during training}.
Gold, silver and bronze denote the three highest distinct reported mean accuracies within each ranked table\rev{, and} ties at the reported precision share a color.

\paragraph{Training setup.}
We train ViT-S/8 on Imagenette at $128\times128$ resolution for 800 epochs. All LocaLeJEPA runs use a learning rate of $10^{-3}$, ten warm-up epochs, cosine decay to $10^{-6}$, gradient clipping at 1.0, and an effective optimizer batch size of 256. \revred{These runs are separate from the rotation-recovery encoders of \Cref{app:yoda}, whose local recipe keeps the learning rate constant after warm-up (\Cref{tab:yoda-pretrain}).} We compare per-step batch sizes of 64 and 256, with four and one gradient-accumulation steps, respectively, denoted by $64\times4$ and $256\times1$. We report final-epoch online linear-probe validation accuracy as mean $\pm$ standard deviation over three seeds.

\paragraph{The $\lambda \times K$ grid at a fixed batch size, $B=64\times 4$.}
Table~\ref{tab:grid} evaluates regularization weights $\lambda \in \{0.02,0.05,0.1,0.2,0.3\}$ and requested neighborhood counts $K \in \{1,2,4,8,16,32\}$ using the $64\times4$ batch configuration. The highest mean accuracy is obtained with $K=2$ at both $\lambda=0.05$ and $\lambda=0.1$, followed by $K=1$ and $K=4$ at $\lambda=0.1$. At each fixed $\lambda$, the best local configuration ($K>1$) has a higher reported mean accuracy than the global control ($K=1$), except at $\lambda=0.02$, where their rounded means are equal (Table~\ref{tab:better-locality}). These differences vary in magnitude and do not, by themselves, establish statistical significance.

\begin{table}[htbp]
\centering
\caption{\textit{Grid ablations $\lambda \times K$, on Imagenette, 800 epochs, $B=64\times4$.} Accuracy (\%): mean $\pm$ SD, $n=3$. $\dagger$ flags requested mean cluster size below four. }
\label{tab:grid}
\small
\setlength{\tabcolsep}{4pt}
\resizebox{\linewidth}{!}{%
\begin{tabular}{lcccccc}
\toprule
$\lambda$ & $K=1$ & $K=2$ & $K=4$ & $K=8$ & $K=16$ & $K=32$ \\
\midrule
0.02 & \makecell{$90.28 \pm 0.42$} & \makecell{$90.28 \pm 0.18$} & \makecell{$52.00 \pm 5.98$} & \makecell{$30.37 \pm 11.25$} & \makecell{$26.72 \pm 2.97$} & \makecell{$18.63 \pm 5.87$ $\dagger$} \\
0.05 & \makecell{$91.08 \pm 0.14$} & \cellcolor{ablationgold}\makecell{$91.52 \pm 0.13$} & \makecell{$91.21 \pm 0.14$} & \makecell{$86.84 \pm 2.46$} & \makecell{$64.20 \pm 14.80$} & \makecell{$61.04 \pm 9.43$ $\dagger$} \\
0.1 & \cellcolor{ablationsilver}\makecell{$91.38 \pm 0.34$} & \cellcolor{ablationgold}\makecell{$91.52 \pm 0.24$} & \cellcolor{ablationbronze}\makecell{$91.35 \pm 0.17$} & \makecell{$90.78 \pm 0.31$} & \makecell{$88.56 \pm 0.16$} & \makecell{$31.43 \pm 2.78$ $\dagger$} \\
0.2 & \makecell{$90.32 \pm 0.27$} & \makecell{$90.43 \pm 0.41$} & \makecell{$90.82 \pm 0.36$} & \makecell{$90.45 \pm 0.13$} & \makecell{$89.09 \pm 0.48$} & \makecell{$80.96 \pm 0.97$ $\dagger$} \\
0.3 & \makecell{$87.76 \pm 0.09$} & \makecell{$88.32 \pm 0.15$} & \makecell{$89.32 \pm 0.11$} & \makecell{$89.11 \pm 0.26$} & \makecell{$88.16 \pm 0.46$} & \makecell{$81.03 \pm 0.11$ $\dagger$} \\
\bottomrule
\end{tabular}
}
\end{table}

\begin{table}[htbp]
\centering
\caption{\textit{Best fixed local arm versus its $K=1$ control at each $\lambda$, $B=64\times4$.} Accuracy (\%): mean $\pm$ SD, $n=3$. $\Delta$ denotes local minus control accuracy in percentage points\rev{, and} ``Chosen $K$'' identifies the best local setting. Bold \rev{marks} the highest mean in each row, including ties at the displayed precision.}
\label{tab:better-locality}
\small
\setlength{\tabcolsep}{4pt}
\begin{tabular}{lcccc}
\toprule
$\lambda$ & $K=1$ & Best $K>1$ & Chosen $K$ & $\Delta$ \\
\midrule
0.02 & \makecell{$\mathbf{90.28 \pm 0.42}$} & \makecell{$\mathbf{90.28 \pm 0.18}$} & 2 & $-0.01$ \\
0.05 & \makecell{$91.08 \pm 0.14$} & \makecell{$\mathbf{91.52 \pm 0.13}$} & 2 & $+0.44$ \\
0.1 & \makecell{$91.38 \pm 0.34$} & \makecell{$\mathbf{91.52 \pm 0.24}$} & 2 & $+0.14$ \\
0.2 & \makecell{$90.32 \pm 0.27$} & \makecell{$\mathbf{90.82 \pm 0.36}$} & 4 & $+0.50$ \\
0.3 & \makecell{$87.76 \pm 0.09$} & \makecell{$\mathbf{89.32 \pm 0.11}$} & 4 & $+1.56$ \\
\bottomrule
\end{tabular}
\end{table}

Increasing $K$ sharply reduces accuracy under weak regularization. Larger regularization weights mitigate this decline over part of the grid, although \rev{strong} degradation remains at high $K$. In particular, requesting $K=32$ from a batch of 64 would induce $64/32=2$ samples per cluster for a balanced clustering, below the minimum cluster size of four. \rev{As a result}, many clusters may be discarded, leaving only a few neighborhoods contributing to the regularization loss. If no cluster contains at least four samples, the loss falls back to regularizing the entire per-step batch. This filtering provides a plausible contributing mechanism for the poor performance at large requested $K$, but the reported results do not establish it as the sole cause. Table~\ref{tab:effective-K} reports the realized neighborhood counts $\widehat{K}$ averaged over the final ten training epochs.

\begin{table}[htbp]
\centering
\caption{\textit{Reported realized \revred{neighborhood} counts at $\lambda=0.1$}. Fixed-$K$ counts average the last ten epochs\rev{, and} ramp entries describe the terminal part of the schedule.}
\label{tab:effective-K}
\small
\setlength{\tabcolsep}{4pt}
\begin{tabular}{lcccccccc}
\toprule
Batch & $K=1$ & $K=2$ & $K=4$ & $K=8$ & $K=16$ & $K=32$ & fwd & rev \\
\midrule
$64\times4$ & 1.0 & 2.0 & 3.9 & 6.9 & 8.7 & 5.7 & 9.0 & 2.0 \\
$256\times1$ & 1.0 & 2.0 & 4.0 & 7.9 & 15.4 & 27.8 & 15.5 & 2.0 \\
\bottomrule
\end{tabular}
\end{table}

\paragraph{Dynamic \revred{neighborhood} scheduling across $\lambda$ at $B=64\times4$.}
Table~\ref{tab:ramps} examines how increasing or decreasing the requested number of neighborhoods affects performance across regularization weights $\lambda$. The forward schedule increases $K$ from 2 to 16, whereas the reverse schedule decreases it from 16 to 2. 

The forward schedule at $\lambda=0.1$ achieves the highest mean accuracy among the $B=64\times4$ configurations, at $91.57\pm0.20\%$. However, it exceeds the best fixed local configuration by only 0.05 percentage points and the best global control by 0.19 percentage points\rev{. These} summaries alone do not establish a statistically significant improvement. Schedule direction has a much larger effect under weak regularization, as the forward schedule outperforms the reverse schedule by 47.38 percentage points at $\lambda=0.02$ and 18.88 percentage points at $\lambda=0.05$. The gap narrows at stronger regularization weights, with the reverse schedule achieving a slightly higher mean at $\lambda=0.3$. These results favor starting with fewer neighborhoods when regularization is weak, but do not establish \rev{one} schedule direction \rev{as better in all settings}.

\begin{table}[htbp]
\centering
\caption{\textit{Forward and reverse ramps, $B=64\times4$, 800 epochs.} Accuracy (\%): mean $\pm$ SD, $n=3$. }
\label{tab:ramps}
\small
\setlength{\tabcolsep}{4pt}
\begin{tabular}{lccccc}
\toprule
Schedule & $\lambda=0.02$ & $\lambda =0.05$ & $\lambda = 0.1$ & $\lambda= 0.2$ & $\lambda = 0.3$ \\
\midrule
$2\to16$ (fwd) & \makecell{$87.31 \pm 0.25$} & \cellcolor{ablationsilver}\makecell{$91.20 \pm 0.22$} & \cellcolor{ablationgold}\makecell{$91.57 \pm 0.20$} & \cellcolor{ablationbronze}\makecell{$90.45 \pm 0.25$} & \makecell{$88.76 \pm 0.23$} \\
$16\to2$ (rev) & \makecell{$39.93 \pm 6.75$} & \makecell{$72.32 \pm 5.20$} & \makecell{$89.04 \pm 0.43$} & \makecell{$89.48 \pm 0.49$} & \makecell{$88.84 \pm 0.27$} \\
\bottomrule
\end{tabular}
\end{table}

\paragraph{Ablation in $K$ for $\lambda=0.1$ and $B=256\times 1$.}
We further investigate the role of locality by increasing the per-step batch size from $B=64\times4$ to $B=256\times1$. The effective batch size remains unchanged, but clustering now operates on four times as many samples at each step. We fix $\lambda=0.1$, a competitive regularization strength in the preceding ablations, and vary the number of training clusters $K\in\{\rev{1,}2,4,8,16,32\}$.
As shown in Table~\ref{tab:k-ablation-online}, accuracy increases from \rev{$87.44\pm0.41\%$ at $K=1$ through $88.54\pm0.36\%$} at $K=2$ to \rev{$89.58\pm0.05\%$} at $K=8$, before decreasing slightly to \rev{$89.50\pm0.44\%$} and \rev{$88.99\pm0.27\%$} at $K=16$ and $K=32$, respectively. Compared with the $64\times4$ setting, the best observed local mean shifts from $K=2$ to $K=8$, although it is lower (\rev{$89.58\%$} versus $91.52\%$). The clearest difference \rev{appears with} fine partitions: at $K=32$, the $64\times4$ setting reaches only $31.43\pm2.78\%$, whereas the $256\times1$ setting retains \rev{$88.99\pm0.27\%$}. \rev{With} the larger per-step batch \rev{and $\lambda$ fixed at $0.1$}, the highest mean accuracies are attained at $K=8$ and $K=16$\rev{. This suggests} that finer local partitions become more effective when each clustering step has access to more samples\rev{. At fixed $\lambda$, however, a larger $K$ also lowers the effective regularization strength (\Cref{tab:strength-bands})}.

\begin{table}[htbp]
\revstart
\centering
\caption{\textit{Neighborhood-count ablation on Imagenette, 800 epochs, $B=256\times1$, $\lambda=0.1$.} Final-epoch online linear-probe accuracy (\%): mean $\pm$ SD over three pretraining seeds. The row \revred{$\widehat{K}$} gives the realized number of neighborhoods (mean of the last ten epochs, averaged over seeds).}
\label{tab:k-ablation-online}
\small
\setlength{\tabcolsep}{4pt}
\begin{tabular}{lcccccc}
\toprule
$\lambda$ & $K=1$ & $K=2$ & $K=4$ & $K=8$ & $K=16$ & $K=32$ \\
\midrule
0.1 & \makecell{$87.44 \pm 0.41$} & \makecell{$88.54 \pm 0.36$} & \cellcolor{ablationbronze}\makecell{$89.07 \pm 0.33$} & \cellcolor{ablationgold}\makecell{$89.58 \pm 0.05$} & \cellcolor{ablationsilver}\makecell{$89.50 \pm 0.44$} & \makecell{$88.99 \pm 0.27$} \\ 
\revred{$\widehat{K}$} & $1.0$ & $2.0$ & $4.0$ & $7.9$ & $15.5$ & $27.7$ \\ 
\bottomrule
\end{tabular}
\end{table}

{\revstart\paragraph{Effective regularization strength.}\label{sec:effective-strength}
Because the local discrepancies are weighted by neighborhood size and averaged over neighborhoods (\Cref{eq:objective,eq:local-discrepancy}), the regularizer acts with an effective strength $\lambda N/K$, where $N$ is the number of samples in each clustering step. \Cref{tab:strength-bands} groups every cell of the $800$-epoch grid, at both per-step batches, by this quantity. Every cell below $0.4$ collapses (at most $64.20\%$ online accuracy), and every cell between $0.64$ and $8$ is healthy (at least $85.65\%$). The post-hoc linear probe gives the same separation. Above $8$, accuracy declines again. The weight $\lambda$ and the number of neighborhoods $K$ therefore interact, and at fixed $\lambda$, comparing values of $K$ also compares strengths. With $\lambda$ chosen separately for each $K$, the best $K>1$ and $K{=}1$ are within $0.20$ points of each other on Imagenette ($+0.09$ at $B=64\times4$ and $-0.20$ at $B=256\times1$, post-hoc linear probe). On ImageNet-1K, in contrast, locality improves accuracy beyond its effect on strength (\Cref{sec:experiments}).\par}
\begin{table}[htbp]
\revstart
\centering
\caption{\textit{Effective regularization strength orders the Imagenette grid (800 epochs, both per-step batches).} Each cell of the $(B, \lambda, K)$ grid (three pretraining seeds) is placed by its effective strength $\lambda N/K$, with $N$ the per-step batch. For each band, the table gives the number of cells and the minimum / median / maximum of the cell means (\%), on the final-epoch online probe and on the post-hoc linear probe.}
\label{tab:strength-bands}
\small
\setlength{\tabcolsep}{3pt}
\resizebox{\linewidth}{!}{\begin{tabular}{l l c ccc ccc}
\toprule
 & & & \multicolumn{3}{c}{Online probe (\%)} & \multicolumn{3}{c}{Linear probe (\%)} \\
\cmidrule(lr){4-6}\cmidrule(lr){7-9}
$\lambda N/K$ & regime & cells & min & median & max & min & median & max \\
\midrule
$<0.4$ & collapse & 9 & $18.63$ & $45.53$ & $64.20$ & $33.02$ & $54.06$ & $68.56$ \\ 
$0.4$--$0.64$ & degraded & 5 & $73.41$ & $81.03$ & $88.56$ & $76.69$ & $82.39$ & $88.88$ \\ 
$0.64$--$8$ & healthy & 33 & $85.65$ & $90.25$ & $91.52$ & $86.17$ & $90.08$ & $91.44$ \\ 
$8$--$20$ & declining & 8 & $77.33$ & $88.04$ & $90.53$ & $79.20$ & $88.36$ & $90.30$ \\ 
$>20$ & over-regularized & 5 & $64.46$ & $78.13$ & $87.44$ & $69.53$ & $79.95$ & $87.56$ \\ 
\bottomrule
\end{tabular}}
\end{table}

\paragraph{Reference baselines under their published recipes with $B=64 \times 4$.}
Under their respective training recipes, LeJEPA and RDMReg achieve mean final-epoch online linear-probe accuracies of $89.29\%$ and $81.90\%$, respectively~(\Cref{tab:baselines}). Although their performance is worse than that of the local models ablated above, their differences from the local grid cannot be attributed to locality or target distribution alone, since optimizer settings, regularization weights and centering choices differ. \rev{The} $K=1$ LocaLeJEPA control is not equivalent to SIGReg, as \rev{it matches centered embeddings to the uniform ball, our entropy-maximizing target, while SIGReg matches raw embeddings to the isotropic Gaussian}.

\begin{table}[htbp]
\centering
\caption{\textit{Complete reference runs at $B=64\times4$, 800 epochs.} Final accuracy (\%), mean $\pm$ SD over three seeds. These use their own published recipes, not the controlled local-grid recipe.}
\label{tab:baselines}
\small
\setlength{\tabcolsep}{4pt}
\begin{tabular}{lc}
\toprule
Method & Final-epoch accuracy \\
\midrule
LeJEPA (SIGReg) & \makecell{$89.29 \pm 0.39$} \\
RDMReg & \makecell{$81.90 \pm 0.52$} \\
\bottomrule
\end{tabular}
\end{table}

\section{Additional Experimental Details}
\label{sec:additional-experiments}

This appendix complements \revred{\Cref{sec:experiments}} with the full experimental
specification underlying the reported results. We first consolidate the experimental
conventions shared across studies in \revred{\Cref{app:shared-protocol}}, including the
representation interface, training and checkpointing conventions, and frozen-feature
evaluation protocol. We then provide some more details on the baselines employed in \revred{\Cref{app:objectives}}, the experiment-specific settings and analyses
for the torus geometry-preservation study in \Cref{app:yoda}, followed by the
ImageNet-1K experiments across Euclidean, hyperbolic, and equivariant representation
spaces in \revred{\Cref{app:imagenet-protocol}}.

\subsection{Shared Experimental Protocol}
\label{app:shared-protocol}
Across all experiments, we adopt a common interface between self-supervised pretraining and downstream evaluation. The encoder produces a 512-dimensional representation, which is kept frozen for all downstream readouts, while the self-supervised regularizer is applied after a $2048$--$2048$--$16$ MLP\rev{ (ViT-L/14: $2048$--$2048$--$512$, LeJEPA's own)} 
projector. This separation allows us to evaluate the information retained by the
learned representation independently of the lower-dimensional space in which
distribution matching is enforced.

\textbf{Training Details.}
\rev{Each image contributes four views (ViT-L/14: two global and six local views). On ImageNet-1K, each optimizer step processes 256 images (ViT-L/14: 512) in shards of 64 images, one per GPU, without gradient accumulation. On Imagenette, the per-step batch is stated with each experiment ($64$ images with four accumulation steps, or $256$ with one).} Local neighborhood assignments and distribution statistics, however, are computed independently within each \rev{64-image shard on ImageNet-1K and within each per-step batch on Imagenette}. Experiment-specific optimization choices, including the number of training epochs, learning rate, and regularization parameters, are reported in the corresponding sections below.

\textbf{Online Evaluation and Checkpointing.}
Following~\cite{balestriero2025lejepa}, we train a lightweight classifier on detached encoder features during pretraining as an online diagnostic. Its gradients never propagate through the encoder or projector and therefore do not affect representation learning\rev{.} Each model is trained for the complete prescribed schedule\rev{. The ImageNet-1K tables and the Imagenette ablations of \Cref{sec:ablations} use the final checkpoint, at epoch 100 and epoch 800, respectively. The rotation-recovery study (\Cref{app:yoda}) evaluates each encoder at the checkpoint with the best online-probe accuracy. The COIL-100 analysis (\Cref{app:coil_experiment}) reports both the final checkpoint and the one with the best online-probe accuracy.} 

\textbf{Frozen-feature Evaluation.}
All downstream evaluations keep the pretrained encoder fixed and fit only the
specified readout. Reported differences therefore reflect information already
present in the frozen representation\rev{, since the encoder is not adapted} to a
particular downstream task.

\textbf{Dynamic Neighborhood Schedule.}
\revred{For the neighborhood count,} we also consider a \emph{dynamic variant} in which
the locality of the regularizer is progressively refined during pretraining. At
epoch $e$, the number of neighborhoods is:
\begin{equation}
    K(e)
    =
    \min\!\left(
        K_{\max},
        K_{\mathrm{start}}
        +
        \left\lfloor \frac{e}{E_{\mathrm{step}}} \right\rfloor
    \right)
\end{equation}
with $K_{\mathrm{start}}{=}1$ throughout. Training therefore begins with a single
neighborhood and introduces progressively finer local partitions as the
representation develops. We use the same scheduling rule across experiments,
while $K_{\max}$ and $E_{\mathrm{step}}$ are adapted to the corresponding
pretraining horizon and specified in each experiment-specific section.

\subsection{Distribution-Matching Objectives and Reference Distributions}
\label{app:objectives}

Our comparisons span distribution-matching objectives that differ in both the
domain on which regularization is imposed and the reference distribution toward
which representations are driven. We make this important distinction in our evaluation because LocaLeJEPA changes the regularization from a single global constraint to a collection of locally centered distribution-matching problems, while preserving the same self-supervised prediction objective.

\textbf{Global baselines.}
\emph{SIGReg}~\revred{\citep{balestriero2025lejepa}}, the regularizer underlying LeJEPA,
matches projected representations globally to an isotropic Gaussian.
\emph{SUSReg}~\revred{\citep{nicollier2026spherejepa}} instead promotes uniformity on the
hypersphere, providing a reference distribution adapted to spherical embeddings.
\emph{SubReg}~\revred{\citep{zhao2026subjepa}} applies Gaussian regularization within
frozen random feature subspaces \rev{of} the complete projected
representation. Finally, \emph{RDMReg}, introduced by Rectified LpJEPA
\revred{\citep{kuang2026rectifiedlpjepajointembeddingpredictive}}, matches rectified
representations to a generalized-Gaussian target\rev{. We} use $p{=}1$, corresponding
to the rectified Laplace distribution.

Except for RDMReg, these objectives can be written as:
\begin{equation}
    \mathcal L
    =
    (1-\lambda)\,\mathcal L_{\mathrm{pred}}
    +
    \lambda\,\mathcal R
\end{equation}
where $\mathcal L_{\mathrm{pred}}$ \rev{matches} the representations of the different
views of each image and $\mathcal R$ denotes the corresponding
distribution-matching regularizer. RDMReg instead follows the weighting proposed
by its authors:
\begin{equation}
    \mathcal L_{\mathrm{RDM}}
    =
    25\,\mathcal L_{\mathrm{pred}}
    +
    125\,\mathcal R_{\mathrm{RDM}}
\end{equation}
whose coefficients are absolute and therefore not directly comparable to the
mixing parameter $\lambda$ used by the other objectives.

\textbf{LocaLeJEPA and Local Control.}
LocaLeJEPA replaces the global reference constraint with distribution matching
within learned neighborhoods. \rev{Projected representations are averaged over views and partitioned once per step} by cosine-similarity $K$-means using at most ten iterations\rev{, and the partition is shared by all views}. In order to avoid uninformative clusters, we discard clusters containing fewer than four samples, and the regularizer is averaged
uniformly across the remaining clusters. Each cluster is centered independently, and the resulting residuals are matched to the uniform distribution on the $d$-dimensional Euclidean ball of radius:
\begin{equation}
    \revred{R}=\sqrt{d+2}
\end{equation}
which has identity covariance. With the $d{=}16$ projector used throughout our experiments, this gives $\revred{R}=\sqrt{18}$.

We \rev{also} evaluate \emph{LocaLeJEPA-SIGReg} as a local control. It uses the same neighborhood partition as LocaLeJEPA but replaces the centered bounded reference with Gaussian matching on the uncentered cluster representations. \rev{This} comparison tests whether local partitioning alone is sufficient
to explain the gains of LocaLeJEPA, while not separately isolating the effects of
centering and reference distribution.

\subsection{Toroidal Geometry Preservation}
\label{app:yoda}

This section specifies the experimental choices particular to the torus geometry-preservation study of \Cref{fig:yoda-torus-main}. We describe the
Imagenette pretraining configuration and circular-regression protocol, followed
by the complete ablations over neighborhood granularity, regularization strength,
and probe-label budget. Unless stated otherwise, training, checkpointing, and frozen-feature evaluation follow \Cref{app:shared-protocol}, while the compared regularizers are defined in \Cref{app:objectives}.

\textbf{Imagenette Pretraining.}
\label{app:yoda-pretrain}
All encoders evaluated in \Cref{tab:yoda-full} are pretrained from scratch
for 800 epochs on Imagenette, the ten-class ImageNet subset
\revred{\citep{Howard_Imagenette_2019}}, using ViT-S/8 variants with 384 hidden
dimensions. \revred{Besides the standard ViT, two variants keep every token representation on the unit hypersphere. nViT is a ViT built on the normalized Transformer of \citet{loshchilov2025ngpt}, and Spherical ViT replaces its learned positional embeddings with spherical-harmonic encodings.} Global SIGReg \citep{balestriero2025lejepa}, SUSReg \citep{nicollier2026spherejepa}, and SubReg \citep{zhao2026subjepa} use a learning rate of
$2\times10^{-3}$ and one warmup epoch, with $\lambda{=}0.02$.
RDMReg follows the same global optimization recipe, with the objective weighting
defined in \revred{\Cref{app:objectives}}. LocaLeJEPA and its local control use a
learning rate of $10^{-3}$, ten warmup epochs, and gradient clipping at $1.0$. We report the corresponding values of $\lambda$ and $K$ for each specific configuration. For the dynamic-neighborhood variant, we instantiate the schedule of \revred{\Cref{app:shared-protocol}} with $K_{\max}=8$ and $E_{\mathrm{step}}=100$, so that the partition evolves from $K{=}1$ to $K{=}8$, reaching its maximum granularity at epoch 700. All models are evaluated \rev{at the checkpoint with the best online-probe accuracy (\Cref{app:shared-protocol})}. Complete optimization settings are reported in Table~\ref{tab:yoda-pretrain}.
\begin{table}[ht]
\centering
\footnotesize
\caption{\emph{Imagenette Pretraining Settings.} The upper block lists settings shared by all encoders, while the remaining blocks give the optimization parameters and loss weights for each objective family.}
\label{tab:yoda-pretrain}
\begin{tabular}{ll}
\toprule
\multicolumn{2}{l}{\emph{Shared across all backbones}} \\
Dataset & Imagenette, 160\,px variant (10 classes) \\
Training horizon & 800 epochs \\
Batch / accumulation & 64 images per loss / 4 steps (256 per update) \\
Encoder & Small ViT variants, patch 8\rev{,} 384 internal / 512 output dims \\
Projector & MLP $2048$-$2048$-$16$ with BatchNorm \\
Views per image & 4 \\
Crop & \texttt{RandomResizedCrop} $128{\times}128$, scale $[0.08,1.0]$ \\
Augmentations & h-flip\rev{,} jitter $p{=}0.8$\rev{,} grayscale $p{=}0.2$\rev{,} \\
              & blur $p{=}0.5$\rev{,} solarize $p{=}\rev{0.1}$\rev{,} ImageNet norm \\
Optimizer & AdamW, weight decay $5\times10^{-2}$ \\
 & \rev{Spherical ViT and nViT: Adam, no weight decay} \\
\rev{LR schedule} & \rev{linear warm-up, then cosine decay to half the peak} \\
 & \rev{local recipe: constant after warm-up} \\
Online probe & LayerNorm $+$ Linear, LR $10^{-3}$, wd $10^{-7}$ \\
\midrule
\multicolumn{2}{l}{\emph{Global} (SIGReg / SUSReg / SubReg)} \\
LR / warmup / grad-clip & $2\times10^{-3}$ / 1 ep / - \\
Loss balance & $\lambda=0.02$ \\
\midrule
\multicolumn{2}{l}{\emph{Local} (LocaLeReg / SIGReg) - stabilized recipe} \\
LR / warmup / grad-clip & $10^{-3}$ / 10 ep / $1.0$ \\
Loss balance & $\lambda\in\{0.1,0.3\}$ in the reported angle table \\
\midrule
\multicolumn{2}{l}{\emph{Rectified LpJEPA} (RDMReg)} \\
LR / warmup / grad-clip & $2\times10^{-3}$ / 1 ep / - \\
Loss balance & $25\,\mathcal{L}_{\mathrm{inv}}+125\,\mathrm{RDMReg}$ \\
\bottomrule
\end{tabular}
\end{table}

\textbf{Circular Regression and Data Splits.}
\label{app:yoda-proto}
The bobblehead sequence contains 8{,}097 RGB frames of size $240\times320$, each associated with two independent rotation angles. Frames are resized to $128\times128$ and encoded once using the frozen encoder\rev{ at its best online-probe checkpoint}. All
subsequent analyses operate on these fixed representations. For each seed $s\in\{0,\ldots,9\}$, the frames are randomly permuted and split into 5{,}182 training, 1{,}296 validation, and 1{,}619 test examples. The same
ten splits are shared across all encoders, and the pretrained representation is held fixed throughout, so seeds vary only the downstream probe partition. For a probe-training budget $n$, we use the first $\min(n,5182)$ examples of the corresponding shuffled training pool. Budgets are therefore nested within each seed, while its validation and test sets remain unchanged. Features are standardized using statistics computed exclusively from the selected training subset. For each rotation coordinate $\theta$, an independent ridge regressor predicts $(\cos\theta,\sin\theta)$ together with an unregularized intercept. The ridge coefficient is selected from $  \left\{
    10^{-2},10^{-1},1,10,10^{2},10^{3}
    \right\}$, using angular MAE on the complete 1{,}296-frame validation split and is then evaluated once on the corresponding test set.

Given predictions $(\hat c,\hat s)$, we recover $\hat\theta=\operatorname{atan2}(\hat s,\hat c)$ and report the wrapped angular MAE:
\begin{equation}
    \operatorname{MAE}
    =
    \frac{180}{\pi N_{\mathrm{test}}}
    \sum_{i=1}^{N_{\mathrm{test}}}
    \left|
    \operatorname{Arg}
    \left(
        e^{\mathrm{i}(\hat\theta_i-\theta_i)}
    \right)
    \right|
\end{equation}
The reported $R^2$ pools residual sums of squares over the cosine and sine targets. Results are means and sample standard deviations over the ten shared probe splits. Because each configuration uses a fixed pretrained encoder, these deviations quantify probe-split variability \rev{and do not include} variability across independent pretraining runs. The random split holds out individual frames\rev{, not} complete objects, trajectories, or angular sectors. The experiment therefore measures recovery and interpolation of the two periodic factors on the observed sequence\rev{. It does not test} extrapolation to unseen objects or disjoint regions of $\mathbb T^2$.

\begin{table}[thbp]
\centering
\caption{\emph{Rotation recovery after Imagenette pretraining, All Configurations.} Same protocol as \rev{Figure}~\ref{fig:yoda-torus-main}, with every neighborhood count $K$, weight $\lambda$, and the local Gaussian control (LocaLeJEPA-SIGReg: cluster-wise SIGReg on uncentered embeddings).}
\label{tab:yoda-full}
\small
\setlength{\tabcolsep}{5pt}
\resizebox{\linewidth}{!}{\begin{tabular}{llcccc}
\toprule
 & & \multicolumn{2}{c}{Figurine 0} & \multicolumn{2}{c}{Figurine 1} \\
\cmidrule(lr){3-4}\cmidrule(lr){5-6}
Encoder & Regularizer & $R^2\uparrow$ & MAE ($^\circ$)$\downarrow$ & $R^2\uparrow$ & MAE ($^\circ$)$\downarrow$ \\
\midrule
\multicolumn{6}{l}{\emph{Global distribution matching}} \\
Spherical ViT & SUSReg & $0.9896 \pm 0.0003$ & $3.25 \pm 0.05$ & $0.9738 \pm 0.0010$ & $4.04 \pm 0.07$ \\
nViT & SIGReg & $0.9897 \pm 0.0003$ & $3.21 \pm 0.09$ & $0.9678 \pm 0.0015$ & $4.37 \pm 0.08$ \\
nViT & SubReg & $0.9891 \pm 0.0002$ & $3.26 \pm 0.03$ & $0.9643 \pm 0.0012$ & $4.63 \pm 0.11$ \\
ViT & SIGReg & $0.9930 \pm 0.0002$ & $2.75 \pm 0.06$ & $0.9775 \pm 0.0008$ & $3.85 \pm 0.08$ \\
ViT & SubReg & $0.9926 \pm 0.0002$ & $2.84 \pm 0.06$ & $0.9758 \pm 0.0010$ & $3.97 \pm 0.07$ \\
ViT & RDMReg & $0.9899 \pm 0.0001$ & $3.33 \pm 0.03$ & $0.9714 \pm 0.0012$ & $4.25 \pm 0.09$ \\
\midrule
\multicolumn{6}{l}{\emph{Local distribution matching}} \\
ViT & LocaLeJEPA ($\lambda{=}0.1$, $K{=}2$) & $\mathbf{0.9959 \pm 0.0001}$ & $\mathbf{2.13 \pm 0.04}$ & $0.9840 \pm 0.0007$ & $3.21 \pm 0.09$ \\
ViT & LocaLeJEPA ($\lambda{=}0.1$, $K{=}1{\to}8$) & $0.9938 \pm 0.0002$ & $2.58 \pm 0.05$ & $\mathbf{0.9841 \pm 0.0003}$ & $\mathbf{3.20 \pm 0.08}$ \\
ViT & LocaLeJEPA ($\lambda{=}0.1$, $K{=}4$) & $0.9946 \pm 0.0001$ & $2.44 \pm 0.04$ & $0.9821 \pm 0.0003$ & $3.42 \pm 0.07$ \\
ViT & LocaLeJEPA ($\lambda{=}0.3$, $K{=}2$) & $0.9919 \pm 0.0002$ & $2.89 \pm 0.04$ & $0.9808 \pm 0.0007$ & $3.48 \pm 0.07$ \\
ViT & LocaLeJEPA-SIGReg ($\lambda{=}0.1$, $K{=}4$) & $0.9919 \pm 0.0002$ & $2.94 \pm 0.07$ & $0.9770 \pm 0.0009$ & $3.84 \pm 0.08$ \\
\bottomrule
\end{tabular}}
\end{table}

\subsubsection{Complete Torus Results at Full Probe Budget}
\label{app:torus-complete-table}
In this subsection, we provide the complete full-budget comparison of Table~\ref{tab:yoda-budget-main} in \Cref{sec:experiments}, together with ablations over the regularization weight $\lambda \in \{0.1,0.3\}$ and the neighborhood count $K \in \{2,4,1{\to}8\}$. 

\textbf{Results.} As can be observed in Table~\ref{tab:yoda-full}, across the three LocaLeJEPA variants with $\lambda{=}0.1$, local matching improves both rotation coordinates over every global baseline. Fixed $K{=}2$ gives the
strongest overall recovery, reaching $2.13^\circ$ and $3.21^\circ$ MAE on
$\theta_0$ and $\theta_1$, respectively. Increasing the partition granularity to $K{=}4$ degrades both readouts to $2.44^\circ$ and $3.42^\circ$\rev{, so}
finer neighborhoods are not uniformly beneficial. The dynamic schedule $K{=}1{\to}8$ \rev{is worse than both fixed settings} on $\theta_0$ (\rev{$2.58^\circ$ versus $2.13^\circ$ and $2.44^\circ$}), while attaining the best $\theta_1$ result ($R^2{=}0.9841$,
MAE $3.20^\circ$). Thus, progressively increasing locality can match the best fixed setting on one factor, but does not improve the aggregate recovery over
$K{=}2$. Performance is also sensitive to the strength of the local regularizer. At fixed $K{=}2$, increasing $\lambda$ from $0.1$ to $0.3$ raises MAE from
$2.13^\circ$ to $2.89^\circ$ on $\theta_0$ and from $3.21^\circ$ to $3.48^\circ$ on $\theta_1$. The lower-weight configuration is therefore clearly preferred for this readout: with $\lambda{=}0.3$, performance on $\theta_0$ falls slightly below global ViT-SIGReg ($2.89^\circ$ versus $2.75^\circ$), although the advantage on $\theta_1$ remains ($3.48^\circ$ versus
$3.85^\circ$).

In this ablation, we also consider \emph{LocaLeJEPA-SIGReg} as introduced in \Cref{app:objectives}, which tests whether local neighborhoods alone are sufficient. Keeping $\lambda{=}0.1$ and $K{=}4$ but replacing the centered ball construction with cluster-wise SIGReg on \rev{uncentered} embeddings yields $2.94^\circ$ and $3.84^\circ$ MAE, effectively eliminating the improvement over global ViT-SIGReg. This shows that the gains cannot be explained by locality alone\rev{ and that the centered ball construction contributes to them}. However, the experiment does not disentangle \rev{the contributions of centering and of the reference distribution}.

\subsubsection{Probe Label-Budget Curves}
\label{app:yoda-lc}
In this subsection we extend \Cref{tab:yoda-budget-main} and ablate over multiple probe-training budgets, together with a subset of parameters $(\lambda, K) \in \{0.1,0.3\}\times\{2,4,1{\to}8\} $. We analyze up to ten budgets, from 5 labeled frames to the full 5{,}182-frame training pool, and report MAE and $R^2$ separately for both rotation coordinates. All values are means $\pm$ sample standard deviations over the same ten probe splits used in the main paper, with the pretrained encoder fixed throughout. 

\textbf{Results.} Across the two coordinates, LocaLeJEPA with $\lambda{=}0.1$ and fixed $K{=}2$ achieves the lowest mean angular error at every budget from 10 frames onward. The improvement becomes increasingly clear as the budget increases. Relative to global ViT-SIGReg, the two-coordinate mean MAE decreases from $39.2^\circ$ to $36.0^\circ$ with 10 labels, from $17.8^\circ$ to $15.5^\circ$
with 50, and from $11.0^\circ$ to $9.0^\circ$ with 100. At 500 frames and at full supervision, the corresponding errors decrease from $4.70^\circ$ to $3.80^\circ$ and from $3.30^\circ$ to $2.67^\circ$, respectively. The coordinate-wise curves \rev{are less uniform in the} low-data regime. For $\theta_0$, fixed-$K$ LocaLeJEPA attains the lowest MAE from 10 frames onward and remains best throughout the remainder of the curve
(\Cref{tab:yoda-lc-mae0}). For $\theta_1$, the leading method varies at small budgets: RDMReg is best with 5 frames, ViT-SubReg with 10 and 20, dynamic-$K$ LocaLeJEPA with 50, and the $\lambda{=}0.3$, $K{=}2$ variant with 100 (\Cref{tab:yoda-lc-mae1}). Once the probe has at least 500 training frames, dynamic-$K$ gives the lowest $\theta_1$ error, reaching $3.20^\circ$ at the full budget compared with $3.21^\circ$ for fixed $K{=}2$. Thus, the consistent advantage of fixed $K{=}2$ in the two-coordinate mean does not imply uniform dominance on each rotation factor individually. The $R^2$ curves mirror the MAE results. Fixed-$K$ LocaLeJEPA gives the strongest $\theta_0$ recovery from 10 frames onward, reaching $R^2=0.996$ at full supervision, while the preferred configuration for $\theta_1$ depends more strongly on the available label budget. At five training frames, performance is close to the floor of the task: no configuration exceeds $R^2=0.08$ on $\theta_0$, and all methods are near zero or negative on $\theta_1$. All comparisons in Table~\ref{tab:yoda-lc} concern variability of the downstream
probe only. A summary of the ablation over $\lambda$ and $K$ can be found also inside \Cref{tab:yoda-budget}.
\begin{table}[thbp]
\centering
\caption{\emph{Mean Rotation-Recovery Error across Probe-label Budgets.}
Wrapped angular MAE ($^\circ$) averaged over the two rotation coordinates as the number of labeled frames used to fit the ridge probe increases. Results are mean $\pm$ sample standard deviation over ten shared probe splits, with the two coordinates averaged within each split. Per-coordinate results are reported in Table~\ref{tab:yoda-lc}. Bold marks the lowest mean at each budget.}
\label{tab:yoda-budget}
\small
\setlength{\tabcolsep}{4pt}
\resizebox{\linewidth}{!}{\begin{tabular}{lccccc}
\toprule
Method & 10 & 50 & 100 & 500 & Full \\
\midrule
Spherical ViT SUSReg & $44.6 \pm 4.3$ & $20.2 \pm 2.0$ & $11.0 \pm 0.7$ & $4.91 \pm 0.09$ & $3.64 \pm 0.04$ \\
nViT SIGReg & $41.8 \pm 4.9$ & $23.1 \pm 2.0$ & $11.5 \pm 0.5$ & $5.23 \pm 0.10$ & $3.79 \pm 0.06$ \\
nViT SubReg & $46.0 \pm 4.5$ & $26.7 \pm 1.7$ & $14.6 \pm 0.7$ & $5.65 \pm 0.12$ & $3.94 \pm 0.05$ \\
ViT SIGReg & $39.2 \pm 3.6$ & $17.8 \pm 1.1$ & $11.0 \pm 0.4$ & $4.70 \pm 0.09$ & $3.30 \pm 0.05$ \\
ViT SubReg & $37.9 \pm 4.4$ & $18.0 \pm 1.7$ & $11.2 \pm 0.5$ & $4.70 \pm 0.10$ & $3.40 \pm 0.04$ \\
ViT RDMReg & $39.8 \pm 7.0$ & $19.8 \pm 2.0$ & $10.5 \pm 0.6$ & $5.47 \pm 0.08$ & $3.79 \pm 0.05$ \\
\midrule
ViT LocaLeJEPA ($\lambda{=}0.1$, $K{=}2$) & $\mathbf{36.0 \pm 3.8}$ & $\mathbf{15.5 \pm 1.2}$ & $\mathbf{9.0 \pm 0.3}$ & $\mathbf{3.80 \pm 0.05}$ & $\mathbf{2.67 \pm 0.05}$ \\
ViT LocaLeJEPA ($\lambda{=}0.1$, $K{=}1{\to}8$) & $37.5 \pm 4.3$ & $17.1 \pm 1.1$ & $10.9 \pm 0.5$ & $4.10 \pm 0.10$ & $2.89 \pm 0.04$ \\
ViT LocaLeJEPA ($\lambda{=}0.1$, $K{=}4$) & $37.6 \pm 4.7$ & $19.2 \pm 1.8$ & $9.6 \pm 0.4$ & $4.07 \pm 0.08$ & $2.93 \pm 0.03$ \\
ViT LocaLeJEPA ($\lambda{=}0.3$, $K{=}2$) & $40.3 \pm 3.3$ & $22.1 \pm 1.6$ & $9.7 \pm 0.3$ & $4.66 \pm 0.07$ & $3.19 \pm 0.04$ \\
ViT LocaLeJEPA-SIGReg ($\lambda{=}0.1$, $K{=}4$) & $41.5 \pm 5.0$ & $20.5 \pm 1.3$ & $9.8 \pm 0.8$ & $4.75 \pm 0.06$ & $3.39 \pm 0.05$ \\
\bottomrule
\end{tabular}}
\end{table}

\begin{table*}[h]
\centering
\caption{\emph{Rotation Recovery across Probe-label Budgets.}
Circular ridge probes are fitted on frozen representations using different labeled frame budgets $\in\{5, 10, 20, 50, 100, 500, 800, 2000, 4000, 5182 \}$, with \emph{Full=5182} denoting the complete training pool. Panels report angular MAE and $R^2$ separately for the two rotation coordinates. Ridge strength is selected on the same 1{,}296-frame validation set at every budget. Results are mean $\pm$ sample standard deviation over ten shared probe splits\rev{. Bold} marks the best mean in each column. Complete encoder configurations are given in \Cref{tab:yoda-full}.}
\label{tab:yoda-lc}
\small
\begin{subtable}{\textwidth}
\centering
\caption{Figurine 0. Test MAE ($^\circ$, $\downarrow$)}
\label{tab:yoda-lc-mae0}
\resizebox{\textwidth}{!}{\begin{tabular}{lcccccccccc}
\toprule
Method & 5 & 10 & 20 & 50 & 100 & 500 & 800 & 2000 & 4000 & Full \\
\midrule
\revred{Spherical ViT} SUSReg & $73.69 \pm 10.20$ & $52.70 \pm 7.03$ & $42.24 \pm 4.85$ & $18.90 \pm 2.34$ & $10.68 \pm 0.80$ & $4.53 \pm 0.15$ & $4.07 \pm 0.10$ & $3.57 \pm 0.06$ & $3.32 \pm 0.05$ & $3.25 \pm 0.05$ \\
\revred{nViT} SIGReg & $62.17 \pm 13.10$ & $42.51 \pm 7.38$ & $31.74 \pm 2.74$ & $18.84 \pm 3.00$ & $9.43 \pm 0.60$ & $4.37 \pm 0.15$ & $3.94 \pm 0.11$ & $3.49 \pm 0.07$ & $3.28 \pm 0.09$ & $3.21 \pm 0.09$ \\
\revred{nViT} SubReg & $66.39 \pm 8.12$ & $50.91 \pm 7.13$ & $39.64 \pm 4.80$ & $25.49 \pm 2.25$ & $12.94 \pm 0.87$ & $4.77 \pm 0.07$ & $4.18 \pm 0.06$ & $3.58 \pm 0.06$ & $3.33 \pm 0.03$ & $3.26 \pm 0.03$ \\
\revred{ViT} SIGReg & $60.46 \pm 9.80$ & $40.02 \pm 7.95$ & $25.79 \pm 6.50$ & $15.24 \pm 0.95$ & $9.50 \pm 0.61$ & $4.01 \pm 0.13$ & $3.59 \pm 0.11$ & $3.06 \pm 0.08$ & $2.83 \pm 0.07$ & $2.75 \pm 0.06$ \\
\revred{ViT} SubReg & $56.99 \pm 8.76$ & $41.60 \pm 7.97$ & $29.86 \pm 3.99$ & $18.30 \pm 2.44$ & $10.17 \pm 0.85$ & $3.97 \pm 0.12$ & $3.57 \pm 0.06$ & $3.14 \pm 0.07$ & $2.92 \pm 0.06$ & $2.84 \pm 0.06$ \\
\revred{ViT} RDMReg & $69.09 \pm 13.24$ & $44.97 \pm 12.60$ & $29.27 \pm 4.07$ & $16.73 \pm 2.34$ & $9.14 \pm 0.54$ & $4.94 \pm 0.11$ & $4.42 \pm 0.10$ & $3.77 \pm 0.06$ & $3.43 \pm 0.04$ & $3.33 \pm 0.03$ \\
\midrule
\revred{ViT LocaLeJEPA} ($\lambda{=}0.1$, $K{=}2$) & $60.98 \pm 10.18$ & $\mathbf{34.90 \pm 6.01}$ & $\mathbf{20.18 \pm 2.96}$ & $\mathbf{12.37 \pm 1.30}$ & $\mathbf{7.31 \pm 0.53}$ & $\mathbf{3.08 \pm 0.08}$ & $\mathbf{2.71 \pm 0.06}$ & $\mathbf{2.33 \pm 0.04}$ & $\mathbf{2.17 \pm 0.04}$ & $\mathbf{2.13 \pm 0.04}$ \\
\revred{ViT LocaLeJEPA} ($\lambda{=}0.1$, $K{=}1{\to}8$) & $\mathbf{56.74 \pm 8.90}$ & $39.87 \pm 8.14$ & $27.53 \pm 6.69$ & $16.91 \pm 1.07$ & $10.08 \pm 0.64$ & $3.74 \pm 0.07$ & $3.29 \pm 0.08$ & $2.83 \pm 0.06$ & $2.63 \pm 0.06$ & $2.58 \pm 0.05$ \\
\revred{ViT LocaLeJEPA} ($\lambda{=}0.1$, $K{=}4$) & $59.09 \pm 8.22$ & $40.83 \pm 7.16$ & $27.27 \pm 4.44$ & $17.94 \pm 3.03$ & $7.98 \pm 0.37$ & $3.48 \pm 0.07$ & $3.09 \pm 0.04$ & $2.66 \pm 0.04$ & $2.50 \pm 0.04$ & $2.44 \pm 0.04$ \\
\revred{ViT LocaLeJEPA} ($\lambda{=}0.3$, $K{=}2$) & $60.37 \pm 8.25$ & $45.13 \pm 5.35$ & $29.27 \pm 4.11$ & $21.55 \pm 1.78$ & $9.12 \pm 0.54$ & $4.50 \pm 0.12$ & $3.99 \pm 0.07$ & $3.30 \pm 0.05$ & $2.98 \pm 0.04$ & $2.89 \pm 0.04$ \\
\revred{ViT LocaLeJEPA}-SIGReg ($\lambda{=}0.1$, $K{=}4$) & $63.91 \pm 11.89$ & $43.39 \pm 6.62$ & $31.05 \pm 4.63$ & $18.92 \pm 1.83$ & $8.97 \pm 1.13$ & $4.24 \pm 0.10$ & $3.83 \pm 0.08$ & $3.30 \pm 0.07$ & $3.02 \pm 0.07$ & $2.94 \pm 0.07$ \\
\bottomrule
\end{tabular}}
\end{subtable}
\vspace{4pt}
\begin{subtable}{\textwidth}
\centering
\caption{Figurine 1. Test MAE ($^\circ$, $\downarrow$)}
\label{tab:yoda-lc-mae1}
\resizebox{\textwidth}{!}{\begin{tabular}{lcccccccccc}
\toprule
Method & 5 & 10 & 20 & 50 & 100 & 500 & 800 & 2000 & 4000 & Full \\
\midrule
\revred{Spherical ViT} SUSReg & $44.72 \pm 16.31$ & $36.42 \pm 5.34$ & $30.06 \pm 3.67$ & $21.50 \pm 2.25$ & $11.30 \pm 0.90$ & $5.29 \pm 0.13$ & $4.90 \pm 0.11$ & $4.35 \pm 0.08$ & $4.12 \pm 0.06$ & $4.04 \pm 0.07$ \\
\revred{nViT} SIGReg & $49.03 \pm 12.13$ & $41.14 \pm 5.63$ & $37.05 \pm 4.63$ & $27.41 \pm 2.95$ & $13.48 \pm 0.84$ & $6.08 \pm 0.11$ & $5.40 \pm 0.10$ & $4.71 \pm 0.09$ & $4.44 \pm 0.09$ & $4.37 \pm 0.08$ \\
\revred{nViT} SubReg & $48.08 \pm 12.75$ & $41.04 \pm 5.51$ & $38.27 \pm 6.95$ & $27.95 \pm 2.71$ & $16.34 \pm 1.14$ & $6.53 \pm 0.22$ & $5.78 \pm 0.17$ & $5.03 \pm 0.12$ & $4.72 \pm 0.10$ & $4.63 \pm 0.11$ \\
\revred{ViT} SIGReg & $42.66 \pm 10.10$ & $38.45 \pm 3.45$ & $30.76 \pm 3.04$ & $20.40 \pm 2.30$ & $12.47 \pm 0.64$ & $5.39 \pm 0.17$ & $4.86 \pm 0.16$ & $4.25 \pm 0.09$ & $3.95 \pm 0.08$ & $3.85 \pm 0.08$ \\
\revred{ViT} SubReg & $43.27 \pm 8.44$ & $\mathbf{34.18 \pm 3.03}$ & $\mathbf{25.52 \pm 4.82}$ & $17.79 \pm 2.18$ & $12.18 \pm 0.71$ & $5.42 \pm 0.15$ & $4.91 \pm 0.10$ & $4.33 \pm 0.09$ & $4.05 \pm 0.07$ & $3.97 \pm 0.07$ \\
\revred{ViT} RDMReg & $\mathbf{37.11 \pm 5.14}$ & $34.71 \pm 5.75$ & $30.30 \pm 2.91$ & $22.87 \pm 2.30$ & $11.84 \pm 1.26$ & $6.01 \pm 0.12$ & $5.48 \pm 0.08$ & $4.76 \pm 0.09$ & $4.38 \pm 0.09$ & $4.25 \pm 0.09$ \\
\midrule
\revred{ViT LocaLeJEPA} ($\lambda{=}0.1$, $K{=}2$) & $43.38 \pm 10.17$ & $37.11 \pm 4.53$ & $27.84 \pm 5.95$ & $18.71 \pm 2.38$ & $10.67 \pm 0.65$ & $4.52 \pm 0.08$ & $4.05 \pm 0.09$ & $3.49 \pm 0.09$ & $3.26 \pm 0.09$ & $3.21 \pm 0.09$ \\
\revred{ViT LocaLeJEPA} ($\lambda{=}0.1$, $K{=}1{\to}8$) & $43.61 \pm 10.21$ & $35.14 \pm 4.66$ & $26.95 \pm 2.95$ & $\mathbf{17.21 \pm 1.78}$ & $11.81 \pm 1.02$ & $\mathbf{4.47 \pm 0.19}$ & $\mathbf{4.01 \pm 0.10}$ & $\mathbf{3.48 \pm 0.08}$ & $\mathbf{3.26 \pm 0.08}$ & $\mathbf{3.20 \pm 0.08}$ \\
\revred{ViT LocaLeJEPA} ($\lambda{=}0.1$, $K{=}4$) & $41.99 \pm 9.89$ & $34.37 \pm 3.81$ & $27.69 \pm 4.20$ & $20.50 \pm 2.38$ & $11.15 \pm 0.81$ & $4.65 \pm 0.15$ & $4.20 \pm 0.10$ & $3.70 \pm 0.09$ & $3.48 \pm 0.08$ & $3.42 \pm 0.07$ \\
\revred{ViT LocaLeJEPA} ($\lambda{=}0.3$, $K{=}2$) & $39.12 \pm 4.90$ & $35.41 \pm 3.49$ & $31.10 \pm 2.79$ & $22.64 \pm 1.94$ & $\mathbf{10.36 \pm 0.51}$ & $4.83 \pm 0.12$ & $4.36 \pm 0.11$ & $3.82 \pm 0.09$ & $3.56 \pm 0.07$ & $3.48 \pm 0.07$ \\
\revred{ViT LocaLeJEPA}-SIGReg ($\lambda{=}0.1$, $K{=}4$) & $42.44 \pm 7.86$ & $39.61 \pm 5.12$ & $32.31 \pm 3.89$ & $21.98 \pm 1.77$ & $10.55 \pm 0.81$ & $5.25 \pm 0.07$ & $4.78 \pm 0.10$ & $4.22 \pm 0.10$ & $3.92 \pm 0.09$ & $3.84 \pm 0.08$ \\
\bottomrule
\end{tabular}}
\end{subtable}
\vspace{4pt}
\begin{subtable}{\textwidth}
\centering
\caption{Figurine 0. Test $R^2$ \revred{($\uparrow$)}}
\label{tab:yoda-lc-r2-0}
\resizebox{\textwidth}{!}{\begin{tabular}{lcccccccccc}
\toprule
Method & 5 & 10 & 20 & 50 & 100 & 500 & 800 & 2000 & 4000 & Full \\
\midrule
\revred{Spherical ViT} SUSReg & $-0.270 \pm 0.213$ & $0.182 \pm 0.207$ & $0.324 \pm 0.258$ & $0.743 \pm 0.045$ & $0.897 \pm 0.010$ & $0.980 \pm 0.001$ & $0.984 \pm 0.001$ & $0.987 \pm 0.000$ & $0.989 \pm 0.000$ & $0.990 \pm 0.000$ \\
\revred{nViT} SIGReg & $-0.058 \pm 0.294$ & $0.352 \pm 0.188$ & $0.545 \pm 0.043$ & $0.759 \pm 0.042$ & $0.920 \pm 0.008$ & $0.981 \pm 0.001$ & $0.985 \pm 0.001$ & $0.988 \pm 0.000$ & $0.989 \pm 0.000$ & $0.990 \pm 0.000$ \\
\revred{nViT} SubReg & $-0.227 \pm 0.315$ & $0.169 \pm 0.162$ & $0.378 \pm 0.141$ & $0.612 \pm 0.043$ & $0.857 \pm 0.014$ & $0.977 \pm 0.001$ & $0.982 \pm 0.000$ & $0.987 \pm 0.000$ & $0.989 \pm 0.000$ & $0.989 \pm 0.000$ \\
\revred{ViT} SIGReg & $-0.034 \pm 0.234$ & $0.362 \pm 0.198$ & $0.642 \pm 0.122$ & $0.825 \pm 0.021$ & $0.915 \pm 0.009$ & $0.985 \pm 0.001$ & $0.988 \pm 0.001$ & $0.991 \pm 0.000$ & $0.993 \pm 0.000$ & $0.993 \pm 0.000$ \\
\revred{ViT} SubReg & $\mathbf{0.077 \pm 0.258}$ & $0.315 \pm 0.285$ & $0.583 \pm 0.066$ & $0.772 \pm 0.043$ & $0.907 \pm 0.011$ & $0.985 \pm 0.000$ & $0.988 \pm 0.000$ & $0.991 \pm 0.000$ & $0.992 \pm 0.000$ & $0.993 \pm 0.000$ \\
\revred{ViT} RDMReg & $-0.301 \pm 0.419$ & $0.219 \pm 0.285$ & $0.588 \pm 0.085$ & $0.792 \pm 0.040$ & $0.920 \pm 0.009$ & $0.978 \pm 0.001$ & $0.982 \pm 0.001$ & $0.987 \pm 0.000$ & $0.989 \pm 0.000$ & $0.990 \pm 0.000$ \\
\midrule
\revred{ViT LocaLeJEPA} ($\lambda{=}0.1$, $K{=}2$) & $-0.040 \pm 0.256$ & $\mathbf{0.486 \pm 0.140}$ & $\mathbf{0.741 \pm 0.045}$ & $\mathbf{0.869 \pm 0.022}$ & $\mathbf{0.947 \pm 0.007}$ & $\mathbf{0.991 \pm 0.000}$ & $\mathbf{0.993 \pm 0.000}$ & $\mathbf{0.995 \pm 0.000}$ & $\mathbf{0.996 \pm 0.000}$ & $\mathbf{0.996 \pm 0.000}$ \\
\revred{ViT LocaLeJEPA} ($\lambda{=}0.1$, $K{=}1{\to}8$) & $0.029 \pm 0.231$ & $0.366 \pm 0.189$ & $0.595 \pm 0.109$ & $0.785 \pm 0.027$ & $0.900 \pm 0.012$ & $0.986 \pm 0.000$ & $0.989 \pm 0.000$ & $0.992 \pm 0.000$ & $0.994 \pm 0.000$ & $0.994 \pm 0.000$ \\
\revred{ViT LocaLeJEPA} ($\lambda{=}0.1$, $K{=}4$) & $0.025 \pm 0.223$ & $0.390 \pm 0.146$ & $0.639 \pm 0.062$ & $0.753 \pm 0.087$ & $0.938 \pm 0.006$ & $0.989 \pm 0.001$ & $0.991 \pm 0.000$ & $0.994 \pm 0.000$ & $0.994 \pm 0.000$ & $0.995 \pm 0.000$ \\
\revred{ViT LocaLeJEPA} ($\lambda{=}0.3$, $K{=}2$) & $-0.032 \pm 0.211$ & $0.318 \pm 0.100$ & $0.586 \pm 0.069$ & $0.687 \pm 0.031$ & $0.916 \pm 0.009$ & $0.980 \pm 0.001$ & $0.985 \pm 0.000$ & $0.989 \pm 0.000$ & $0.991 \pm 0.000$ & $0.992 \pm 0.000$ \\
\revred{ViT LocaLeJEPA}-SIGReg ($\lambda{=}0.1$, $K{=}4$) & $-0.049 \pm 0.233$ & $0.303 \pm 0.165$ & $0.559 \pm 0.126$ & $0.751 \pm 0.028$ & $0.922 \pm 0.017$ & $0.983 \pm 0.001$ & $0.986 \pm 0.000$ & $0.990 \pm 0.000$ & $0.991 \pm 0.000$ & $0.992 \pm 0.000$ \\
\bottomrule
\end{tabular}}
\end{subtable}
\vspace{4pt}
\begin{subtable}{\textwidth}
\centering
\caption{Figurine 1. Test $R^2$ \revred{($\uparrow$)}}
\label{tab:yoda-lc-r2-1}
\resizebox{\textwidth}{!}{\begin{tabular}{lcccccccccc}
\toprule
Method & 5 & 10 & 20 & 50 & 100 & 500 & 800 & 2000 & 4000 & Full \\
\midrule
\revred{Spherical ViT} SUSReg & $-0.173 \pm 0.397$ & $0.045 \pm 0.234$ & $0.241 \pm 0.110$ & $0.505 \pm 0.057$ & $0.810 \pm 0.025$ & $0.952 \pm 0.002$ & $0.959 \pm 0.001$ & $0.969 \pm 0.001$ & $0.973 \pm 0.001$ & $0.974 \pm 0.001$ \\
\revred{nViT} SIGReg & $-0.238 \pm 0.261$ & $-0.093 \pm 0.238$ & $0.098 \pm 0.168$ & $0.357 \pm 0.037$ & $0.742 \pm 0.020$ & $0.938 \pm 0.003$ & $0.951 \pm 0.002$ & $0.962 \pm 0.002$ & $0.967 \pm 0.002$ & $0.968 \pm 0.002$ \\
\revred{nViT} SubReg & $-0.217 \pm 0.300$ & $0.011 \pm 0.094$ & $0.097 \pm 0.153$ & $0.323 \pm 0.112$ & $0.655 \pm 0.032$ & $0.929 \pm 0.004$ & $0.944 \pm 0.002$ & $0.958 \pm 0.002$ & $0.963 \pm 0.001$ & $0.964 \pm 0.001$ \\
\revred{ViT} SIGReg & $-0.083 \pm 0.233$ & $-0.044 \pm 0.336$ & $0.273 \pm 0.120$ & $0.560 \pm 0.061$ & $0.771 \pm 0.016$ & $0.955 \pm 0.002$ & $0.964 \pm 0.001$ & $0.973 \pm 0.001$ & $0.976 \pm 0.001$ & $0.978 \pm 0.001$ \\
\revred{ViT} SubReg & $-0.106 \pm 0.210$ & $\mathbf{0.166 \pm 0.146}$ & $\mathbf{0.424 \pm 0.132}$ & $0.638 \pm 0.040$ & $0.778 \pm 0.025$ & $0.954 \pm 0.002$ & $0.962 \pm 0.001$ & $0.971 \pm 0.001$ & $0.975 \pm 0.001$ & $0.976 \pm 0.001$ \\
\revred{ViT} RDMReg & $\mathbf{0.005 \pm 0.150}$ & $0.148 \pm 0.144$ & $0.291 \pm 0.089$ & $0.456 \pm 0.080$ & $0.786 \pm 0.040$ & $0.940 \pm 0.003$ & $0.951 \pm 0.002$ & $0.964 \pm 0.001$ & $0.970 \pm 0.001$ & $0.971 \pm 0.001$ \\
\midrule
\revred{ViT LocaLeJEPA} ($\lambda{=}0.1$, $K{=}2$) & $-0.093 \pm 0.215$ & $0.093 \pm 0.130$ & $0.380 \pm 0.139$ & $0.594 \pm 0.056$ & $0.821 \pm 0.019$ & $0.966 \pm 0.001$ & $0.973 \pm 0.001$ & $0.981 \pm 0.001$ & $0.983 \pm 0.001$ & $0.984 \pm 0.001$ \\
\revred{ViT LocaLeJEPA} ($\lambda{=}0.1$, $K{=}1{\to}8$) & $-0.094 \pm 0.247$ & $0.122 \pm 0.129$ & $0.388 \pm 0.096$ & $\mathbf{0.646 \pm 0.046}$ & $0.791 \pm 0.029$ & $\mathbf{0.968 \pm 0.002}$ & $\mathbf{0.974 \pm 0.001}$ & $\mathbf{0.981 \pm 0.000}$ & $\mathbf{0.983 \pm 0.000}$ & $\mathbf{0.984 \pm 0.000}$ \\
\revred{ViT LocaLeJEPA} ($\lambda{=}0.1$, $K{=}4$) & $-0.040 \pm 0.195$ & $0.149 \pm 0.143$ & $0.369 \pm 0.091$ & $0.564 \pm 0.052$ & $0.804 \pm 0.023$ & $0.966 \pm 0.002$ & $0.972 \pm 0.001$ & $0.979 \pm 0.000$ & $0.981 \pm 0.000$ & $0.982 \pm 0.000$ \\
\revred{ViT LocaLeJEPA} ($\lambda{=}0.3$, $K{=}2$) & $-0.014 \pm 0.124$ & $0.162 \pm 0.087$ & $0.285 \pm 0.065$ & $0.471 \pm 0.064$ & $\mathbf{0.830 \pm 0.016}$ & $0.961 \pm 0.002$ & $0.969 \pm 0.001$ & $0.977 \pm 0.001$ & $0.980 \pm 0.001$ & $0.981 \pm 0.001$ \\
\revred{ViT LocaLeJEPA}-SIGReg ($\lambda{=}0.1$, $K{=}4$) & $-0.070 \pm 0.191$ & $0.021 \pm 0.184$ & $0.248 \pm 0.136$ & $0.483 \pm 0.062$ & $0.821 \pm 0.020$ & $0.952 \pm 0.003$ & $0.962 \pm 0.001$ & $0.971 \pm 0.001$ & $0.976 \pm 0.001$ & $0.977 \pm 0.001$ \\
\bottomrule
\end{tabular}}
\end{subtable}
\end{table*}

\subsubsection{Spatial Structure of Rotation Recovery}
\label{app:yoda-torus}
We also investigate the visualization of the spatial structure of the rotation recovery experiment to expose where the errors occur over the underlying angular domain, a factor that the aggregate probe metrics do not reveal. For that purpose, we use
$(\theta_0,\theta_1)$ as coordinates on $\mathbb T^2$ and embed the resulting torus in $\mathbb R^3$ for visualization. These plots should be interpreted as visualizations of \emph{coordinate recovery}.
Their geometry is determined \rev{only} by the ground-truth or probe-decoded angles\rev{. They are not} a direct embedding of the encoder features themselves. They therefore show how faithfully the periodic factors are recovered across the angular domain, but do not by themselves establish the topology of the learned representation space. 

\begin{figure}[t]
\centering

\begin{tikzpicture}
\node[anchor=south west, inner sep=0] (img) at (0,0) {%
    \includegraphics[width=\textwidth]{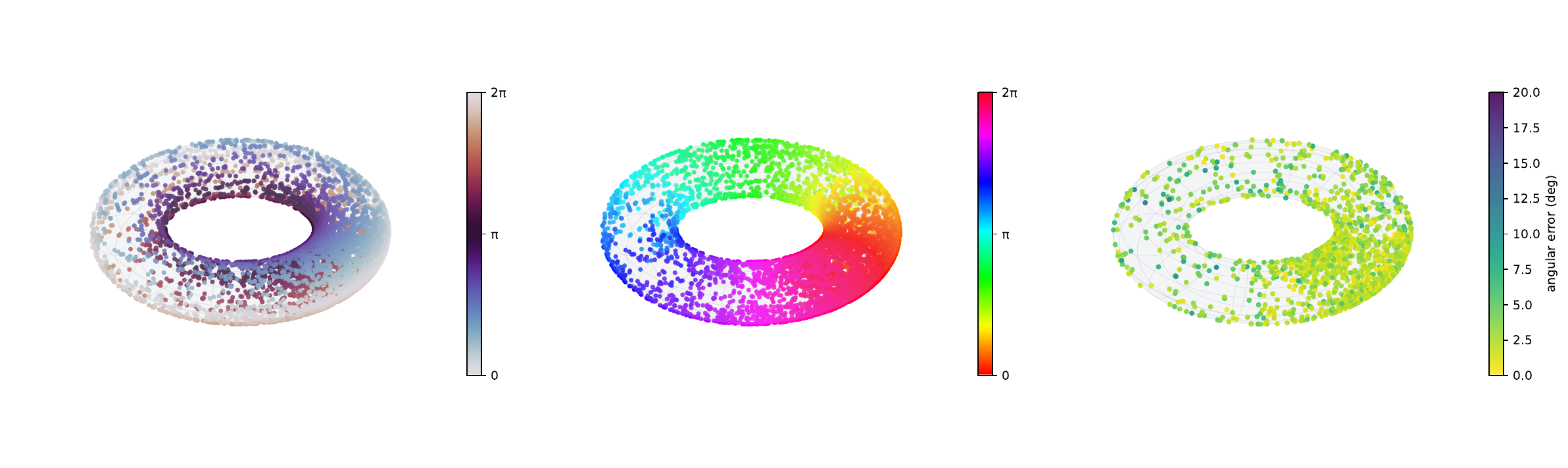}%
};

\begin{scope}[
    x={(img.south east)},
    y={(img.north west)}
]
    \node[
        anchor=north,
        font=\bfseries\small,
        fill=white,
        fill opacity=0.8,
        text opacity=1,
        inner sep=1pt
    ] at (0.155,0.175) {(a)};

    \node[
        anchor=north,
        font=\bfseries\small,
        fill=white,
        fill opacity=0.8,
        text opacity=1,
        inner sep=1pt
    ] at (0.485,0.175) {(b)};

    \node[
        anchor=north,
        font=\bfseries\small,
        fill=white,
        fill opacity=0.8,
        text opacity=1,
        inner sep=1pt
    ] at (0.815,0.175) {(c)};
\end{scope}
\end{tikzpicture}

\caption{\emph{Ground-truth and Recovered Angular Coordinates.}
(a,b) The 8{,}097 frames positioned on $\mathbb T^2$ using the ground-truth
rotation coordinates and colored by $\theta_0$ and $\theta_1$, respectively.
(c) Held-out frames positioned using angles decoded from frozen LocaLeJEPA
features ($\lambda{=}0.1$, $K{=}2$), with color indicating per-frame angular
error (mean $2.7^\circ$).}
\label{fig:yoda-torus}
\end{figure}

\textbf{Results.}
Figure~\ref{fig:yoda-torus} provides the basic construction. The first two panels
show the ground-truth torus colored by the two periodic coordinates, while the final
panel places held-out frames according to the angles decoded from fixed-$K$
LocaLeJEPA ($\lambda{=}0.1$, $K{=}2$) and colors them by their angular error. The
decoded coordinates cover the full angular domain, with most residual error remaining
small relative to the range of the two periodic variables. Figure~\ref{fig:yoda-torus-compare} extends this view across all evaluated encoders and probe-label budgets. Rows correspond to representations and columns to 10, 50, 100, 500, and all 5{,}182 probe-training frames. Each panel reports the MAE for the same held-out split and colors individual predictions using a shared angular
error scale. The progression across columns makes the effect of probe supervision visible directly: sparse-label probes produce more dispersed and higher-error reconstructions, whereas increasing the training budget progressively concentrates the predictions around the underlying toroidal coordinate structure. To compare methods spatially, Figure~\ref{fig:yoda-torus-adv} uses fixed-$K$
LocaLeJEPA as the reference and visualizes the per-frame difference:
\begin{equation}
    \Delta_{\mathrm{err}}
    =
    \operatorname{err}(\mathrm{competitor})
    -
    \operatorname{err}(\mathrm{LocaLeJEPA})
\end{equation}
Positive values therefore indicate frames on which fixed-$K$ LocaLeJEPA is more
accurate. Each panel \rev{also} reports the mean error difference and the fraction
of held-out frames won by LocaLeJEPA. At the 10-frame budget the comparison is mixed,
whereas from 50 frames onward the mean difference is positive against every displayed
competitor on this split. At the full budget, the mean advantage remains positive
throughout the grid and LocaLeJEPA is more accurate on $53$--$69\%$ of held-out
frames, depending on the competitor. All three visualizations use one fixed probe split, whereas
\Cref{app:yoda-lc} reports averages over ten shared splits.

\begin{figure}[p]
\centering
\includegraphics[
    width=\textwidth,
    height=0.82\textheight,
    keepaspectratio
]{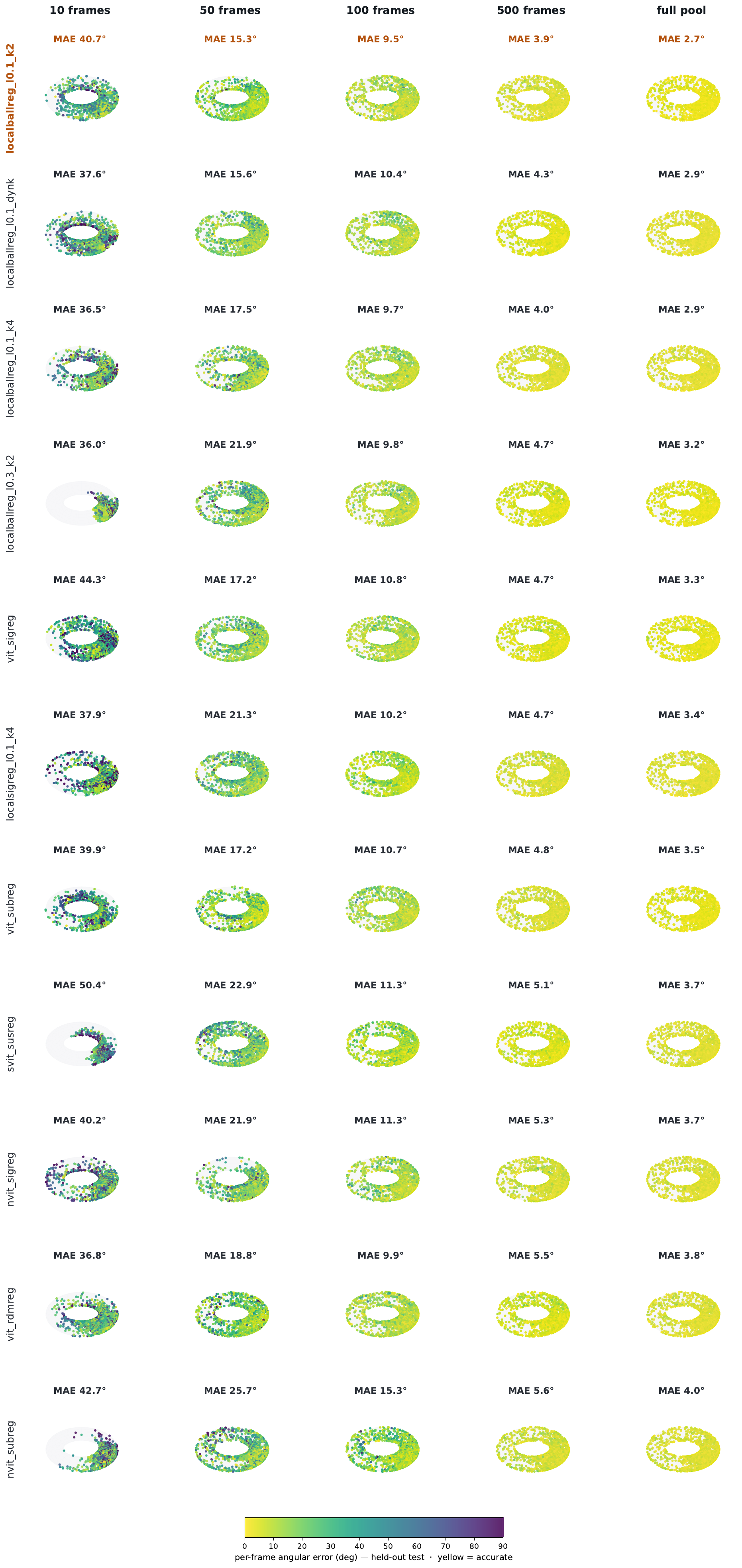}
\caption{\emph{Spatial Rotation Recovery across Encoders and Probe-label Budgets.}
Held-out frames are positioned according to their probe-decoded rotation coordinates
and colored by angular error using a common scale across all panels. Rows compare
encoders and columns vary the number of labeled frames used to fit the probe.
Visualizations correspond to one shared probe split. The aggregate results over ten splits are reported in \Cref{app:yoda-lc}.}
\label{fig:yoda-torus-compare}
\end{figure}

\begin{figure}[p]
\centering
\includegraphics[
    width=\textwidth,
    height=0.82\textheight,
    keepaspectratio
]{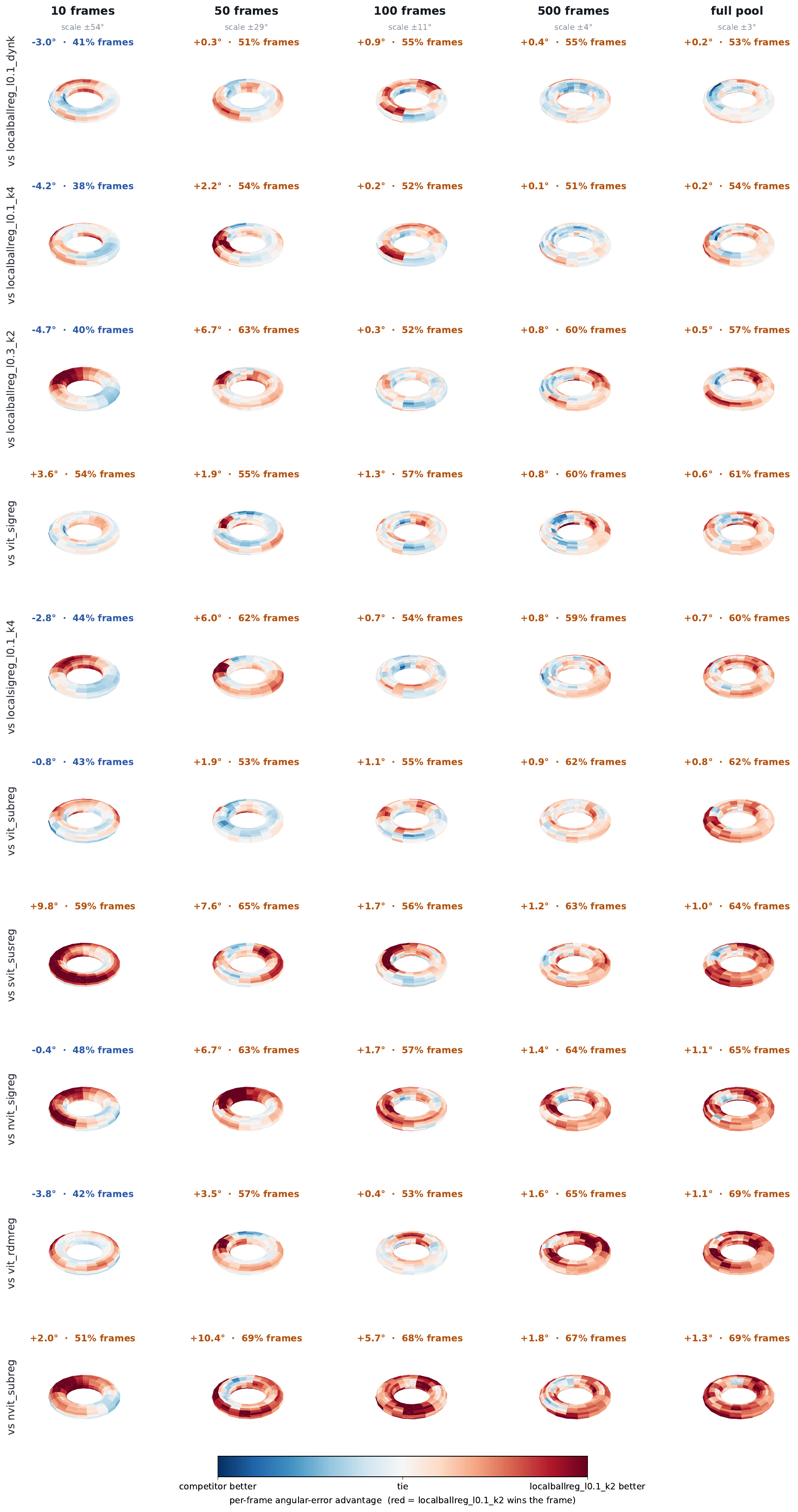}
\caption{\emph{Spatial Distribution of the LocaLeJEPA Recovery Advantage.}
Regional mean of
$\operatorname{err}(\mathrm{competitor})
-\operatorname{err}(\mathrm{LocaLeJEPA})$
relative to fixed-$K$ LocaLeJEPA ($\lambda{=}0.1$, $K{=}2$).
Positive values (red) favor LocaLeJEPA, while negative values (blue) favor the
competitor. Panel titles report the mean error difference and the fraction of
held-out frames on which LocaLeJEPA obtains lower error. \emph{Full} denotes the
complete 5{,}182-frame probe-training pool. All panels use the same fixed probe
split.}
\label{fig:yoda-torus-adv}
\end{figure}

\clearpage
\newpage
\subsection{ImageNet-1K: Training and Evaluation Protocol}
\label{app:imagenet-protocol}
This section specifies the experimental choices particular to the ImageNet-1K study underlying \Cref{subsec:imagenet-classification,subsec:imagenet-transfer-main}. We first describe the three encoder geometries and their optimization settings, and then detail the frozen-feature classification\rev{ and} downstream-transfer protocols. Shared training, checkpointing, and representation conventions follow \Cref{app:shared-protocol}, while the compared distribution-matching objectives are defined in \Cref{app:objectives}.

\subsubsection{ImageNet-1K: Pretraining Details}
\label{app:imagenet-pretrain}
Following the experimental setup of \citep{balestriero2025lejepa}, we pretrain all models for 100 epochs on the 1.28M-image ImageNet-1K \citep{deng2009imagenet} training split. 

\textbf{Encoder Families.} We consider three encoder families spanning distinct representation geometries. The \emph{Euclidean} encoder is a standard ViT-S/8 (21.7M parameters). The \emph{Hyperbolic} encoder uses the same ViT-S/8 backbone and maps its output to a Poincar\'e ball of curvature $-0.1$, with feature clipping at radius
$2.3$ \citep{ermolov2022hyperbolic}. Finally, the \emph{$D_4$-equivariant} encoder is a 24.4M-parameter ResNet-50 implemented with escnn \citep{klee2023comparison}, retaining the regular representation at its output \rev{without} pooling to an invariant feature. All three encoders feed the common Euclidean projector defined in \Cref{app:shared-protocol}. \rev{We also train a ViT-L/14 (303.2M parameters) at $238\times238$ pixels, with LeJEPA's published ViT-L settings for both methods.} 

\textbf{Details.} All runs use AdamW with cosine learning-rate decay and weight decay
$5\times10^{-2}$. Global SIGReg \rev{at LeJEPA's published settings} and RDMReg use one warmup epoch\rev{, a learning rate of $2\times10^{-3}$ for all three encoders, and decay to half the peak rate}. LocaLeJEPA\rev{, its local control and LeJEPA (our settings)} use a learning rate of $10^{-3}$, ten warmup epochs, \rev{decay to $10^{-3}$ of the peak rate,} and gradient clipping at $1.0$. SIGReg uses $\lambda{=}0.02$\rev{ at LeJEPA's published settings and $\lambda{=}0.1$ in LeJEPA (our settings)}, while the LocaLeJEPA values of $\lambda$ and $K$ are reported with the corresponding configurations.
RDMReg uses the objective weighting defined in \Cref{app:objectives}. For the dynamic-neighborhood variant, we instantiate the schedule of \Cref{app:shared-protocol} with $K_{\max}=10$, and $E_{\mathrm{step}}=10$, so that the partition evolves from $K{=}1$ to $K{=}10$ over the 100-epoch training schedule, reaching its maximum granularity at epoch 90. Following the shared checkpointing protocol described in \Cref{app:shared-protocol}, every downstream evaluation uses the final epoch-100 checkpoint. \rev{For ViT-L/14, both methods use LeJEPA's published ViT-L settings, with a learning rate of $5\times10^{-4}$, ten warmup epochs, decay to $10^{-3}$ of the peak rate, batches of 512 images, $\lambda{=}0.05$\revred{,} 1024 slicing directions\revred{ and an integration range of $[-5,5]$}.}

\subsubsection{ImageNet-1K Frozen-Feature Readouts}
\label{app:imagenet-readouts}

Following the in-domain evaluation protocol of
SPHERE-JEPA \citep{nicollier2026spherejepa}, the ImageNet-1K training split
forms the feature bank and the validation split is used for evaluation.
Images are resized and center-cropped to $128\times128$ \rev{($238\times238$ for ViT-L/14)} and normalized using
ImageNet statistics.

\textbf{$k$-NN and Linear Probes.}
For what concerns \emph{$k$-NN probe}, the frozen representations are $\ell_2$-normalized and compared using cosine similarity. Class votes are weighted according to $\exp\!\left(\operatorname{sim}/0.07\right)$. We evaluate $k\in\{10,20,100,200\}$ and use $k{=}20$ in the main paper. \rev{Across the tested configurations, this value is within $0.81$ percentage points of the best setting, and the ordering of methods is the same at every $k$ on ViT-S/8 and ViT-L/14. On the other two encoders, three pairs of runs that are within one point of each other swap places.} Instead, in the \emph{linear probe}, frozen representations are standardized using statistics computed from the ImageNet-1K training split. We fit a single affine classifier for 50 epochs using SGD with momentum $0.9$, zero weight decay, batch size $4096$, and cosine learning-rate decay. The learning rate is selected from $\{0.01,0.03,0.1,0.3,1,3,10\}$, 
using a 2\% holdout from the training split. After selection, the classifier is refitted on the complete training split and evaluated once on the ImageNet-1K validation set. Both readouts operate on the same frozen epoch-100 representations. \rev{ViT-S/8 results for LeJEPA, LeJEPA (our settings) and LocaLeJEPA with $K{=}1$ (at both weights) and $K{=}4$ are means over three pretraining seeds. All other ImageNet-1K results are single pretraining runs, and all results use one linear-probe seed.}

\begin{table}[t]
\revstart
\centering
\caption{\emph{ImageNet-1K Frozen-feature Classification across All Configurations.} Top-1 $k$-NN and linear-probe accuracy (\%) on the ImageNet-1K validation set using the final-epoch checkpoints and one linear-probe seed. ViT-S/8 rows other than RDMReg are mean $\pm$ SD over three pretraining seeds, and all other rows are single runs. \textbf{Bold} denotes the best mean within each encoder family.}
\label{tab:in1k-indomain-full}
\small
\setlength{\tabcolsep}{4pt}
\resizebox{\linewidth}{!}{\begin{tabular}{l cccc cc}
\toprule
 & \multicolumn{4}{c}{$k$-NN top-1 (\%)} & \multicolumn{2}{c}{Linear (\%)} \\
\cmidrule(lr){2-5}\cmidrule(lr){6-7}
Method & $k{=}10$ & $k{=}20$ & $k{=}100$ & $k{=}200$ & top-1 & top-5 \\
\midrule
\multicolumn{7}{l}{\emph{Plain ViT-S/8 (Euclidean, 21.7M)}} \\ 
LeJEPA & $32.1 \pm 0.36$ & $32.8 \pm 0.33$ & $31.3 \pm 0.41$ & $29.7 \pm 0.35$ & $44.6 \pm 0.23$ & $70.8 \pm 0.17$ \\ 
LeJEPA (our settings) & $38.6 \pm 0.15$ & $38.9 \pm 0.09$ & $36.8 \pm 0.16$ & $35.3 \pm 0.15$ & $51.6 \pm 0.28$ & $77.3 \pm 0.36$ \\ 
Rectified LpJEPA (RDMReg) & $11.6$ & $12.6$ & $13.4$ & $13.1$ & $25.7$ & $48.5$ \\ 
LocaLeJEPA ($\lambda{=}0.1$, $K{=}1$) & $45.4 \pm 0.23$ & $45.4 \pm 0.22$ & $42.9 \pm 0.19$ & $41.2 \pm 0.18$ & $59.7 \pm 0.22$ & $83.5 \pm 0.09$ \\ 
LocaLeJEPA ($\lambda{=}0.025$, $K{=}1$) & $47.8 \pm 0.22$ & $47.7 \pm 0.26$ & $45.0 \pm 0.12$ & $43.1 \pm 0.17$ & $59.8 \pm 0.29$ & $83.6 \pm 0.24$ \\ 
LocaLeJEPA ($\lambda{=}0.1$, $K{=}4$) & $\mathbf{50.1 \pm 0.22}$ & $\mathbf{50.1 \pm 0.12}$ & $\mathbf{47.5 \pm 0.24}$ & $\mathbf{45.7 \pm 0.23}$ & $\mathbf{62.1 \pm 0.27}$ & $\mathbf{85.1 \pm 0.16}$ \\ 
\midrule
\multicolumn{7}{l}{\emph{Hyperbolic ViT-S/8 (Poincar\'e ball, 21.7M)}} \\ 
LeJEPA & $32.1$ & $32.6$ & $31.3$ & $29.9$ & $44.0$ & $69.8$ \\ 
LeJEPA (our settings) & $39.1$ & $39.4$ & $37.3$ & $35.7$ & $51.6$ & $77.5$ \\ 
Rectified LpJEPA (RDMReg) & $12.1$ & $13.1$ & $13.8$ & $13.6$ & $26.0$ & $48.7$ \\ 
LocaLeJEPA ($\lambda{=}0.1$, $K{=}1$) & $45.3$ & $45.7$ & $43.0$ & $41.3$ & $59.5$ & $83.4$ \\ 
LocaLeJEPA ($\lambda{=}0.025$, $K{=}1$) & $47.8$ & $48.0$ & $45.0$ & $43.0$ & $59.9$ & $83.3$ \\ 
LocaLeJEPA ($\lambda{=}0.1$, $K{=}2$) & $48.6$ & $48.5$ & $46.0$ & $44.2$ & $61.4$ & $84.6$ \\ 
LocaLeJEPA ($\lambda{=}0.1$, $K{=}4$) & $50.5$ & $50.6$ & $47.9$ & $46.1$ & $61.9$ & $84.9$ \\ 
LocaLeJEPA ($\lambda{=}0.1$, $K{=}1{\to}10$) & $\mathbf{51.0}$ & $\mathbf{50.9}$ & $\mathbf{48.4}$ & $\mathbf{46.6}$ & $\mathbf{62.2}$ & $\mathbf{85.1}$ \\ 
LocaLeJEPA ($\lambda{=}0.3$, $K{=}2$) & $38.7$ & $39.0$ & $37.5$ & $36.0$ & $51.7$ & $77.6$ \\ 
LocaLeJEPA-SIGReg ($\lambda{=}0.1$, $K{=}4$) & $35.1$ & $35.4$ & $33.7$ & $32.2$ & $48.2$ & $73.8$ \\ 
\midrule
\multicolumn{7}{l}{\emph{$D_4$-equivariant ResNet-50 (24.4M)}} \\ 
LeJEPA & $33.9$ & $34.5$ & $33.2$ & $31.7$ & $43.4$ & $70.7$ \\ 
LeJEPA (our settings) & $39.5$ & $39.7$ & $38.3$ & $36.9$ & $48.8$ & $75.7$ \\ 
Rectified LpJEPA (RDMReg) & $22.7$ & $23.4$ & $22.6$ & $21.8$ & $31.3$ & $56.9$ \\ 
LocaLeJEPA ($\lambda{=}0.1$, $K{=}1$) & $43.7$ & $43.9$ & $42.6$ & $41.2$ & $52.4$ & $79.0$ \\ 
LocaLeJEPA ($\lambda{=}0.025$, $K{=}1$) & $44.0$ & $44.5$ & $42.6$ & $40.9$ & $52.4$ & $78.7$ \\ 
LocaLeJEPA ($\lambda{=}0.1$, $K{=}2$) & $45.8$ & $46.1$ & $44.6$ & $43.2$ & $54.3$ & $79.9$ \\ 
LocaLeJEPA ($\lambda{=}0.1$, $K{=}4$) & $\mathbf{47.0}$ & $\mathbf{47.5}$ & $45.8$ & $\mathbf{44.6}$ & $\mathbf{55.2}$ & $80.6$ \\ 
LocaLeJEPA ($\lambda{=}0.1$, $K{=}1{\to}10$) & $\mathbf{47.0}$ & $47.4$ & $\mathbf{45.9}$ & $44.3$ & $55.1$ & $\mathbf{80.8}$ \\ 
LocaLeJEPA ($\lambda{=}0.3$, $K{=}2$) & $40.0$ & $40.4$ & $39.1$ & $37.7$ & $48.7$ & $75.8$ \\ 
LocaLeJEPA-SIGReg ($\lambda{=}0.1$, $K{=}4$) & $34.2$ & $34.8$ & $33.6$ & $32.4$ & $43.2$ & $69.5$ \\ 
\midrule
\multicolumn{7}{l}{\emph{Plain ViT-L/14 (Euclidean, 303.2M, LeJEPA's ViT-L settings)}} \\ 
LeJEPA & $45.8$ & $46.0$ & $44.0$ & $42.3$ & $54.7$ & $79.9$ \\ 
LocaLeJEPA ($\lambda{=}0.05$, $K{=}1$) & $46.3$ & $46.5$ & $44.6$ & $43.0$ & $54.8$ & $80.1$ \\ 
LocaLeJEPA ($\lambda{=}0.05$, $K{=}4$) & $\mathbf{47.2}$ & $\mathbf{47.4}$ & $\mathbf{45.5}$ & $\mathbf{43.8}$ & $\mathbf{61.0}$ & $\mathbf{84.2}$ \\ 
\bottomrule
\end{tabular}}
\end{table}

\textbf{Downstream Transfer.}
\label{app:imagenet-transfer}
We evaluate the same frozen ImageNet-pretrained encoders on seven downstream
datasets spanning textures, fine-grained recognition, and generic object
classification: DTD \citep{cimpoi14describing}, FGVC-Aircraft
\citep{maji13fine-grained}, CIFAR-10 and CIFAR-100
\citep{krizhevsky2009learning}, Flowers-102
\citep{nilsback2008automated}, Food-101 \citep{bossard14}, and
Oxford-IIIT Pets \citep{parkhi12a}. We use DTD partition~1 train+val (3{,}760 images, 47 classes), Aircraft
trainval (6{,}667/100), CIFAR-10 (50{,}000/10), CIFAR-100
(50{,}000/100), Flowers-102 train+val (2{,}040/102), Food-101
(75{,}750/101), and Pets trainval (3{,}680/37). The evaluation uses the
corresponding official test split.  Images are resized and center-cropped to $128\times128$\revred{ ($238\times238$ for ViT-L/14)}, with no augmentation
applied during probe training. Frozen representations are standardized and a
single linear classifier is optimized for 100 epochs using SGD with momentum
$0.9$, zero weight decay, batch size $256$, and cosine learning-rate decay.
The learning rate is selected from
$\{10^{-3},3\times10^{-3},\ldots,30\}$ using a 10\% holdout from the
training partition, after which the classifier is refitted on the complete
training split. We report top-1 test accuracy\rev{, and} \emph{Avg} denotes the
unweighted mean across the seven datasets.




\subsubsection{ImageNet-1K: Full Results}
\label{app:imagenet-full}
\Cref{tab:in1k-indomain-full,tab:in1k-transfer-full} expand the comparison to the complete set of ImageNet-1K configurations detailed in \revred{\Cref{subsec:imagenet-classification}}, including a subset of parameters over regularization strength and neighborhood granularity.

\textbf{Regularization Strength.}
\rev{At $K{=}2$, $\lambda{=}0.1$ gives stronger ImageNet classification than $\lambda{=}0.3$ on both non-Euclidean encoder families, by $9.6$ / $9.6$ points ($k$-NN / linear) on the hyperbolic and $5.8$ / $5.5$ on the equivariant encoder. Its transfer average is also higher, by $7.2$ and $3.5$ points. At $K{=}2$, $\lambda{=}0.3$ has an effective strength of $9.6$, above the range that trains well on Imagenette (\Cref{sec:effective-strength}). The weight of the local term needs tuning together with $K$.}

\textbf{Neighborhood granularity.}
\rev{At $\lambda{=}0.1$, the dynamic schedule stays within $0.4$ points of fixed $K{=}4$ on both readouts. Fixed $K{=}2$ trails $K{=}4$ by $2.0$ and $1.4$ points in $k$-NN on the hyperbolic and equivariant encoders, with linear accuracy within $1.0$ point. Transfer averages follow the same pattern: the dynamic schedule is $0.5$ points above $K{=}4$ on the hyperbolic and $0.5$ points below it on the equivariant encoder, and $K{=}2$ is up to $1.1$ points below $K{=}4$. The step from one neighborhood to four is analyzed in \Cref{sec:experiments}.}

{\revstart\textbf{Locality across encoder depth.} To test whether the benefit of locality depends on encoder capacity, we also train ViT-S/8 with $6$ and $3$ transformer blocks instead of $12$ (\Cref{tab:depth}, pretraining seed $0$). At matched effective strength ($K{=}4$ at $\lambda{=}0.1$ against $K{=}1$ at $\lambda{=}0.025$), locality adds $2.2$ / $2.1$, $1.7$ / $1.9$ and $1.2$ / $0.6$ points ($k$-NN / linear) at depths $12$, $6$ and $3$. It is positive at every depth and smallest for the shallowest encoder. The difference between depths $12$ and $6$ is within the seed-to-seed variation of this gain, while the difference between depths $12$ and $3$ is not.\par}
\begin{table}[t]
\revstart
\centering
\caption{\emph{Locality across encoder depth (ImageNet-1K, ViT-S/8, 100 epochs).} $k$-NN ($k{=}20$) / linear top-1 (\%) on the ImageNet-1K validation set, final-epoch checkpoints, pretraining seed 0. $K{=}1$ at $\lambda{=}0.025$ matches the effective regularization strength of $K{=}4$ at $\lambda{=}0.1$.}
\label{tab:depth}
\small
\begin{tabular}{l ccc}
\toprule
Depth (blocks) & $K{=}1$, $\lambda{=}0.1$ & $K{=}1$, $\lambda{=}0.025$ & $K{=}4$, $\lambda{=}0.1$ \\
\midrule
12 & $45.7$ / $59.9$ & $47.9$ / $60.0$ & $50.1$ / $62.1$ \\ 
6 & $41.6$ / $52.9$ & $42.7$ / $53.1$ & $44.4$ / $55.0$ \\ 
3 & $32.0$ / $40.0$ & $32.0$ / $40.2$ & $33.2$ / $40.8$ \\ 
\bottomrule
\end{tabular}
\end{table}

{\revstart\textbf{ViT-L/14 at matched strength.} On ViT-L/14, $K{=}1$ at a quarter of the weight ($\lambda{=}0.0125$, matching $K{=}4$ at $\lambda{=}0.05$) does not train stably. Its online probe is at $0.27\%$ after six epochs, against $8.62$--$15.56\%$ for the other three ViT-L/14 runs, and it ends at $25.06$ / $35.14$ ($k$-NN / linear). We therefore report locality at matched strength for the three $128$-pixel encoders only. On ViT-L/14, $K{=}4$ adds $0.8$ / $6.2$ points over $K{=}1$ at the same weight.\par}

\begin{table}[t]
\revstart
\centering
\caption{\emph{Transfer Linear Probing across All Configurations.} Top-1 accuracy (\%) on the official test splits of seven downstream datasets using frozen final-epoch ImageNet-1K representations and one linear-probe seed. ViT-S/8 rows other than RDMReg are mean $\pm$ SD over three pretraining seeds, and all other rows are single runs. \textbf{Bold} marks the best mean within each encoder family and column.}
\label{tab:in1k-transfer-full}
\small
\setlength{\tabcolsep}{3pt}
\resizebox{\linewidth}{!}{\begin{tabular}{l ccccccc c}
\toprule
Method & DTD & Aircraft & CIFAR-10 & CIFAR-100 & Flowers & Food & Pets & Avg \\
\midrule
\multicolumn{9}{l}{\emph{Plain ViT-S/8 (Euclidean, 21.7M)}} \\ 
LeJEPA & $60.0 \pm 0.61$ & $28.4 \pm 0.47$ & $86.1 \pm 0.47$ & $62.6 \pm 1.33$ & $75.4 \pm 0.38$ & $51.0 \pm 0.82$ & $59.8 \pm 1.09$ & $60.5 \pm 0.52$ \\ 
LeJEPA (our settings) & $64.1 \pm 0.64$ & $30.3 \pm 0.87$ & $88.7 \pm 0.18$ & $66.0 \pm 0.19$ & $78.9 \pm 0.22$ & $56.4 \pm 0.29$ & $65.9 \pm 0.52$ & $64.3 \pm 0.13$ \\ 
Rectified LpJEPA (RDMReg) & $44.2$ & $18.4$ & $74.7$ & $49.1$ & $52.4$ & $38.8$ & $37.4$ & $45.0$ \\ 
LocaLeJEPA ($\lambda{=}0.1$, $K{=}1$) & $69.3 \pm 0.32$ & $38.4 \pm 0.43$ & $91.7 \pm 0.32$ & $71.5 \pm 0.87$ & $83.7 \pm 0.20$ & $63.5 \pm 0.08$ & $73.6 \pm 0.81$ & $70.3 \pm 0.04$ \\ 
LocaLeJEPA ($\lambda{=}0.025$, $K{=}1$) & $\mathbf{70.0 \pm 0.48}$ & $40.2 \pm 1.21$ & $91.9 \pm 0.22$ & $72.5 \pm 0.46$ & $85.2 \pm 0.18$ & $64.1 \pm 0.11$ & $73.4 \pm 0.75$ & $71.1 \pm 0.28$ \\ 
LocaLeJEPA ($\lambda{=}0.1$, $K{=}4$) & $69.4 \pm 1.47$ & $\mathbf{42.1 \pm 0.61}$ & $\mathbf{92.2 \pm 0.34}$ & $\mathbf{73.6 \pm 0.37}$ & $\mathbf{85.3 \pm 0.34}$ & $\mathbf{66.0 \pm 0.13}$ & $\mathbf{77.3 \pm 1.79}$ & $\mathbf{72.3 \pm 0.29}$ \\ 
\midrule
\multicolumn{9}{l}{\emph{Hyperbolic ViT-S/8 (Poincar\'e ball, 21.7M)}} \\ 
LeJEPA & $61.5$ & $28.2$ & $85.9$ & $62.6$ & $75.2$ & $50.7$ & $58.0$ & $60.3$ \\ 
LeJEPA (our settings) & $60.8$ & $31.4$ & $88.6$ & $66.1$ & $78.7$ & $56.7$ & $66.7$ & $64.1$ \\ 
Rectified LpJEPA (RDMReg) & $45.0$ & $15.9$ & $73.6$ & $49.1$ & $53.1$ & $38.9$ & $36.6$ & $44.6$ \\ 
LocaLeJEPA ($\lambda{=}0.1$, $K{=}1$) & $68.8$ & $37.5$ & $92.0$ & $71.4$ & $83.7$ & $63.8$ & $73.3$ & $70.1$ \\ 
LocaLeJEPA ($\lambda{=}0.025$, $K{=}1$) & $69.5$ & $39.6$ & $92.4$ & $72.7$ & $84.5$ & $64.0$ & $73.4$ & $70.9$ \\ 
LocaLeJEPA ($\lambda{=}0.1$, $K{=}2$) & $69.3$ & $41.5$ & $\mathbf{92.7}$ & $\mathbf{73.4}$ & $\mathbf{85.6}$ & $65.0$ & $77.1$ & $72.1$ \\ 
LocaLeJEPA ($\lambda{=}0.1$, $K{=}4$) & $69.1$ & $41.5$ & $92.3$ & $72.7$ & $85.3$ & $65.3$ & $\mathbf{78.6}$ & $72.1$ \\ 
LocaLeJEPA ($\lambda{=}0.1$, $K{=}1{\to}10$) & $\mathbf{70.1}$ & $\mathbf{42.5}$ & $92.6$ & $73.3$ & $\mathbf{85.6}$ & $\mathbf{65.9}$ & $78.5$ & $\mathbf{72.6}$ \\ 
LocaLeJEPA ($\lambda{=}0.3$, $K{=}2$) & $64.3$ & $30.5$ & $88.9$ & $66.3$ & $79.5$ & $57.9$ & $66.4$ & $64.8$ \\ 
LocaLeJEPA-SIGReg ($\lambda{=}0.1$, $K{=}4$) & $57.3$ & $30.8$ & $87.9$ & $65.1$ & $77.8$ & $55.9$ & $65.0$ & $62.8$ \\ 
\midrule
\multicolumn{9}{l}{\emph{$D_4$-equivariant ResNet-50 (24.4M)}} \\ 
LeJEPA & $55.4$ & $23.7$ & $86.2$ & $60.6$ & $63.9$ & $40.6$ & $52.2$ & $54.7$ \\ 
LeJEPA (our settings) & $59.5$ & $25.7$ & $88.9$ & $63.9$ & $67.4$ & $45.5$ & $61.1$ & $58.8$ \\ 
Rectified LpJEPA (RDMReg) & $48.0$ & $23.0$ & $80.5$ & $53.0$ & $58.0$ & $31.6$ & $42.4$ & $48.1$ \\ 
LocaLeJEPA ($\lambda{=}0.1$, $K{=}1$) & $60.2$ & $25.7$ & $90.2$ & $65.9$ & $67.8$ & $47.5$ & $65.0$ & $60.3$ \\ 
LocaLeJEPA ($\lambda{=}0.025$, $K{=}1$) & $61.3$ & $25.3$ & $90.5$ & $66.2$ & $66.9$ & $47.1$ & $66.3$ & $60.5$ \\ 
LocaLeJEPA ($\lambda{=}0.1$, $K{=}2$) & $\mathbf{61.8}$ & $28.4$ & $90.2$ & $66.5$ & $69.7$ & $49.0$ & $67.9$ & $61.9$ \\ 
LocaLeJEPA ($\lambda{=}0.1$, $K{=}4$) & $61.3$ & $\mathbf{29.3}$ & $\mathbf{91.2}$ & $\mathbf{69.1}$ & $\mathbf{71.5}$ & $\mathbf{49.5}$ & $69.1$ & $\mathbf{63.0}$ \\ 
LocaLeJEPA ($\lambda{=}0.1$, $K{=}1{\to}10$) & $61.5$ & $27.4$ & $90.9$ & $67.7$ & $70.4$ & $49.4$ & $\mathbf{69.9}$ & $62.5$ \\ 
LocaLeJEPA ($\lambda{=}0.3$, $K{=}2$) & $55.9$ & $25.4$ & $88.3$ & $62.9$ & $66.5$ & $46.6$ & $63.7$ & $58.5$ \\ 
LocaLeJEPA-SIGReg ($\lambda{=}0.1$, $K{=}4$) & $48.4$ & $25.5$ & $88.8$ & $60.6$ & $52.5$ & $25.9$ & $65.3$ & $52.4$ \\ 
\midrule
\multicolumn{9}{l}{\emph{Plain ViT-L/14 (Euclidean, 303.2M, LeJEPA's ViT-L settings)}} \\ 
LeJEPA & $62.3$ & $20.1$ & $90.5$ & $66.2$ & $68.2$ & $57.2$ & $64.8$ & $61.3$ \\ 
LocaLeJEPA ($\lambda{=}0.05$, $K{=}1$) & $61.3$ & $17.3$ & $90.4$ & $65.9$ & $66.8$ & $56.3$ & $62.6$ & $60.1$ \\ 
LocaLeJEPA ($\lambda{=}0.05$, $K{=}4$) & $\mathbf{66.0}$ & $\mathbf{26.1}$ & $\mathbf{92.3}$ & $\mathbf{70.6}$ & $\mathbf{74.0}$ & $\mathbf{62.8}$ & $\mathbf{66.4}$ & $\mathbf{65.5}$ \\ 
\bottomrule
\end{tabular}}
\end{table}

\textbf{Locality versus reference distribution.}
LocaLeJEPA-SIGReg preserves the same local partitioning as LocaLeJEPA while
replacing the centered uniform-ball target with Gaussian matching on uncentered
clusters. \rev{Trained with the same settings, it is $15.2$ / $13.7$ points ($k$-NN / linear) below LocaLeJEPA on the hyperbolic and $12.7$ / $12.0$ on the equivariant encoder, and its transfer average is $9.3$ and $10.6$ points lower. It also falls below LeJEPA (our settings). Local partitioning alone does not explain the gain}, although,
as discussed in \revred{\Cref{app:objectives}}, this control does not separately
identify the effects of centering and reference distribution.

\textbf{RDMReg.}
RDMReg trails global SIGReg on the ImageNet classification readouts across the evaluated encoder families\rev{, by $12.0$--$18.9$ points in linear top-1 accuracy and $11.0$--$20.2$ in $k$-NN, with the same optimization settings for both}. This provides an additional distribution-matching
baseline with a \rev{very} different reference geometry and shows that
replacing the Gaussian target alone is not sufficient to reproduce the
LocaLeJEPA improvements.

\section{Qualitative Embedding Analysis}
\label{sec:qualitative-emb-analysis}
In this section, we present a series of qualitative analyses and comparisons between the embeddings attained by the different JEPA recipes. In \Cref{app:coil_experiment}, we complement the results presented in \Cref{fig:coil-grid-main}, showing the recovery of topology in latent space. In \Cref{app:pca}, we provide extended details about \Cref{fig:pca-teaser}, which shows that embeddings \rev{capture semantic structure} across geometric non-Euclidean embeddings. Lastly, in \Cref{app:knn-neighborhood} we perform a class consistency analysis across \rev{neighborhoods} coming from different regularizations.

\subsection{Topology Preservation on COIL Dataset}
\label{app:coil_experiment}
\paragraph{Data and training.}
COIL-100 \citep{nene1996coil} contains 7200 images of 100 objects, each \revred{taken} on a turntable at 72 poses $5^\circ$ apart. For a fixed object the frames sample a closed curve $\theta \mapsto x_\theta$, $\theta \in S^1$, in image space, and we ask whether an encoder $f$ maps this curve to a closed curve in embedding space or collapses it. We train ViT-S/8 encoders with LocaLeReg and SIGReg for 100 epochs with appearance-only augmentations (\revred{color} jitter, grayscale, Gaussian blur) and \revred{analyze} the 512-d encoder output.

\paragraph{Persistent homology.}
To measure whether the 72 embeddings $\{f(x_\theta)\}$ of an object lie on a loop we use persistent homology \citep{edelsbrunner2002topological,zomorodian2005computing}. For a scale $r \ge 0$, the Vietoris--Rips complex $\mathrm{VR}_r$ contains a simplex for every set of points whose pairwise distances are at most $r$. These complexes are nested as $r$ grows, and each topological feature appears at some scale $b$ (its birth) and disappears at a scale $d$ (its death): in dimension 0 the features are connected components, in dimension 1 they are loops. The persistence diagram plots every feature as the point $(b, d)$, and its lifetime $d - b$ is the range of scales over which it exists. For a sampled closed curve, the loop is born once consecutive samples connect, and dies once the hole fills in\rev{. A} long lifetime therefore indicates a \rev{real} hole, while features close to the diagonal are likely due to noise\rev{. This} is made precise by the stability of persistence diagrams \citep{chazal2009gromov}. We compute diagrams with Ripser \citep{bauer2021ripser} on the full 512-d embeddings after dividing all distances by the median pairwise distance among the 72 points of the object, so that lifetimes are comparable across representations of different scale and dimension.

\paragraph{Results.}
In pixel space the median longest $H_1$ lifetime over the 100 objects is 0.387, so the loop is present in the input. But encoders shorten it: under either regularizer the median is at most 0.196 for every seed and checkpoint. Over 3 training seeds, the median lifetime at the final weights is higher under LocaLeReg for every seed ($0.190 \pm 0.005$ against $0.175 \pm 0.002$, mean and standard deviation over seeds)\rev{. At} the checkpoint with the best accuracy of the online linear probe, it is higher for two of the three ($0.188 \pm 0.007$ against $0.181 \pm 0.012$). \Cref{fig:coil-grid} shows examples for several objects: for each method, the persistence diagram of the object's 72 embeddings and their projection on the first 2 principal components, colored by rotation angle. The same runs also measure how well each representation separates the 100 objects (\Cref{tab:coil-identity}): the probe accuracy on held-out poses (every 6-th angle of every object) and the nearest-neighbor purity, the fraction of frames whose nearest neighbor in embedding space belongs to the same object, computed on 24 frames per object. Both are higher under LocaLeReg than under SIGReg for each of the 3 seeds, at the final weights and at the best checkpoint.

\begin{table}[htbp]
  \centering
  \caption{Object identity on COIL-100, mean $\pm$ standard deviation over 3 seeds. The best checkpoint is the one with the highest probe accuracy.}
  \label{tab:coil-identity}
  \begin{tabular}{lcccc}
    \toprule
    & \multicolumn{2}{c}{Probe accuracy} & \multicolumn{2}{c}{NN purity} \\
    \cmidrule(lr){2-3} \cmidrule(lr){4-5}
    & final & best & final & best \\
    \midrule
    LocaLeReg & $0.384 \pm 0.047$ & $0.436 \pm 0.026$ & $0.492 \pm 0.004$ & $0.500 \pm 0.006$ \\
    SIGReg    & $0.264 \pm 0.016$ & $0.360 \pm 0.022$ & $0.431 \pm 0.023$ & $0.469 \pm 0.018$ \\
    \bottomrule
  \end{tabular}
\end{table}

\begin{figure}[htbp]
  \centering
  \includegraphics[width=\textwidth]{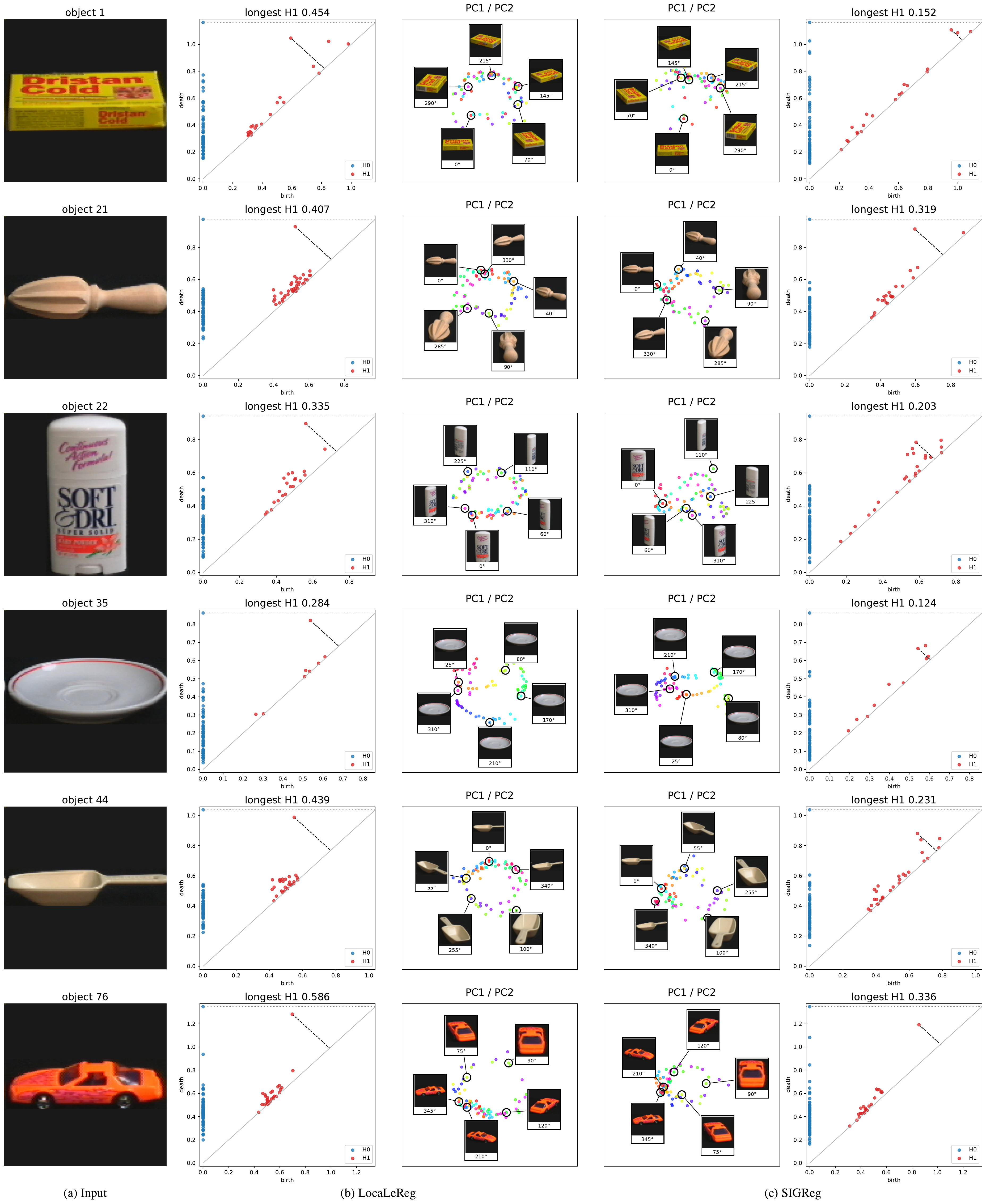}
  \caption{Per-object geometry of COIL-100 embeddings (ViT-S/8, 100 epochs, $\lambda=0.1$\rev{,} LocaLeReg with \revred{$K=8$} clusters).
Each row is one object: (a) its frame at $0^\circ$\rev{.} (b) LocaLeReg and (c) SIGReg, each with the persistence diagram and the
2-d PCA projection of the embeddings of the object's 72 frames.
\emph{Persistence}: Vietoris--Rips diagram of the 72 embeddings, rescaled to unit median pairwise distance\rev{. The} dashed segment
marks the most persistent $H_1$ class. \emph{PC1/PC2}: points are colored by rotation angle.}
  \label{fig:coil-grid}
\end{figure}


\subsection{Patch-Level Representation Structure}
\label{app:pca}

The quantitative frozen-feature evaluations characterize global representation
quality but do not reveal how spatial information is organized before pooling.
We therefore complement them with a qualitative PCA visualization of
patch-level features, following the general visualization strategy used in
LeJEPA \revred{\citep{balestriero2025lejepa}} and DINOv2
\revred{\citep{oquab2023dinov2}}.

\begin{figure}[htb]
    \centering
    \includegraphics[
        width=\linewidth
    ]{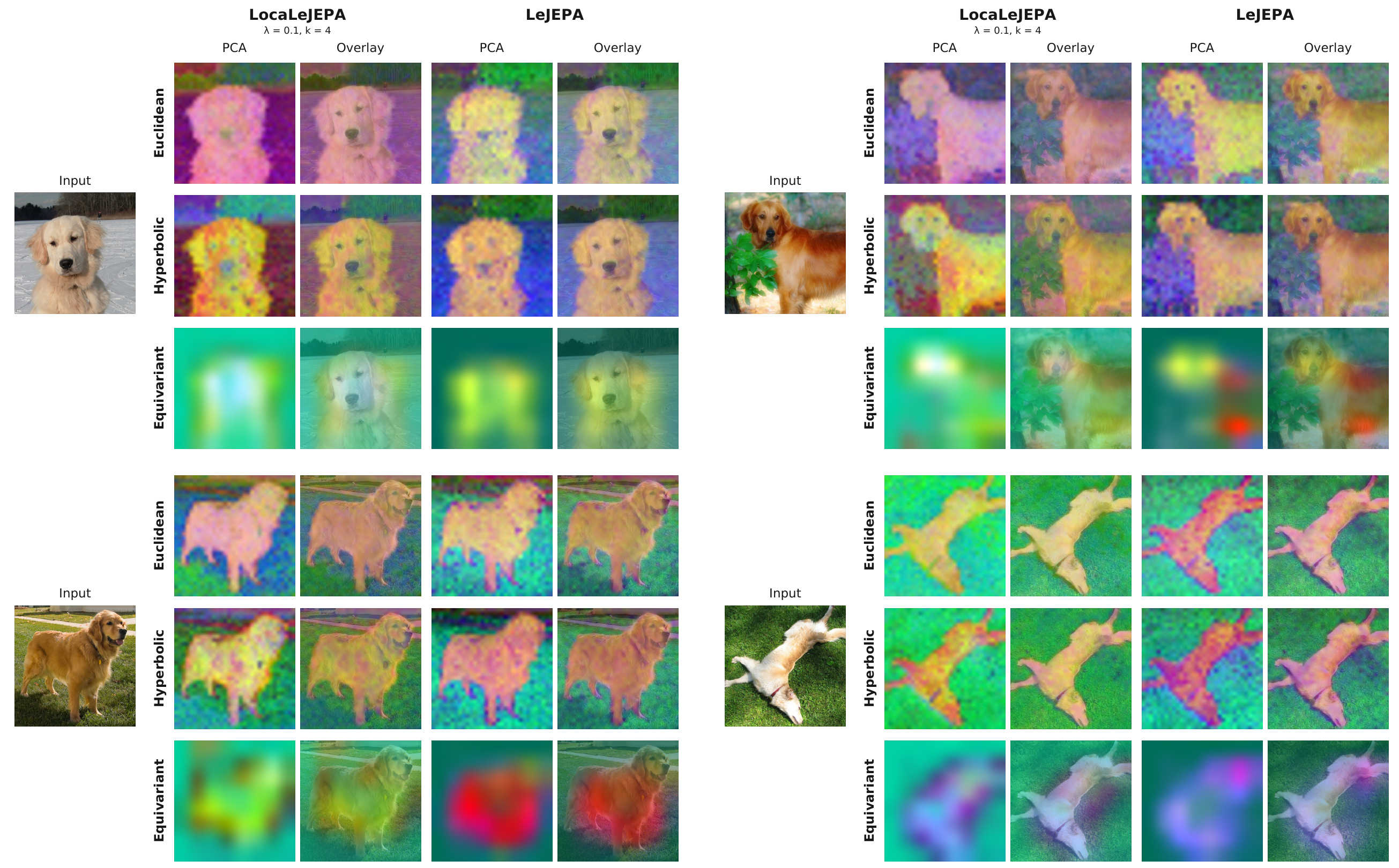}
    \caption{\emph{Patch-level Principal Components for LocaLeJEPA and LeJEPA.}
    Frozen spatial features from the same ImageNet-1K validation images are
    projected onto their first three principal components and mapped to RGB.
    Each model uses an independently fitted PCA basis, so colors are directly
    comparable across images within a model but not across models.}
    \label{fig:pca-vs-sigreg}
\end{figure}

\textbf{Results.} For each backbone--objective pair, we select five ImageNet-1K validation images
from the golden-retriever class (class 207) and extract spatial features from
the frozen encoder immediately before global aggregation. For the Euclidean
and hyperbolic ViT-S/8 encoders, we use the 384-dimensional tokens after the
final transformer normalization, excluding prefix tokens, pooling, and output
heads. At the $128\times128$ input resolution, these features form a
$16\times16$ spatial grid. For the hyperbolic encoder, features are taken
immediately before the Poincar\'e-ball mapping. For the $D_4$-equivariant
ResNet-50, we use the 512-channel equivariant feature map immediately before
average pooling. For each model, spatial features from the five images are concatenated into an $N\times d$ matrix, where the $N$ spatial locations are observations
and the $d$ feature channels are variables. PCA is fitted once to this matrix, and the
three leading components are retained. We fix the sign of each component by
requiring its largest-magnitude loading to be positive and independently rescale
the resulting component values to $[0,1]$. The three channels are then mapped
to RGB and reshaped to the native spatial grid. For visualization only, the
result is upsampled to $512\times512$ using nearest-neighbor interpolation and
overlaid on the input image at 55\% opacity. No foreground masking or
thresholding is applied. \Cref{fig:pca-vs-sigreg} compares LocaLeJEPA
($\lambda{=}0.1$, $K{=}4$) with the corresponding SIGReg baseline on four
representative images. Both objectives exhibit visible object--background
structure in their leading principal components. These qualitative plots allow the PCA to identify directions of feature variance \rev{without using} semantic labels, and each model uses an independently fitted PCA basis. \rev{As a result},
colors are comparable across images within a row, but not across independently
fitted models, and the visualization should not be interpreted as a quantitative
segmentation or localization result.

\subsection{Class Consistency of Frozen k-NN Neighborhoods}
\label{app:knn-neighborhood}

\rev{On the three $128$-pixel encoders, the largest ImageNet-1K improvement of LocaLeJEPA over LeJEPA (our settings)} appears under the
$k$-NN readout (\Cref{tab:in1k-indomain-full}), which depends directly on
the local geometry of the frozen representation. We therefore examine whether
this improvement is reflected in the composition of the neighborhoods on which
the classifier operates. Specifically, we ask whether the nearest ImageNet-1K
training examples of a validation image are more likely to share its semantic
class under LocaLeJEPA than under SIGReg. We use exactly the representation space employed by the main $k$-NN evaluation: the $128\times128$ evaluation view, frozen 512-dimensional encoder features, cosine similarity, and the complete 1.28M-image ImageNet-1K training split as the retrieval bank. \rev{This analysis uses encoders from earlier pretraining runs, not those reported in \Cref{tab:in1k-main}.}
\begin{table}[t]
\centering
\caption{\emph{Class Consistency of Frozen $k$-NN Neighborhoods by Semantic Group.}
Mean top-20 same-class neighborhood consistency on the ImageNet-1K validation set, using the full training split as the retrieval bank. LocaLeJEPA uses $\lambda{=}0.1$ and $K{=}4$. Group sizes are shown in parentheses, and the final row reports the number of individual classes for which LocaLeJEPA yields higher consistency. \textbf{Bold} denotes the higher value within each encoder family.}
\label{tab:knn-neighbor-groups}
\small
\setlength{\tabcolsep}{4pt}
\resizebox{\linewidth}{!}{\begin{tabular}{l cc cc cc}
\toprule
 & \multicolumn{2}{c}{ViT-S/8} & \multicolumn{2}{c}{H-ViT-S/8} & \multicolumn{2}{c}{$D_4$-RN50} \\
\cmidrule(lr){2-3}\cmidrule(lr){4-5}\cmidrule(lr){6-7}
 & SIGReg & LocaLeJEPA & SIGReg & LocaLeJEPA & SIGReg & LocaLeJEPA \\
\midrule
Fish (16) & 0.25 & \textbf{0.34} & 0.26 & \textbf{0.35} & 0.28 & \textbf{0.33} \\
Birds (38) & 0.13 & \textbf{0.22} & 0.13 & \textbf{0.22} & 0.14 & \textbf{0.19} \\
Amphibians (8) & 0.11 & \textbf{0.21} & 0.11 & \textbf{0.21} & 0.11 & \textbf{0.22} \\
Reptiles (36) & 0.08 & \textbf{0.15} & 0.09 & \textbf{0.16} & 0.10 & \textbf{0.17} \\
Invertebrates (86) & 0.16 & \textbf{0.26} & 0.16 & \textbf{0.26} & 0.18 & \textbf{0.25} \\
Marine mammals (4) & 0.29 & \textbf{0.40} & 0.28 & \textbf{0.41} & 0.26 & \textbf{0.31} \\
Dogs (118) & 0.07 & \textbf{0.12} & 0.07 & \textbf{0.12} & 0.06 & \textbf{0.12} \\
Wild canids (12) & 0.09 & \textbf{0.16} & 0.09 & \textbf{0.16} & 0.09 & \textbf{0.15} \\
Cats (13) & 0.22 & \textbf{0.31} & 0.22 & \textbf{0.31} & 0.20 & \textbf{0.30} \\
Bears (4) & 0.16 & \textbf{0.26} & 0.17 & \textbf{0.27} & 0.15 & \textbf{0.29} \\
Other mammals (63) & 0.12 & \textbf{0.20} & 0.12 & \textbf{0.21} & 0.13 & \textbf{0.21} \\
Objects, food, scenes (602) & 0.14 & \textbf{0.20} & 0.14 & \textbf{0.21} & 0.14 & \textbf{0.20} \\
\midrule
All 1000 classes & 0.13 & \textbf{0.20} & 0.13 & \textbf{0.21} & 0.13 & \textbf{0.20} \\
LocaLeJEPA higher & \multicolumn{2}{c}{989 / 1000} & \multicolumn{2}{c}{984 / 1000} & \multicolumn{2}{c}{971 / 1000} \\
\bottomrule
\end{tabular}}
\end{table}

\begin{table}[t]
\centering
\caption{\emph{Class Consistency of Frozen $k$-NN Neighborhoods by Class.}
Mean top-20 same-class consistency over the 50 ImageNet-1K validation images of each class. LocaLeJEPA uses $\lambda{=}0.1$ and $K{=}4$. \textbf{Bold} denotes the higher value within each encoder family. The final row reports the average over all 1{,}000 classes.}
\label{tab:knn-neighbor-classes}
\small
\setlength{\tabcolsep}{4pt}
\resizebox{0.7\linewidth}{!}{\begin{tabular}{l cc cc cc}
\toprule
 & \multicolumn{2}{c}{ViT-S/8} & \multicolumn{2}{c}{H-ViT-S/8} & \multicolumn{2}{c}{$D_4$-RN50} \\
\cmidrule(lr){2-3}\cmidrule(lr){4-5}\cmidrule(lr){6-7}
 & SIGReg & LocaLeJEPA & SIGReg & LocaLeJEPA & SIGReg & LocaLeJEPA \\
\midrule
\multicolumn{7}{l}{\emph{Bears, reptiles, marine mammals}} \\
Polar bear & 0.19 & \textbf{0.31} & 0.22 & \textbf{0.32} & 0.14 & \textbf{0.26} \\
Brown bear & 0.11 & \textbf{0.20} & 0.12 & \textbf{0.23} & 0.09 & \textbf{0.22} \\
Komodo & 0.14 & \textbf{0.25} & 0.21 & \textbf{0.31} & 0.29 & \textbf{0.33} \\
Triceratops & 0.20 & \textbf{0.35} & 0.20 & \textbf{0.40} & 0.33 & \textbf{0.52} \\
Alligator & 0.05 & \textbf{0.08} & 0.05 & \textbf{0.07} & 0.05 & \textbf{0.10} \\
Cobra & 0.05 & \textbf{0.11} & 0.06 & \textbf{0.10} & 0.10 & \textbf{0.24} \\
Grey whale & 0.31 & \textbf{0.41} & 0.31 & \textbf{0.42} & \textbf{0.30} & 0.27 \\
Orca & 0.40 & \textbf{0.56} & 0.39 & \textbf{0.58} & 0.37 & \textbf{0.50} \\
Dugong & 0.38 & \textbf{0.49} & 0.32 & \textbf{0.49} & 0.27 & \textbf{0.38} \\
Sea lion & 0.06 & \textbf{0.12} & 0.09 & \textbf{0.13} & 0.08 & \textbf{0.11} \\
\midrule
\multicolumn{7}{l}{\emph{Other animals}} \\
Goldfish & 0.18 & \textbf{0.35} & 0.19 & \textbf{0.34} & 0.26 & \textbf{0.29} \\
Great white shark & 0.35 & \textbf{0.40} & 0.38 & \textbf{0.43} & 0.25 & \textbf{0.32} \\
Bald eagle & 0.16 & \textbf{0.27} & 0.14 & \textbf{0.25} & 0.14 & \textbf{0.22} \\
Flamingo & 0.28 & \textbf{0.43} & 0.31 & \textbf{0.46} & 0.29 & \textbf{0.38} \\
King penguin & 0.32 & \textbf{0.57} & 0.29 & \textbf{0.53} & 0.45 & \textbf{0.68} \\
Golden retriever & 0.04 & \textbf{0.09} & 0.04 & \textbf{0.08} & 0.04 & \textbf{0.08} \\
Tabby cat & 0.19 & \textbf{0.30} & 0.20 & \textbf{0.31} & 0.23 & \textbf{0.33} \\
Tiger & 0.32 & \textbf{0.47} & 0.32 & \textbf{0.42} & 0.24 & \textbf{0.34} \\
Zebra & 0.77 & \textbf{0.85} & 0.79 & \textbf{0.85} & 0.72 & \textbf{0.78} \\
African elephant & 0.25 & \textbf{0.35} & 0.23 & \textbf{0.34} & 0.30 & \textbf{0.32} \\
Giant panda & 0.46 & \textbf{0.61} & 0.54 & \textbf{0.66} & 0.49 & \textbf{0.64} \\
Ladybug & 0.05 & \textbf{0.16} & 0.06 & \textbf{0.13} & 0.07 & \textbf{0.10} \\
Monarch butterfly & 0.51 & \textbf{0.72} & 0.59 & \textbf{0.77} & 0.33 & \textbf{0.52} \\
\midrule
\multicolumn{7}{l}{\emph{Objects, food, scenes}} \\
Acoustic guitar & 0.06 & \textbf{0.12} & 0.07 & \textbf{0.13} & 0.05 & \textbf{0.12} \\
Airliner & 0.53 & \textbf{0.58} & 0.52 & \textbf{0.59} & 0.50 & \textbf{0.52} \\
Ambulance & 0.34 & \textbf{0.36} & 0.31 & \textbf{0.38} & 0.22 & \textbf{0.33} \\
Castle & 0.25 & \textbf{0.28} & 0.22 & \textbf{0.27} & 0.16 & \textbf{0.23} \\
Grand piano & 0.09 & \textbf{0.22} & 0.08 & \textbf{0.21} & 0.07 & \textbf{0.23} \\
Laptop & 0.04 & \textbf{0.10} & 0.04 & \textbf{0.10} & 0.04 & \textbf{0.06} \\
School bus & 0.12 & \textbf{0.20} & 0.15 & \textbf{0.22} & 0.10 & \textbf{0.15} \\
Sports car & 0.17 & \textbf{0.22} & 0.14 & \textbf{0.24} & 0.19 & \textbf{0.23} \\
Teapot & 0.05 & \textbf{0.07} & \textbf{0.07} & 0.07 & 0.06 & \textbf{0.13} \\
Banana & 0.06 & \textbf{0.12} & 0.06 & \textbf{0.15} & 0.07 & \textbf{0.11} \\
Pizza & 0.22 & \textbf{0.37} & 0.23 & \textbf{0.34} & 0.18 & \textbf{0.26} \\
Volcano & 0.25 & \textbf{0.32} & 0.23 & \textbf{0.28} & 0.27 & \textbf{0.33} \\
Daisy & 0.72 & \textbf{0.84} & 0.78 & \textbf{0.83} & \textbf{0.85} & 0.83 \\
\midrule
All 1000 classes & 0.13 & \textbf{0.20} & 0.13 & \textbf{0.21} & 0.13 & \textbf{0.20} \\
\bottomrule
\end{tabular}}
\end{table}

\textbf{Qualitative Retrieval.}
\Cref{fig:knn-retrieval} visualizes the eight nearest training examples
for three query images under the Euclidean ViT-S/8. As shown, for the polar-bear query, the number of same-class neighbors increases from four under SIGReg to seven under LocaLeJEPA\rev{. For the Komodo dragon, it rises from one to five, and for the triceratops, from three to four.} Several SIGReg neighbors instead share strong visual attributes with the query (such as color or
background) without sharing its class. The aggregate analysis below evaluates
whether this pattern extends beyond these examples.
\begin{figure}[htbp]
    \centering
    \includegraphics[width=\linewidth]{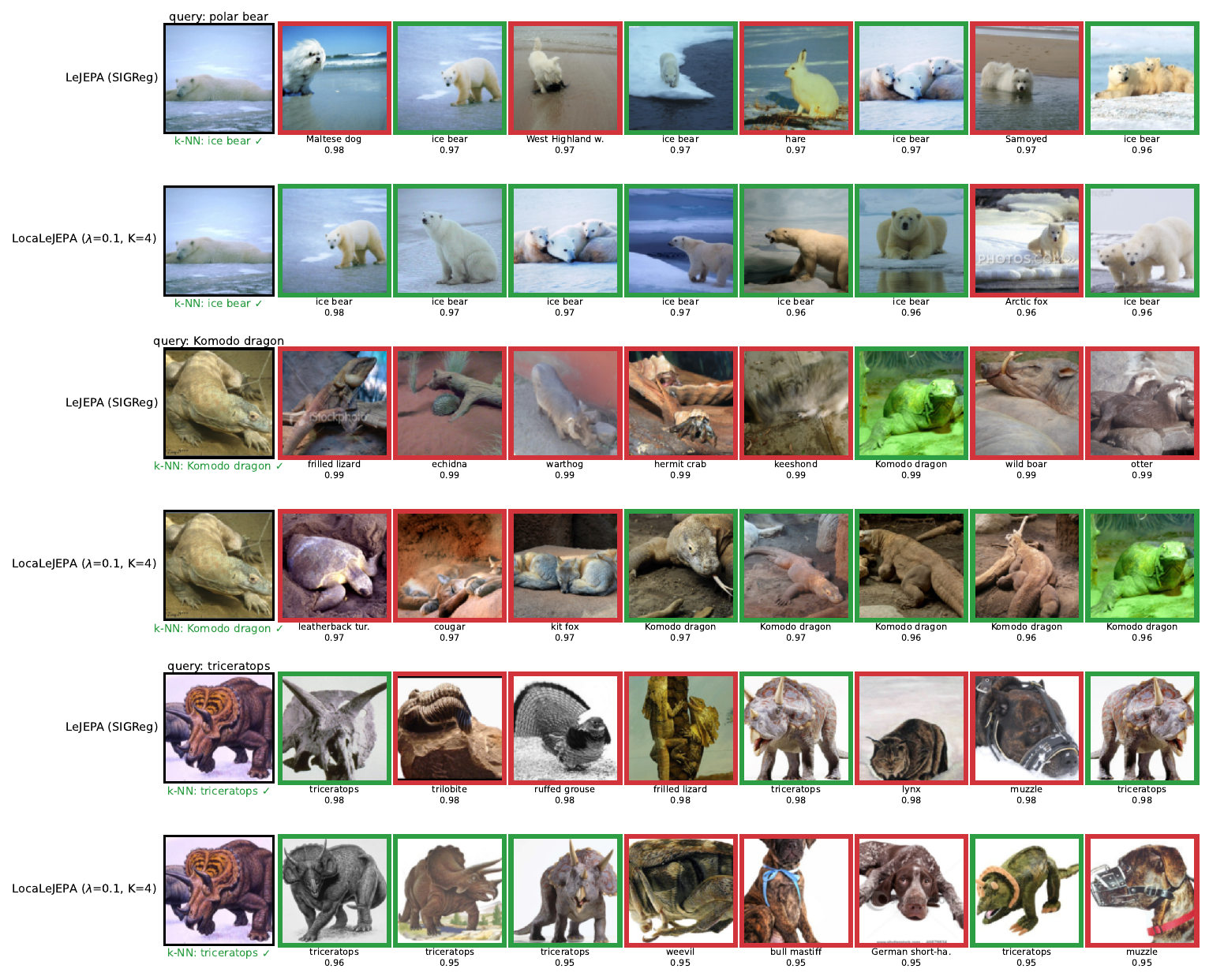}
    \caption{\emph{Nearest-neighbor Retrieval in Frozen ViT-S/8 Features.}
    For each preselected query, we show the eight nearest ImageNet-1K training
    images under cosine similarity for SIGReg and LocaLeJEPA
    ($\lambda{=}0.1$, $K{=}4$). Green borders denote the query class and red
    borders a different class, with labels reporting class and cosine similarity.}
    \label{fig:knn-retrieval}
\end{figure}

\textbf{Neighborhood Class Consistency.}
For a validation image $x$ with class label $c(x)$, let
$\mathcal N_{\revred{k}}(x)$ denote its \revred{$k$} nearest training examples under cosine
similarity. We define the neighborhood class consistency as:
\begin{equation}
    q_{\revred{k}}(x)
    =
    \frac{1}{\revred{k}}
    \sum_{y\in\mathcal N_{\revred{k}}(x)}
    \mathbbm{1}\!\left[c(y)=c(x)\right]
    \label{eq:knn-class-consistency}
\end{equation}
Thus, $q_{\revred{k}}(x){=}1$ when every retrieved neighbor belongs to the query class and
$q_{\revred{k}}(x){=}0$ when none does. We use \revred{$k{=}20$}, matching the headline
$k$-NN readout, and average the statistic either over the 50 validation images
of each class or over the complete 50k-image validation set.

\textbf{Results.}
\Cref{tab:knn-neighbor-groups} reports the statistic over all 1{,}000
ImageNet-1K classes, grouped by WordNet category. Under SIGReg, approximately
$13\%$ of the top-20 neighborhood shares the query class for each of the three
encoder families. LocaLeJEPA increases this fraction to approximately $20\%$
for the Euclidean ViT-S/8 and $D_4$-equivariant ResNet-50, and to $21\%$ for
the hyperbolic ViT-S/8. LocaLeJEPA increases class consistency in every reported WordNet group and in 989, 984, and 971 of the 1{,}000 individual classes for the Euclidean,
hyperbolic, and equivariant encoders, respectively.
\Cref{tab:knn-neighbor-classes} provides a representative 36-class
breakdown, where LocaLeJEPA yields higher consistency in 105 of the 108
class--encoder comparisons. Absolute neighborhood consistency varies \rev{widely} across semantic groups.
Fine-grained families containing many visually similar categories, such as dog
breeds and reptiles, exhibit lower class purity under both objectives, whereas
visually \rev{distinct} categories such as zebra, daisy, and monarch butterfly
show \rev{much} higher values. This behavior is expected for ImageNet-1K:
the weighted $k$-NN prediction requires only that the correct class receive the
largest aggregate vote, not that a majority of retrieved samples share the
query label. \rev{Hence}, \revred{the $k$-NN top-1 accuracies of these encoders} are compatible with neighborhoods whose
majority consists of other classes. Together, the retrieval examples and full-validation statistics show that the
$k$-NN improvement of LocaLeJEPA is accompanied by a systematic increase in
local semantic class consistency across architectures and categories. This
analysis directly characterizes the neighborhoods used by the classifier
\rev{and needs no} additional learned readout.


\end{document}